\documentclass{ecai} 

\usepackage{latexsym}
\usepackage{amssymb}
\usepackage{amsmath}
\usepackage{amsthm}
\usepackage{booktabs}
\usepackage{enumitem}
\usepackage{graphicx}
\usepackage{color}

\usepackage{amsmath,amsfonts,bm}

\def\eqref#1{equation~\ref{#1}}
\def\Eqref#1{Equation~(\ref{#1})}

\def\1{\bm{1}}

\def\vb{{\bm{b}}}

\def\vx{{\bm{x}}}
\def\vy{{\bm{y}}}

\def\mW{{\bm{W}}}

\DeclareMathAlphabet{\mathsfit}{\encodingdefault}{\sfdefault}{m}{sl}
\SetMathAlphabet{\mathsfit}{bold}{\encodingdefault}{\sfdefault}{bx}{n}

\newcommand{\sm}{\mathbf{\sigma}}
\newcommand{\e}{{e}}

\newcommand{\set}[1]{\left\{#1\right\}}
\newcommand{\class}[1]{c_{#1}}
\newcommand{\net}[1]{\ensuremath{\mathcal{N}_{#1}}}
\newcommand{\xlayer}[1]{\ensuremath{\vx^{({#1})}}}
\newcommand{\ylayer}[1]{\ensuremath{\vy^{({#1})}}}

\newcommand{\subxlayer}[2]{\ensuremath{\vx_{#1}^{(#2)}}}
\newcommand{\subylayer}[2]{\ensuremath{\vy_{#1}^{(#2)}}}

\newcommand{\txlayer}[1]{\ensuremath{\Tilde{\vx}^{(#1)}}}
\newcommand{\tylayer}[1]{\ensuremath{\Tilde{\vy}^{(#1)}}}

\newcommand{\pr}[2]{\pi^{#1}_{#2}}

\newcommand{\rpr}[3]{\Pi^{#1\mid{#2}}_{#3}}

\newcommand{\lnbigpi}[5]{\ln\left({\Pi^{{#1}\mid{#2}}_{\xlayer{0}}(#5)}\right)~|~{\xlayer{0}\in\sphere{#3}{#4}}}
\newcommand{\bigpiopenedxy}[5]{\ln\left(\frac{\sm_{\class{i}}(\vx^{({#1})}) \cdot \sm_{{\class{j}}}(\vy^{({#2})})}{\sm_{{\class{j}}}(\vx^{({#1})}) \cdot \sm_{\class{i}}(\vy^{({#2})})}\right)
~|~{\xlayer{0}\in\sphere{#3}{#4}}}
\newcommand{\bigpiopenedyx}[5]{\ln\left(\frac{\sm_{\class{i}}(\vy^{({#1})}) \cdot \sm_{{\class{j}}}(\vx^{({#2})})}{\sm_{{\class{j}}}(\vy^{({#1})}) \cdot \sm_{\class{i}}(\vx^{({#2})})}\right)
~|~{\xlayer{0}\in\sphere{#3}{#4}}}
\newcommand{\lrpr}[5]{min\left\{{\Pi^{{#1}\mid{#2}}_{\xlayer{0}}(#5)}~|~{\xlayer{0}\in\sphere{#3}{#4}}\right\}}
\newcommand{\urpr}[5]{max\left\{{\Pi^{{#1}\mid{#2}}_{\xlayer{0}}(#5)}~|~{\xlayer{0}\in\sphere{#3}{#4}}\right\}}

\newcommand{\minprob}[4]{\mathcal{M}^{{#1}\mid{#2}}_{#3,#4}}

\newcommand{\relaxprob}[4]{\mathcal{R}^{{#1}\mid{#2}}_{#3,#4}}

\newcommand{\sphere}[2]{{\bm{D}}_{#1}^{#2}}

\newcommand{\zeronet}{\mathbf{0}}

\usepackage{microtype}
\usepackage[table,xcdraw]{xcolor} 
\usepackage{mathtools}
\usepackage{multirow}
\usepackage{relsize} 
\usepackage{subcaption}
\usepackage[acronym, nohypertypes={acronym}]{glossaries}
\newacronym{DNN}{DNN}{Deep Neural Network}
\newacronym{VNN}{VNN}{Verification-friendly Neural Network}
\newacronym{ReLU}{ReLU}{Rectified Linear Unit}
\newacronym{LP}{LP}{Linear Programming}
\newacronym{MBP}{MBP}{Magnitude-Based Pruning}
\newacronym{RPR}{ROM}{Relative Output Margin}
\newacronym{LRPR}{LROM}{Local Relative Output Margin}
\newacronym{PR}{OM}{Output Margin}
\newacronym{RPRB}{Relative OM Bound}{Relative Output Margin Bound} 
\newacronym{AI}{AI}{Artificial Intelligence}
\newacronym{ML}{ML}{Machine Learning}
\newacronym{PGD}{PGD}{Projected Gradient Descent}

\newcommand\locimplies{\stackrel{\mathclap{\normalfont\mbox{$\sphere{}{}$}}}{\implies}}

\newtheorem{theorem}{Theorem}
\newtheorem{lemma}{Lemma}
\newtheorem{corollary}{Corollary}

\newtheorem{definition}{Definition}
\newtheorem{remark}{Remark}

\newcommand{\BibTeX}{B\kern-.05em{\sc i\kern-.025em b}\kern-.08em\TeX}

\begin{document}


\begin{frontmatter}


\paperid{7416} 


\title{Formal Local Implication Between Two Neural Networks\\[0.5em]
\large\textit{Accepted at European Conference on Artificial Intelligence (ECAI 2025)}}


\author[A]{\fnms{Anahita}~\snm{Baninajjar}\orcid{0000-0002-5719-274X
}\thanks{Corresponding Author. Email: anahita.baninajjar@eit.lth.se.}}
\author[B]{\fnms{Ahmed}~\snm{Rezine}\orcid{0000-0002-0440-4753}}
\author[A]{\fnms{Amir}~\snm{Aminifar}\orcid{0000-0002-1673-4733}} 

\address[A]{Department of Electrical and Information Technology, Lund University, Sweden}
\address[B]{Department of Computer and Information Science, Link\"oping University, Sweden}


\begin{abstract}
Given two neural network classifiers with the same input and output domains, our goal is to compare the two networks in relation to each other over an entire input region (e.g., within a vicinity of an input sample). To this end, we establish the foundation of \emph{formal local implication} between two networks, i.e., $\net{2} \locimplies \net{1}$, in an \emph{entire input region} $\sphere{}{}$. That is, network $\net{1}$ consistently makes a correct decision every time network $\net{2}$ does, and it does so in an entire input region $\sphere{}{}$. We further propose a \emph{sound} formulation for establishing such formally-verified (provably correct) local implications. 
The proposed formulation is relevant in the context of several application domains, e.g., for comparing a trained network and its corresponding compact (e.g., pruned, quantized, distilled) networks. We evaluate our formulation based on the MNIST, CIFAR10, and two real-world medical datasets, to show its relevance.
\end{abstract}

\end{frontmatter}

\glsresetall
\section{Introduction}

\label{intro}
Quantitative comparison of neural networks, e.g., in terms of performance, is a fundamental concept in the \gls{ML} domain. One common example is when a network is pruned, quantized, or distilled to run the compact networks on edge devices or smart sensors. In the medical domain, for instance, neural networks can enable implantable and wearable devices to detect myocardial infarction \citep{Sopic:myocard} or epileptic seizures \citep{baghersalimi2024m2skd} in real time. However, due to their limited computing resources, such devices often adopt the compact networks corresponding to the original medical-grade networks. It is vital for the compact network to reliably detect cardiac abnormalities/seizures, as lack of reliable decisions can jeopardize patients' lives. Therefore, reasoning about the decisions made by the compact network and their relation to the decisions made by the original/reference network is vital for the safe deployment of the compact networks.

In this work, we focus on compatible neural networks, i.e., two neural networks trained for the same learning/classification task, with the same input and output domains, but not the same weights and/or architectures. We define an input region as the region in a vicinity of a given input sample, e.g., captured by the absolute-value/Euclidean/maximum norm centered around the input sample. Given the two networks and an input region, we investigate whether it is possible to prove that, \emph{in the entire input region}, one network ($\net{1}$) consistently makes a correct decision every time the other network ($\net{2}$) does. That is, $\net{2} \locimplies \net{1}$, in an entire input region $\sphere{}{}$. This is the definition of \emph{implication} that is valid in an entire input region $\sphere{}{}$ (i.e., local), hence referred to as \emph{local implication}.

\begin{figure}[t]
\centering
\includegraphics[width=0.87\columnwidth, trim={0.cm 0.cm 0.cm 0.cm},clip]{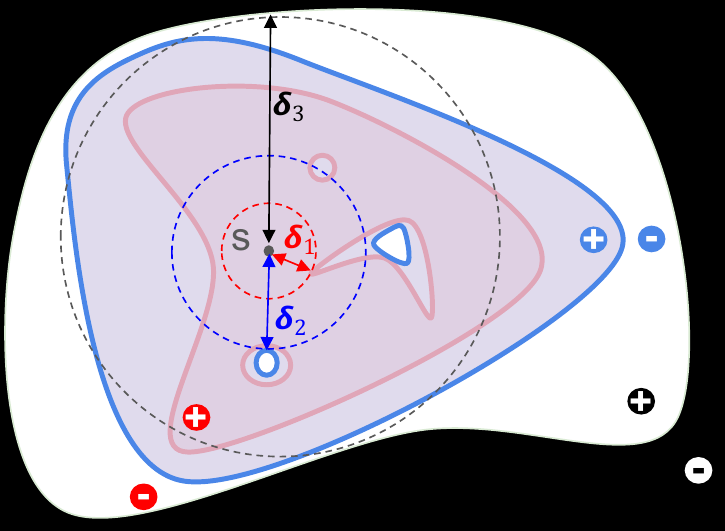}
\caption{Two binary classifiers (red and blue) learn to capture the black/white classes $\mathlarger{\oplus}$ and $\mathlarger{\ominus}$. Formal local implication captures, despite robustness violations of both red and blue classifiers, that the blue classifier makes the right decision each time the red one does within $\delta_3$ perturbations of sample $s$.}
\label{fig:implied-robustness}
\end{figure}

Formal local implication can capture, as illustrated in Figure~\ref{fig:implied-robustness},
that despite possible violations of local robustness, one network (the blue classifier in Figure~\ref{fig:implied-robustness}) makes the correct decision (with respect to ground truth as captured by the black/white  $\oplus$ and $\ominus$ in the figure) each time the other network does (the red classifier in the figure). Tracking the decisions of the two networks separately cannot capture local implication. Instead, outputs of both networks need to be compared for each single input sample throughout the whole considered input region. To this end, we define the notion of \gls{LRPR}. 
\gls{LRPR} not only enables us to reason about the given input samples, but also to reason about the entire region in the vicinity of the samples. 

In this paper, we propose a formulation to establish safe (provably-correct) bounds on \gls{LRPR} and formal verification guarantees on the decisions made by the two networks in the entire input region. \gls{LRPR} enables us to formally prove that a network consistently makes a correct decision every time the other network does, and it does so in the entire input region. We evaluate our proposed formulation extensively on several datasets to show its relevance, including two real-world medical applications for detection of cardiac arrhythmia or epileptic seizures.
Our main contributions are summarized below:
\begin{itemize}
\item We establish the foundation of \emph{formal local implication} between two networks $\net{1}$ and $\net{2}$, i.e., $\net{2} \locimplies \net{1}$.
\item We propose a \emph{sound} formulation for establishing such formally-verified (provably correct) local implications.
\item We conduct extensive experiments to compare the decisions made by pre-trained classifiers and their corresponding pruned, quantized, knowledge-distilled, or verification-friendly counterparts, on the MNIST dataset~\citep{mnist}, CIFAR10 dataset~\citep{krizhevsky2009learning}, CHB-MIT epilepsy dataset~\citep{shoeb_chb-mit_2010}, and MIT-BIH arrhythmia dataset~\citep{mit-bih}.
\end{itemize}

\section{Formal Definitions and Notations}
\label{sec:rpr}

In this section, we formally describe \gls{DNN} classifiers. Moreover, we introduce and formalize the notion of \gls{RPR} and its extension to an entire input region, i.e., \gls{LRPR}. Finally, we connect the notion of \gls{LRPR} to formal local implications.
\subsection{\acrfullpl{DNN}}

In this work, we mainly consider \gls{DNN} classifiers. A \gls{DNN} classifier is a
nonlinear function \( \net{}: \mathbb{R}^{{n^{\net{}}_0}} \rightarrow
\mathbb{R}^{n^{\net{}}_{N}}\) consisting of a sequence of \(N\) layers
followed by a softmax layer. Each layer is a linear transformation
followed by a nonlinear activation function. Here, \(n^{\net{}}_k\)
is the number of neurons in the \(k^{th}\) layer of network
\(\net{}\). Let \(f^{\net{}}_k(\cdot) : \mathbb{R}^{n^{\net{}}_{k-1}}
\rightarrow \mathbb{R}^{n^{\net{}}_k}\) be the function that derives
values of the \(k^{th}\) layer from the output of its preceding layer.
The values of the \(k^{th}\) layer, denoted by \(\xlayer{k}\), are given by:
\[ \xlayer{k} = f^{\net{}}_k(\xlayer{k-1}) = act^{\net{}}_k(\mW^{(k)}\xlayer{k-1} + \vb^{(k)}), \]
where \(\mW^{(k)}\) and \(\vb^{(k)}\) capture weights and biases of
the \(k^{th}\) layer, and \(act^{\net{}}_k\) represents an
activation function. The last layer uses softmax as the activation function. For each class \(\class{i}\) in the last layer \(N+1\), the softmax function value is: \(\subxlayer{\class{i}}{N+1}=\sm_{\class{i}}(\xlayer{N})\).

\subsection{\acrfullpl{LRPR}}
\label{subsec:lrpr}
We consider two \glspl{DNN} \net{1} with $N_1+1$ layers with values  $\xlayer{0}, \ldots \xlayer{N_1+1}$ and \net{2} with $N_2+1$ layers with values $\ylayer{0}, \ldots \ylayer{N_2+1}$. Suppose  \(n^{\net{1}}_{0} = n^{\net{2}}_{0}\) and \(n^{\net{1}}_{N_1} = n^{\net{2}}_{N_2}\). Such networks are said to be {compatible} as their inputs and outputs have the same dimensions.

Let us now introduce the notions of \acrfull{PR} and of \acrfull{RPR}.
\begin{definition}
\acrfull{PR}  $\pr{\net{1}}{\xlayer{0}}(\class{i},\class{j})$ of 
classes $(\class{i},\class{j})$ for \gls{DNN}
$\net{1}$ and input $\xlayer{0}$ is the ratio
\( \pr{\net{1}}{\xlayer{0}}(\class{i},\class{j})
=\frac{\sm_{\class{i}}(\xlayer{N_1})}{\sm_{\class{j}}(\xlayer{N_1})}
\)
of the outcome being $\class{i}$ by the one of being $\class{j}$.

\end{definition}

Recall classifiers decide on the class with a maximum softmax value. Let us consider binary classification for simplicity. Assuming the predicted class is \(\class{i}\), we know \(\sm_{\class{i}}(\xlayer{N_1})\ge {\sm_{\class{j}}(\xlayer{N_1})}\) and, in turn, \(\pr{\net{1}}{\xlayer{0}}(\class{i},\class{j}) = \frac{\sm_{\class{i}}(\xlayer{N_1})}{\sm_{\class{j}}(\xlayer{N_1})}\ge 1\).

\begin{definition} \label{def:rpr}
\acrfull{RPR}
$\rpr{\net{1}}{\net{2}}{\xlayer{0}}(\class{i},\class{j})$ of class pair
$(\class{i},\class{j})$  for $\net{1}$ w.r.t.
compatible $\net{2}$,
and for common input $\xlayer{0}=\ylayer{0}$, is the quotient of
the \gls{PR} in \net{1} by the one in \net{2}: 
\[ \rpr{\net{1}}{\net{2}}{\xlayer{0}}(\class{i},{\class{j}})\!=\!
\frac{\pi^{\net{1}}_{\xlayer{0}}(\class{i},{\class{j}})}{\pi^{\net{2}}_{\ylayer{0}}(\class{i},{\class{j}})}
\!=\! \frac{\sm_{\class{i}}(\vx^{(N_1)})\! \cdot \sm_{{\class{j}}}(\vy^{(N_2)})}{\sm_{{\class{j}}}(\vx^{(N_1)})\! \cdot \sm_{\class{i}}(\vy^{(N_2)})}.\]
\end{definition}
We use
$\rpr{\net{1}}{\net{2}}{\xlayer{0}}(\class{i},{\class{j}})$ to compare
output margins between $\class{i}$ and ${\class{j}}$ in \glspl{DNN} \(\net{1}\) and \(\net{2}\) for the same input $\xlayer{0}$.

To explore formal local implication, we establish in this paper bounds on \gls{RPR} values in entire input regions, e.g., in the vicinity of an input $\txlayer{0}$ or in a $\delta$-neighborhood of an input $\txlayer{0}$, defined as
 $  \sphere{\txlayer{0}}{\delta}=\left\{\xlayer{0} \mathrm{~s.t.~} \lVert \xlayer{0} - \txlayer{0}\rVert_\infty \leq \delta\right\} $.

 \begin{definition} \emph{\acrfull{LRPR}} of classes
  $(\class{i},{\class{j}})$ for $\net{1}$ w.r.t. its compatible network $\net{2}$
  in 
  $\sphere{\txlayer{0}}{\delta}$  is the set
  $\left\{
\rpr{\net{1}}{\net{2}}{\xlayer{0}}(\class{i},{\class{j}})
  ~|~
  \xlayer{0} \in \sphere{{\txlayer{0}}}{\delta}
  \right\}$.
\end{definition}

\begin{remark}
If \(\lrpr{\net{1}}{\net{2}}{\txlayer{0}}{\delta}{\class{i},{\class{j}}}\ge~1\), then 
\(\pr{\net{1}}{\xlayer{0}}(\class{i},\class{j}) \ge  \pr{\net{2}}{\xlayer{0}}(\class{i},\class{j})\), for all \(\xlayer{0}\) in the entire input region $\sphere{\txlayer{0}}{\delta}$. If $\net{2}$ makes a correct decision, i.e., \(\pr{\net{2}}{\xlayer{0}}(\class{i},\class{j}) = \frac{\sm_{\class{i}}(\xlayer{N_2})}{\sm_{\class{j}}(\xlayer{N_2})}\ge 1\), 
then \(\pr{\net{1}}{\xlayer{0}}(\class{i},\class{j}) = \frac{\sm_{\class{i}}(\xlayer{N_1})}{\sm_{\class{j}}(\xlayer{N_1})}\ge 1\), assuming the predicted class is \(\class{i}\). This, in turn, means that, in the entire input region $\sphere{\txlayer{0}}{\delta}$, $\net{1}$ will make a correct decision every time $\net{2}$ does, i.e.,
$\net{2} \stackrel{\sphere{\txlayer{0}}{\delta}}{\implies} \net{1}$. We say that $\net{2}$ implies $\net{1}$ on $\sphere{\txlayer{0}}{\delta}$.
\end{remark}


\section{Method}
\label{sec:method}

In this section, we introduce an optimization problem to bound \glspl{LRPR} for two compatible \glspl{DNN} and establish formal local implication between two networks. We also describe how we introduce and handle an over-approximation of the two networks in order to soundly solve the optimization problem and derive a (provably-correct) verified bound.

Assume two compatible \glspl{DNN} \net{1} and \net{2} with
respectively $N_1+1$ and $N_2+1$ layers, a common input $\txlayer{0}$
in the domain $\bm{D}$ of \net{1} and \net{2}, and a perturbation
bound $\delta$. Our goal is to find, for any class pair
$(\class{i},\class{j})$, a tight lower bound for
$\lrpr{\net{1}}{\net{2}}{\txlayer{0}}{\delta}{\class{i},\class{j}}$
and a tight upper bound for
$\urpr{\net{1}}{\net{2}}{\txlayer{0}}{\delta}{\class{i},\class{j}}$. 

The above optimization problem involves the softmax function. Therefore, to solve this optimization problem, we look into \(
\ln\left(\rpr{\net{1}}{\net{2}}{\xlayer{0}}(\class{i},{\class{j}})\right)
\) and 
observe it coincides with
\(
(\subxlayer{\class{i}}{N_1} - \subxlayer{\class{j}}{N_1})-(\subylayer{\class{i}}{N_2} - \subylayer{\class{j}}{N_2})
\). Hence, we can characterize \gls{LRPR} bounds by reasoning on inputs to the softmax layers (i.e., networks' logits). As a result, our optimization objective is simplified to:
\begin{align}
  \ln\left(\lrpr{\net{1}}{\net{2}}{\txlayer{0}}{\delta}{\class{i},\class{j}}\right)= \nonumber \\{\min_{\xlayer{0}=\ylayer{0}\in\sphere{\txlayer{0}}{\delta}} \left((\subxlayer{\class{i}}{N_1} - \subxlayer{\class{j}}{N_1}) -  (\subylayer{\class{i}}{N_2} - \subylayer{\class{j}}{N_2})\right), } \nonumber\\
\ln\left(\urpr{\net{1}}{\net{2}}{\txlayer{0}}{\delta}{\class{i},\class{j}}\right)= \nonumber \\{\max_{\xlayer{0}=\ylayer{0}\in\sphere{\txlayer{0}}{\delta}} \left((\subxlayer{\class{i}}{N_1} - \subxlayer{\class{j}}{N_1}) -  (\subylayer{\class{i}}{N_2} - \subylayer{\class{j}}{N_2})\right).}\nonumber
\end{align}
Let us now formulate the optimization problem based on the above objective function and the constraints imposed by the neural networks and input region, as follows:
\begin{align}
  \min_{\xlayer{0}}&\hspace{0.5em} 
  (\subxlayer{\class{i}}{N_1} - \subxlayer{\class{j}}{N_1}) -(\subylayer{\class{i}}{N_2} - \subylayer{\class{j}}{N_2}),  \label{obj} \\
    \textrm{s.t.}  
    & \hspace{0.3cm} {{\vy}}^{(0)} = {\vx}^{(0)}, \hspace{0.3cm} \txlayer{0} \in \bm{D},
    \label{input}\\
    & \hspace{0.3cm} \lVert \xlayer{0} - \txlayer{0}\rVert_\infty = \lVert \ylayer{0} - \txlayer{0}\rVert_\infty \leq \delta, \label{pert}\\
    & \hspace{0.3cm} \xlayer{k} = f^{\net{1}}_k(\xlayer{k-1}), \; \forall k \in \{1, \dots, N_1\}, \label{N1} \\
    & \hspace{0.3cm} \ylayer{l} = f^{\net{2}}_l(\ylayer{l-1}), \; \forall l \in \{1, \dots, N_2\}. \label{N2}
\end{align}
Let
$\minprob{\net{1}}{\net{2}}{\txlayer{0}}{\delta}(\class{i},\class{j})$
be the exact value obtained as solution to this problem. 
\Eqref{obj} introduces the objective function used to capture the (logarithm of the) minimum \gls{RPR} of the class pair $(\class{i},\class{j})$ for
$\net{1}$ w.r.t. $\net{2}$ in the input region $\sphere{\txlayer{0}}{\delta}$.\footnote{The objective function captures a stronger condition than mere implication, as it ensures that the margin does not shrink. While our framework is designed for margin preservation, it can be easily adapted to verify only implication by changing the objective to $\subxlayer{\class{i}}{N_1} - \subxlayer{\class{j}}{N_1}$ and moving the second part of the original objective into the constraints as $\subylayer{\class{i}}{N_2} - \subylayer{\class{j}}{N_2} > 0$.} Note that $\subxlayer{\class{i}}{N_1} - \subxlayer{\class{j}}{N_1}$ captures the difference between the logit values associated to classes $\class{i}$ and $\class{j}$ in network \net{1}. Similarly, $\subylayer{\class{i}}{N_2} - \subylayer{\class{j}}{N_2}$ captures the difference between the logit values associated to the same classes in \net{2}. The objective function is then to minimize the difference between these two quantities.

Let us consider Equations~(\ref{input})--(\ref{N2}). \Eqref{input} enforces both that $\ylayer{0}$ (the perturbed input to network \net{2}) equals $\xlayer{0}$ (the perturbed input to network \net{1}), and that the original input \(\txlayer{0}\) belongs to the dataset $\bm{D}$ of the two networks. \Eqref{pert} enforces that the perturbed inputs $\xlayer{0}$ and $\ylayer{0}$ are in the $\delta$-neighborhood
of \(\txlayer{0}\). \Eqref{N1} characterizes values of the first $N_1$ layers of network \net{1} as it relates the values of the $k^{th}$ layer (for $k$ in $\set{1, \ldots, N_1}$) to those of its preceding layer, using the nonlinear function \({f}^{\net{1}}_{k}:\mathbb{R}^{n^{\net{1}}_{k-1}}\to \mathbb{R}^{n^{\net{1}}_{k}}\). The same is applied to network \net{2} using the nonlinear functions \({f}^{\net{2}}_{l}:\mathbb{R}^{n^{\net{2}}_{l-1}}\to \mathbb{R}^{n^{\net{2}}_{l}}\)
for each layer $l$ as captured in~\Eqref{N2}. 

Equations~(\ref{N1})~and~(\ref{N2}) involve activation functions, hence result in nonlinear constraints. As such, finding the exact global optimal solution is intractable. To address this, we consider a sound over-approximation of nonlinearities. In particular, we consider \gls{ReLU} activation function and adopt existing relaxations for it~\citep{Ehler:linear,DeepPoly, baninajjar2023safedeep}, to over-approximate the values computed at each layer using linear inequalities (see the appendix).

These over-approximations result in a relaxed optimization problem that can be solved using \gls{LP}. The solution of the relaxed optimization problem is denoted by $\relaxprob{\net{1}}{\net{2}}{\txlayer{0}}{\delta}(\class{i},\class{j})$. Because the relaxed optimization over-approximates the exact one in Equations~(\ref{obj})--(\ref{N2}) and that we are able to find the optimal solution to the \gls{LP} relaxed formulation, any lower bound obtained for the relaxed problem is guaranteed to be smaller than a solution for the original minimization problem, i.e., \(\relaxprob{\net{1}}{\net{2}}{\txlayer{0}}{\delta}(\class{i},\class{j}) \leq \minprob{\net{1}}{\net{2}}{\txlayer{0}}{\delta}(\class{i},\class{j})\).

\begin{theorem}
  \label{theo:relaxed_bounds}
  Let $(\class{i},{\class{j}})$ be a pair of classes of compatible
  \glspl{DNN} \net{1} and \net{2}.  Assume a neighborhood
  $\sphere{\txlayer{0}}{\delta}$ and let
  $\relaxprob{\net{1}}{\net{2}}{\txlayer{0}}{\delta}(\class{i},\class{j})$
  (resp. $\relaxprob{\net{2}}{\net{1}}{\txlayer{0}}{\delta}(\class{i},\class{j})$)
  be a solution to the relaxed minimization problem corresponding to
  \gls{LRPR} of $\net{1}$ w.r.t. $\net{2}$ (resp. $\net{2}$ w.r.t. $\net{1}$). 
  Then, we have:
  {\small
  \[\arraycolsep=1.4pt\def\arraystretch{2.2}
  \begin{array}{l}
    \relaxprob{\net{1}}{\net{2}}{\txlayer{0}}{\delta}(\class{i},\class{j}) \leq 
    \minprob{\net{1}}{\net{2}}{\txlayer{0}}{\delta}(\class{i},\class{j})= \\ \ln\left(\lrpr{\net{1}}{\net{2}}{\txlayer{0}}{\delta}{\class{i},\class{j}}\right) 

\leq
    \\ \ln\left(\urpr{\net{1}}{\net{2}}{\txlayer{0}}{\delta}{\class{i},\class{j}}\right)

    =
   \\ -\ln\left(\lrpr{\net{2}}{\net{1}}{\txlayer{0}}{\delta}{\class{i},\class{j}}\right)= \\ -\minprob{\net{2}}{\net{1}}{\txlayer{0}}{\delta}(\class{i},\class{j})
    \leq 
   -\relaxprob{\net{2}}{\net{1}}{\txlayer{0}}{\delta}(\class{i},\class{j}).
  \end{array}
  \]
  }
\end{theorem}

Based on Theorem \ref{theo:relaxed_bounds}, the solutions to relaxed optimization problems provide safe lower/upper bounds for \glspl{LRPR}, i.e., \(\relaxprob{\net{1}}{\net{2}}{\txlayer{0}}{\delta}(\class{i},\class{j}) \leq \minprob{\net{1}}{\net{2}}{\txlayer{0}}{\delta}(\class{i},\class{j}) \le -\relaxprob{\net{2}}{\net{1}}{\txlayer{0}}{\delta}(\class{i},\class{j})\). 

\begin{corollary}\label{lab:cor}
Let $(\class{i},{\class{j}})$ be a pair of classes of compatible
  \glspl{DNN} \net{1} and \net{2}.  Assume a neighborhood
  $\sphere{\txlayer{0}}{\delta}$ and let
  $\relaxprob{\net{1}}{\net{2}}{\txlayer{0}}{\delta}(\class{i},\class{j})$
  be a solution to the relaxed minimization problem corresponding to
  \gls{LRPR} of $\net{1}$ w.r.t. $\net{2}$. If $\relaxprob{\net{1}}{\net{2}}{\txlayer{0}}{\delta}(\class{i},\class{j}) > 0$, then for all perturbed inputs $\xlayer{0} \in \sphere{\txlayer{0}}{\delta}$ for which $\net{2}$ correctly classifies $\xlayer{0}$, then $\net{1}$ also correctly classifies $\xlayer{0}$.
  That is, 
  $\net{2}$ implies $\net{1}$ on $\sphere{\txlayer{0}}{\delta}$, i.e., 
  $\net{2} \stackrel{\sphere{\txlayer{0}}{\delta}}{\implies} \net{1}$.
\end{corollary}

\paragraph{Joint vs Independent Analysis}

Our original optimization problem and its linear relaxation compute \glspl{RPR} bounds for a common input across both networks, ranging over the considered neighborhood. An alternative approach is to reason based on independently obtained ranges of \acrfullpl{PR} for each network. However, while independently computing and combining the ranges of \glspl{PR} for each network leads to sound approximations of \glspl{RPR}, it results in a significant loss of precision, as it does not consider a common input to both networks.
This is formalized by the theorem below, and
is witness by our experiments where we evaluate the corresponding loss in precision.
%

 \begin{theorem}
   \label{theo:individual_diff}

Let \(\minprob{\net{1}}{\zeronet}{\txlayer{0}}{\delta}(\class{i},\class{j})\) denote the value of the objective function in~\Eqref{obj}, i.e., $(\subxlayer{\class{i}}{N_1} - \subxlayer{\class{j}}{N_1}) -(\subylayer{\class{i}}{N_2} - \subylayer{\class{j}}{N_2})$, under the choice of a constant second network \(\net{2}\) that assigns uniform probabilities to all outcomes. In this case, the second term becomes zero, and the expression simplifies to
\(
\subxlayer{\class{i}}{N_1} - \subxlayer{\class{j}}{N_1}
\), corresponding to computing minimum \glspl{PR} for \(\net{1}\) on its own. Similarly, \(\minprob{\net{2}}{\zeronet}{\txlayer{0}}{\delta}(\class{i},\class{j})\) equals \(\subylayer{\class{i}}{N_2} - \subylayer{\class{j}}{N_2}\), and \(\minprob{\zeronet}{\net{2}}{\txlayer{0}}{\delta}(\class{i},\class{j})\) corresponds to its negation. This leads to the following inequality, expressing that the sum of the independently obtained \glspl{PR} is less than or equal to the \gls{RPR} value when both networks are considered together:
     \[\minprob{\net{1}}{\zeronet}{\txlayer{0}}{\delta}(\class{i},\class{j})+\minprob{\zeronet}{\net{2}}{\txlayer{0}}{\delta}(\class{i},\class{j})
     \leq 
     \minprob{\net{1}}{\net{2}}{\txlayer{0}}{\delta}(\class{i},\class{j}).
     \]

 \end{theorem}

 \paragraph{Transitivity Property of \glspl{LRPR}}
Here, we show that \glspl{LRPR} has the transitivity property, which can extend the results to more than two compatible networks.

\begin{theorem} \label{theo:transitivity}
Let $(\class{i},{\class{j}})$ be a pair of classes of compatible
  \glspl{DNN} \net{1}, \net{2}, and \net{3}. Assume a neighborhood
  $\sphere{\txlayer{0}}{\delta}$ and let
  $\relaxprob{\net{1}}{\net{2}}{\txlayer{0}}{\delta}(\class{i},\class{j})$
  be a solution to the relaxed minimization problem corresponding to
  \gls{LRPR} of \net{1} w.r.t. \net{2}. Similarly, let
  $\relaxprob{\net{2}}{\net{3}}{\txlayer{0}}{\delta}(\class{i},\class{j})$
  be a solution to the relaxed minimization problem corresponding to
  \gls{LRPR} of \net{2} w.r.t. \net{3}. If we know $\relaxprob{\net{1}}{\net{2}}{\txlayer{0}}{\delta}(\class{i},\class{j})> 0$ and $\relaxprob{\net{2}}{\net{3}}{\txlayer{0}}{\delta}(\class{i},\class{j})>0$, then we can conclude $\relaxprob{\net{1}}{\net{3}}{\txlayer{0}}{\delta}(\class{i},\class{j})> 0$.
\end{theorem}

\begin{proof}
Proof sketches of Theorems~\ref{theo:relaxed_bounds}--\ref{theo:transitivity}, as well as Corollary~\ref{lab:cor}, are presented in the appendix.
\end{proof}

\section{Evaluation}
\label{sec:eval}

We evaluate our proposed formulation for formal local implication and investigate the ranges of \glspl{LRPR} across various datasets and  \gls{DNN} structures.\footnote{The code is available on \url{github.com/anahitabn94/Formal_Local_Implication}.}
Experiments are executed on a MacBook Pro with an 8-core CPU and 32 GB of RAM using the Gurobi solver~\citep{gurobi}.

\subsection{Datasets}

We use the following datasets for evaluation:

\textbf{MNIST dataset~\citep{mnist}} contains grayscale handwritten digits. Each digit is depicted through a 28$\times$28 pixel image. We consider the first 100 images of the test set, similar to~\citep{ugare2022proof} and~\citep{DeepPoly}.

\textbf{CIFAR10 dataset~\citep{krizhevsky2009learning}} comprises 32$\times$32 colored images categorized into 10 different classes. In alignment with~\citep{ugare2022proof} and~\citep{DeepPoly}, we focus on the first 100 images from the test set.

\textbf{CHB-MIT Scalp EEG database~\citep{shoeb_chb-mit_2010}} includes 23 individuals diagnosed with epileptic seizures. These recordings are sampled in the international 10--20 EEG system, and our focus is on F7-T7 and F8-T8 electrode pairs, commonly used in seizure detection~\citep{8351728}.

\textbf{MIT-BIH Arrhythmia database~\citep{mit-bih}} involves 48 individuals with 2-channel ECG signals. To establish a classification problem, we consider a subset of 14 cardiac patients who demonstrated at least two different types of heartbeats.

\subsection{Neural Networks}

In this section, we describe the neural network architecture for each dataset. Detailed information, e.g., accuracy of these \glspl{DNN}, is available in the appendix.

\subsubsection{Original Networks}

For the MNIST and CIFAR10 datasets, we use fully-connected \glspl{DNN} from~\citep{ugare2022proof}, all of which have undergone robust training as outlined in~\citep{chiang2020certified}. We also employ convolutional \glspl{DNN} described in~\citep{ugare2022proof}, with results provided in the appendix. For the CHB-MIT dataset, the personalized \gls{DNN} for each patient has 2048 input neurons, convolution layers followed by max-pooling, and ends with a dense layer~\citep{baninajjarvnn}. For the MIT-BIH dataset, the \gls{DNN} has 320 input neurons, a convolution layer, and a dense layer~\citep{baninajjarvnn}.

\subsubsection{Compact Networks}

Our experiments involve various techniques to derive compact \glspl{DNN}. These techniques derive compact \glspl{DNN} enabling energy-efficient inference on limited resources.

\textbf{Pruned Networks} are created through a pruning procedure applied to \glspl{DNN}, where certain weights and biases are selectively nullified. For the MNIST and CIFAR10 datasets, we use pruned networks generated by~\citep{ugare2022proof} via post-training pruning. Each pruned network removes the smallest weights/biases in each layer, a process called \gls{MBP}, resulting in nine pruned networks with pruning rates ranging from 10\% to 90\%. For the CHB-MIT and MIT-BIH datasets, we apply \gls{MBP} pruning by setting values below 10\% of the maximum weight/bias to zero.

\textbf{\glspl{VNN}} are generated by optimizing weights and biases to maintain their functionality while reducing the number of non-zero weights, as described in~\citep{baninajjarvnn}, and are subsequently used for all networks.

\textbf{Quantized Networks} are obtained by reducing the precision of network weights, converting them from 32-bit floating-point to lower precision. The MNIST and CIFAR10 networks are quantized using post-training float16, int16, int8, and int4 methods provided by~\citep{ugare2022proof}. The same quantization methods are applied to the \glspl{DNN} trained on CHB-MIT and MIT-BIH datasets.

\textbf{Distilled Networks} are networks trained via knowledge distillation to mimic teacher networks' behavior. We consider nine temperatures, i.e., $T=1,\dots,9$, and produce distilled networks for all datasets following~\citep{hinton2015distilling}.

\subsection{Results and Analysis}

We conduct several experiments with our proposed method for assessing 
local implications by establishing
bounds on \glspl{LRPR}. 
We exclusively focus on correctly classified
samples within each test set. 
We consider the widths and depths of the networks when defining
perturbations. We use $\delta=0.001$ and $\delta=0.01$ for the MNIST
and CIFAR10 datasets and experiment with several values for the CHB-MIT
dataset ($\delta$ up to 0.002) and the MIT-BIH dataset ($\delta$ up to
0.4). 
We say we establish local implication of 
$\net{1}$ w.r.t. $\net{2}$ on sample vicinity
$\sphere{\txlayer{0}}{\delta}$
(i.e., $\net{2} \stackrel{\sphere{\txlayer{0}}{\delta}}{\implies} \net{1}$), 
if we show $\relaxprob{\net{1}}{\net{2}}{\txlayer{0}}{\delta}(\class{},\class{j}) \ge 0$ for all pairs $(\class{},\class{j})$, where $\class{}$ is the correct class. 
We omit samples and vicinities when they are clear from the context and write $\net{2} \implies \net{1}$ for short.
Observe that $\relaxprob{\net{1}}{\net{2}}{\txlayer{0}}{\delta}(\class{},\class{j}) \ge 0$ implies $\minprob{\net{1}}{\net{2}}{\txlayer{0}}{\delta}(\class{},\class{j}) \ge 0$, since \(\relaxprob{\net{1}}{\net{2}}{\txlayer{0}}{\delta}(\class{i},\class{j}) \leq \minprob{\net{1}}{\net{2}}{\txlayer{0}}{\delta}(\class{i},\class{j})\) by soundness of our method.
Here, we use zero as a threshold in $\relaxprob{\net{1}}{\net{2}}{\txlayer{0}}{\delta}(\class{},\class{j}) \ge 0$ as it corresponds to checking increases or decreases of \gls{PR} from one network to the other. 
However, our approach can easily accommodate other thresholds. 

Given two compatible networks $\net{2}$ and $\net{1}$ to be compared 
on the $\delta$-vicinities of a set of samples, we state that $\net{1}$  has more established implications if there are more samples for which
we could show $\net{1}$ is implied by $\net{2}$ on the corresponding vicinities
(i.e., we could establish $\net{2} \implies \net{1}$)
than those for which we could establish 
$\net{2}$ is implied by $\net{1}$ (i.e., $\net{1} \implies \net{2}$).
In this context, if $\net{1}$ has more established implications than \net{2}, 
then the number of samples with vicinities where we could establish $\net{1}$
made the correct decision each time $\net{2}$ did is larger than the number of samples with
vicinities where we could establish $\net{2}$
made the correct decision each time $\net{1}$ did.
In other words, we could establish formal local implications for $\net{1}$ w.r.t. $\net{2}$
more often than we could establish it for $\net{2}$ w.r.t. $\net{1}$.

\begin{figure*}[t]
  \centering
  \begin{subfigure}{0.27\textwidth}
    \includegraphics[width=\linewidth]{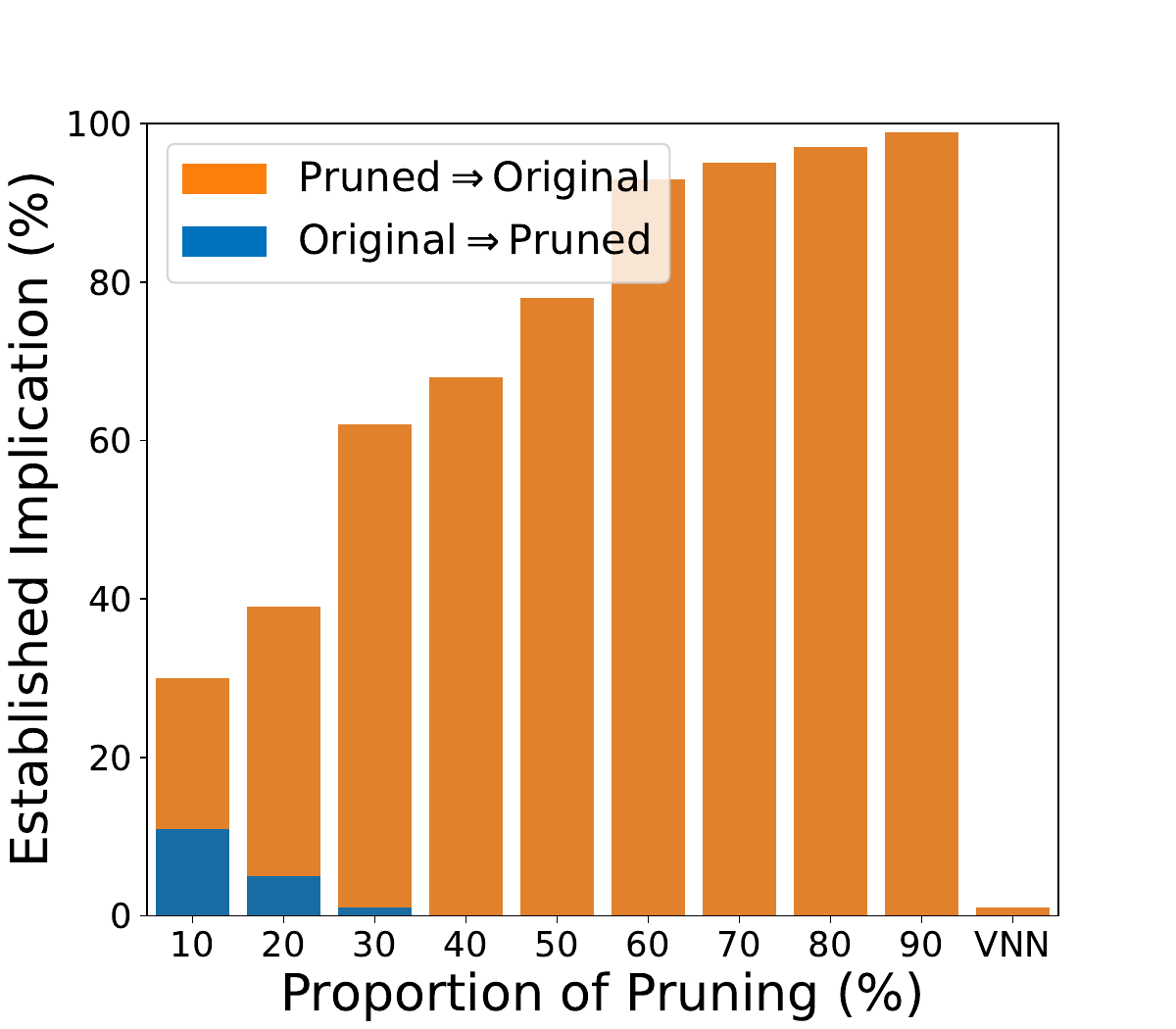}
    \captionsetup{justification=centering}
    \caption{MNIST, Pruned~\citep{ugare2022proof}, ~\citep{baninajjarvnn}}
    \label{FNN_prune_mnist}
  \end{subfigure}
  \begin{subfigure}{0.27\textwidth}
    \includegraphics[width=\linewidth]{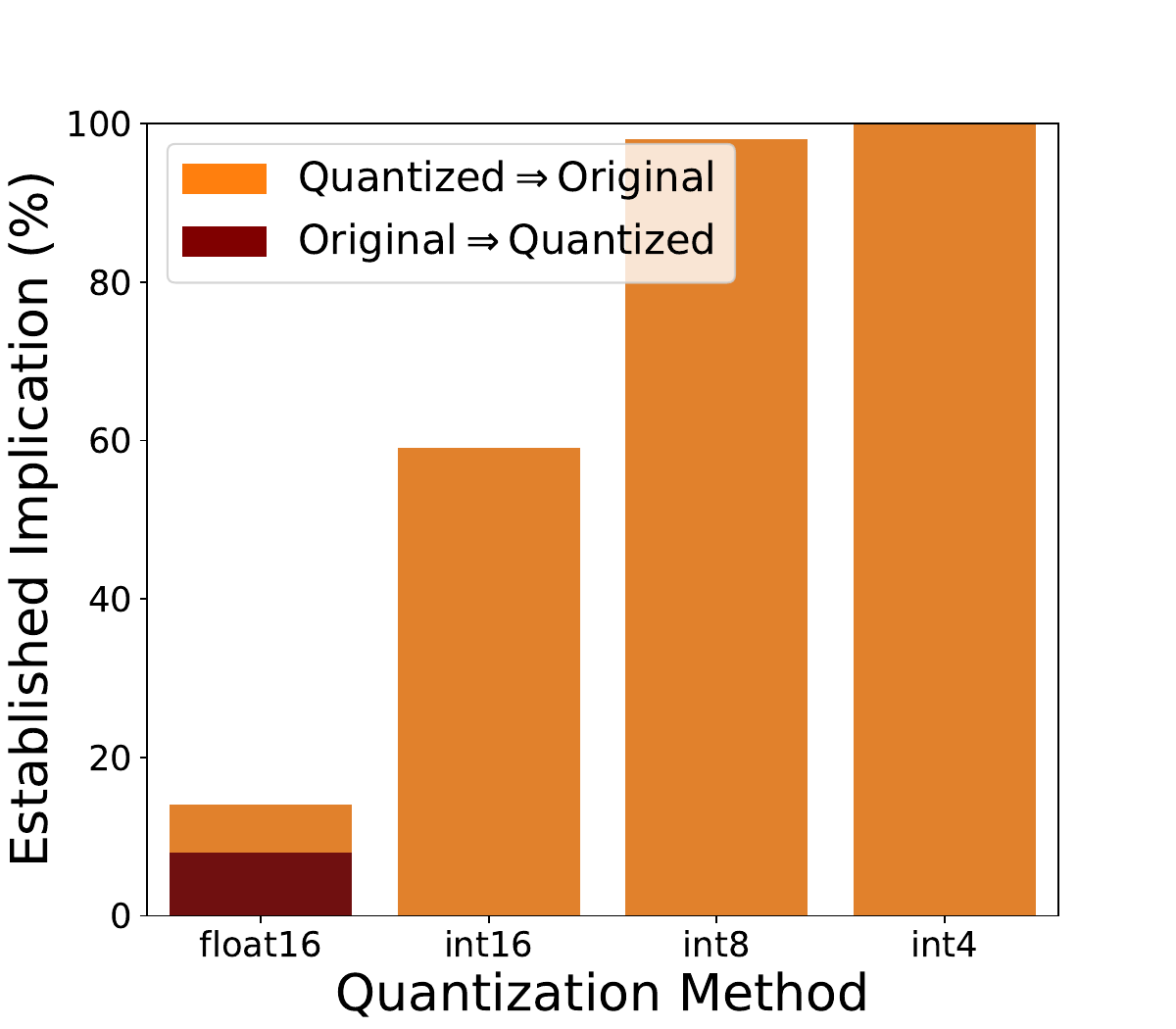}
    \captionsetup{justification=centering}
    \caption{MNIST, Quantized~\citep{ugare2022proof}}
    \label{FNN_quant_mnist}
  \end{subfigure}
  \begin{subfigure}{0.27\textwidth}
    \includegraphics[width=\linewidth]{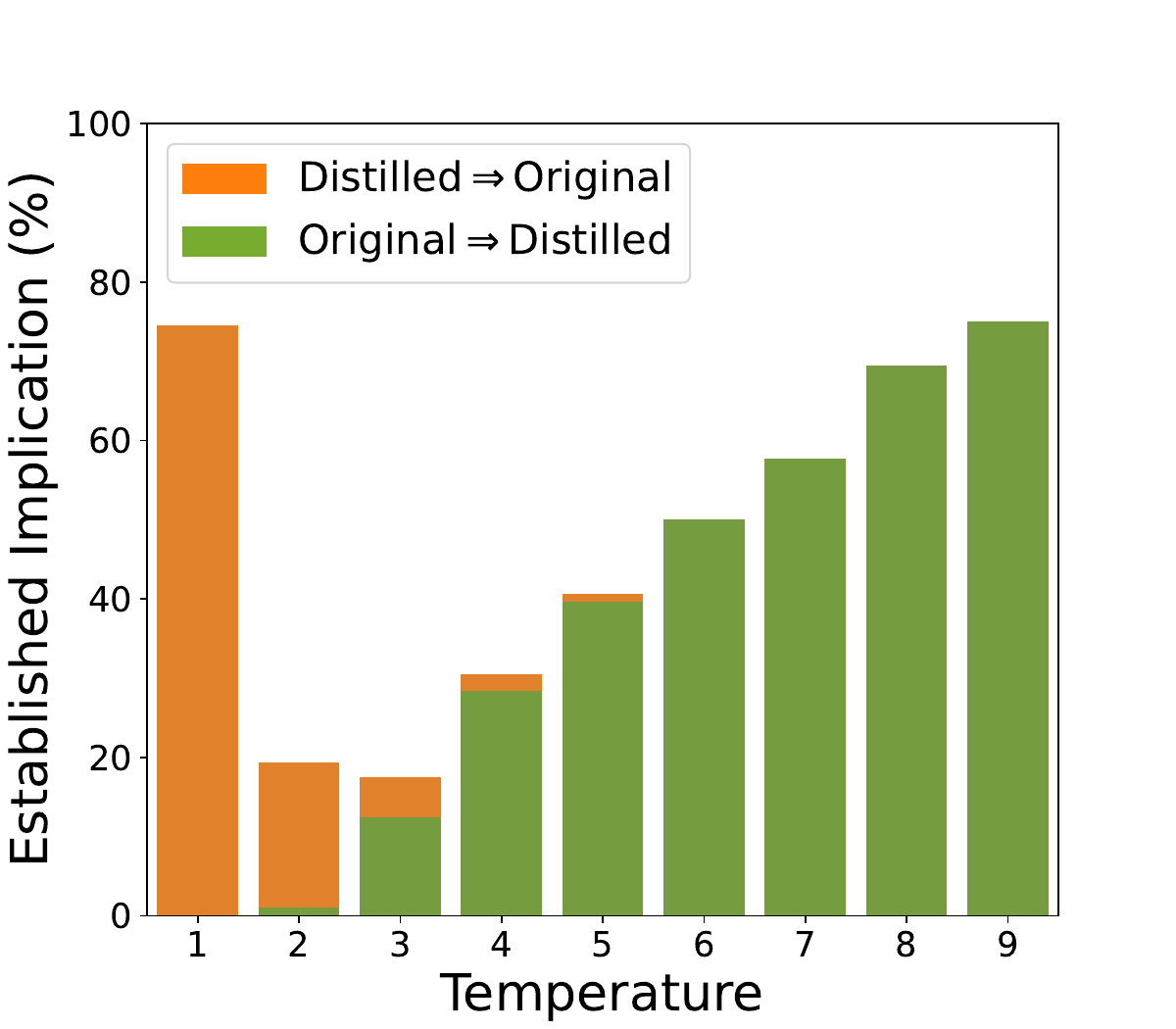}
    \captionsetup{justification=centering}
    \caption{MNIST, Distilled~\citep{hinton2015distilling}}
    \label{FNN_dist_mnist}
  \end{subfigure}

    \begin{subfigure}{0.27\textwidth}
    \includegraphics[width=\linewidth]{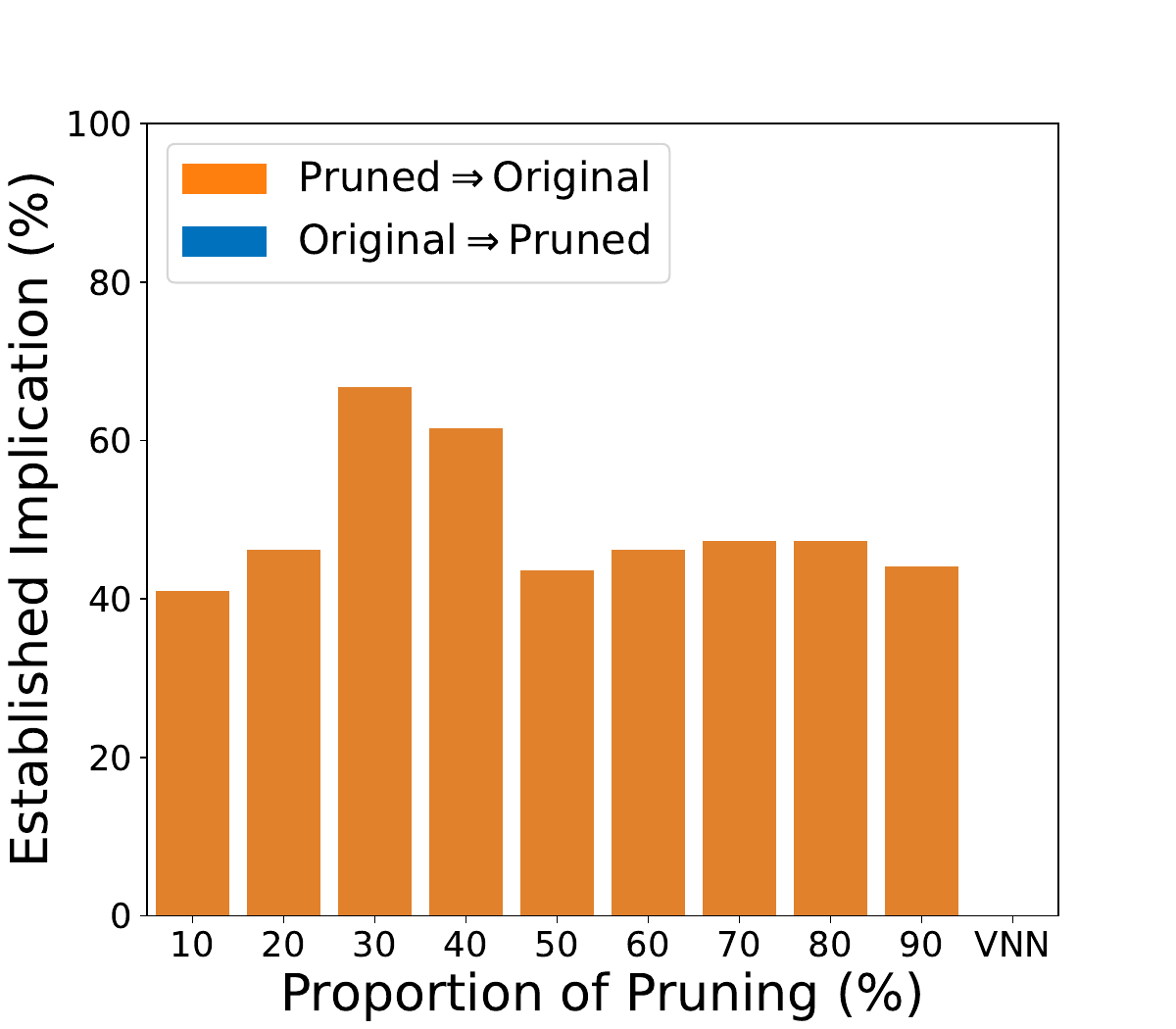}
    \captionsetup{justification=centering}
    \caption{CIFAR10, Pruned~\citep{ugare2022proof},~\citep{baninajjarvnn}}
    \label{FNN_prune_cifar}
  \end{subfigure}
  \begin{subfigure}{0.27\textwidth}
    \includegraphics[width=\linewidth]{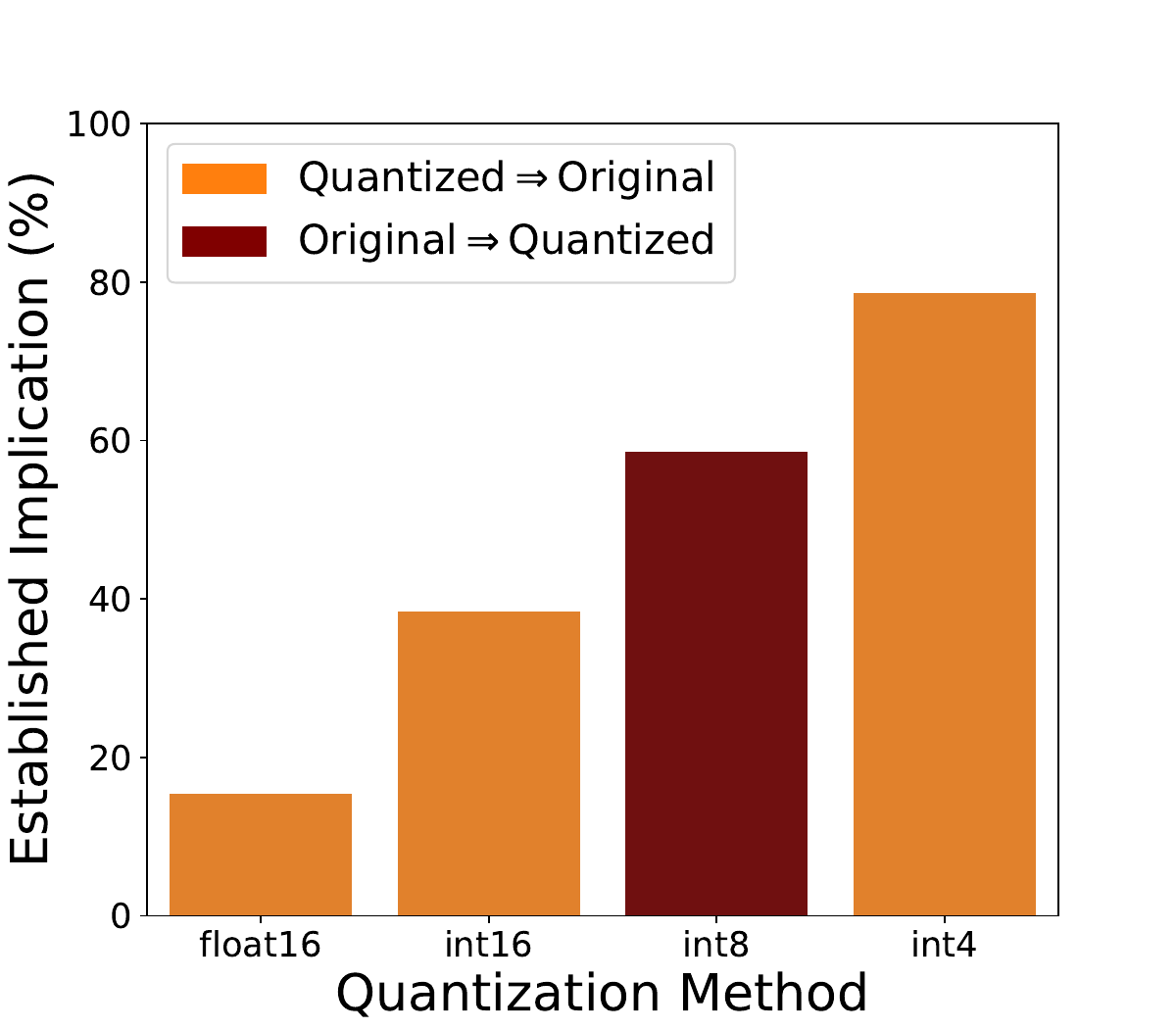}
    \captionsetup{justification=centering}
    \caption{CIFAR10, Quantized~\citep{ugare2022proof}}
    \label{FNN_quant_cifar}
  \end{subfigure}
  \begin{subfigure}{0.27\textwidth}
    \includegraphics[width=\linewidth]{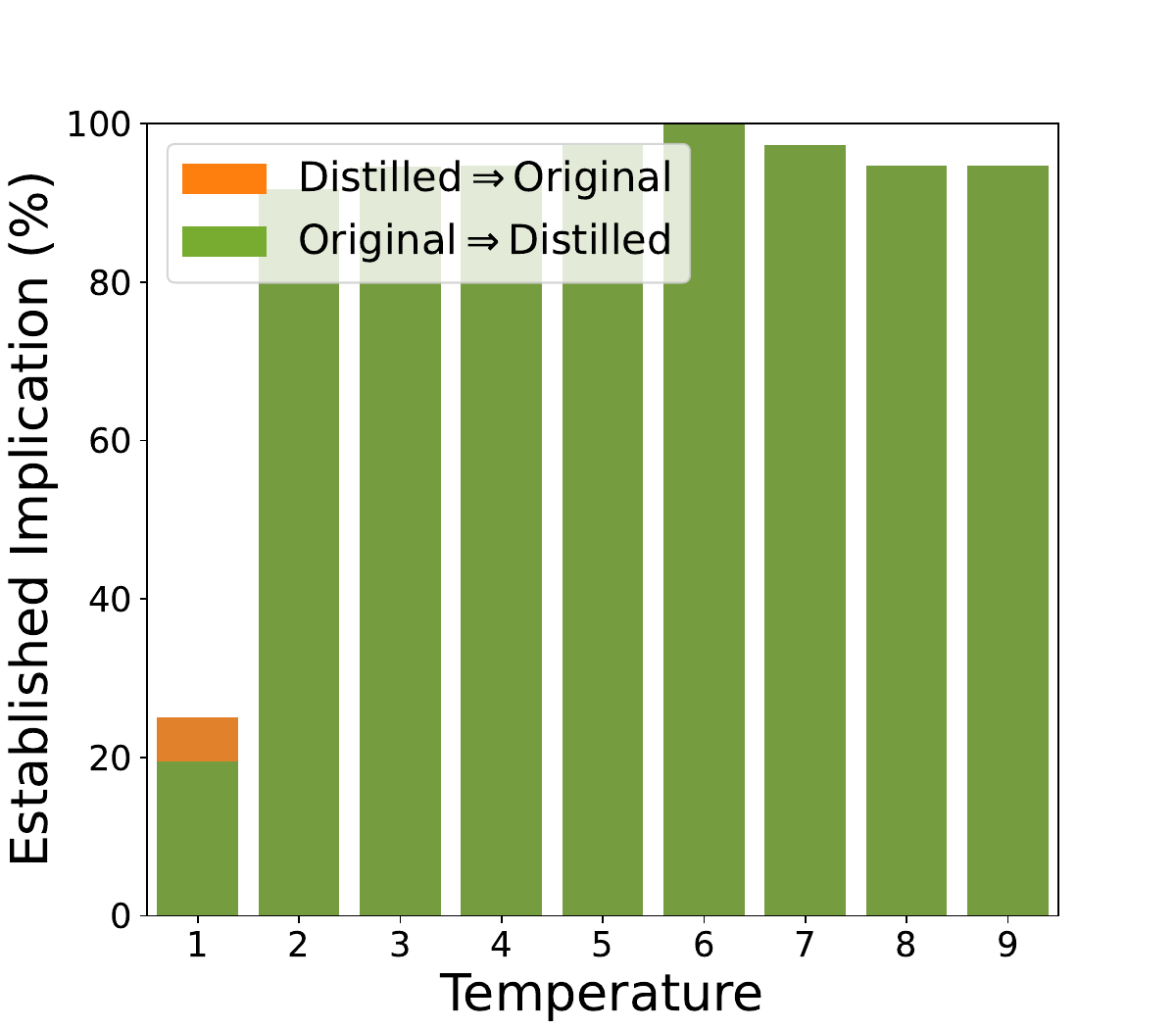}
    \captionsetup{justification=centering}
    \caption{CIFAR10, Distilled~\citep{hinton2015distilling}}
    \label{FNN_dist_cifar}
  \end{subfigure}
  \caption{Stacked bar plots illustrate the established implication of fully-connected \glspl{DNN} trained on the MNIST and CIFAR10 datasets with $\delta = 0.001$. The y-axis represents the percentage of samples in the dataset, referred to as established implication, showing the proportion for which the Compact network implies the Original network (Compact $\implies$ Original) and vice versa. Compact $\implies$ Original means that whenever the Compact network makes a correct decision for a sample, the Original network does as well.} 
  \label{photo_relative_robustness}
\end{figure*}

\subsubsection{Analysis of Formal Local Implication}

Here, we explore formal local implications between original and compact \glspl{DNN} using \gls{LRPR}.

\paragraph{MNIST Dataset.} Figures~\ref{FNN_prune_mnist}--~\ref{FNN_dist_mnist} present 
the results of formal local implication for \glspl{DNN} trained on the MNIST dataset when $\delta = 0.001$. In this figure, the y-axis represents the percentage of samples in the dataset, referred to as established implication, showing the proportion for which the compact network implies the original network on $\delta$-vicinity of considered samples, and the reverse. Compact $\implies$ Original means that whenever the compact network makes a correct decision for a sample, the original network does the same.
Figure~\ref{FNN_prune_mnist} shows an increase in the established implication as the pruning proportion rises, when investigating how pruned networks imply the behavior of the original networks (Pruned $\implies$ Original). There are two potential explanations for this phenomenon. First, the similarity between the original and less-pruned networks may lead to no network having a higher established implication across the entire perturbation neighborhood. Second, our method may establish implication of more samples in more-pruned networks, due to their sparsity. 

The last column of Figure~\ref{FNN_prune_mnist} presents the results of investigating the established implication of the \gls{VNN} generated using~\citep{baninajjarvnn}. The results show that the \gls{VNN} is comparable to that of the original network, as both the implication of the \gls{VNN} by the original network and vice versa are close to zero.

Established implications of quantized networks are depicted in Figure~\ref{FNN_quant_mnist}, 
with quantization precision on the x-axis. Results 
show original networks are more likely to be implied by the quantized ones than vice versa; i.e., there are more cases where we could establish Quantized $\implies$ Original than cases where we could establish Original $\implies$ Quantized.

Figure~\ref{FNN_dist_mnist} illustrates established implication of distilled and original networks implied by each other. The x-axis denotes the temperature of the distilled network, and the y-axis indicates the established implication. This figure demonstrates that the likelihood of distilled networks being implied by original networks increases as temperatures rise, rather than the other way around. The patterns of established implication exhibited by distilled networks differentiate them from pruned and quantized networks, making them a favorable option for creating compact and energy-efficient networks.

The processing time for each sample in the dataset depends on the perturbation, i.e., the value of $\delta$, and the architecture of the original and compact networks. The processing time ($\mu \pm \sigma$) is $15.0 \pm 0.7$ seconds when $\delta=0.001$ and $18.3 \pm 5.2$ seconds when $\delta=0.01$ for pruned and quantized networks. The processing time for distilled networks is $6.7 \pm 0.1$ seconds for $\delta=0.001$ and $6.9 \pm 0.2$ seconds for $\delta=0.01$.

\begin{table*}[]
\caption{The minimum, maximum, and range ($\mu \pm \sigma$) of \glspl{LRPR} for scenarios where implications of Original networks by Compact ones are investigated using independent and joint analyses at $\delta = 0.001$ and $\delta = 0.01$ on MNIST and CIFAR10. Our joint analysis results in a tighter range, with improvement calculated as $\left( 1 - \frac{\text{Range}_{\text{joint}}}{\text{Range}_{\text{ind.}}} \right) \times 100$, indicating the percentage reduction in range.}
\centering
\resizebox{\textwidth}{!}{
\begin{tabular}{ccccccc|cccc}
\toprule[1.5pt]
 &  &  & \multicolumn{4}{c|}{MNIST \cite{mnist}}  & \multicolumn{4}{c}{CIFAR10 \cite{krizhevsky2009learning}}\\ \cline{4-11}
   &    &   & Pruned & Quantized & Distilled & \gls{VNN} & Pruned & Quantized & Distilled & \gls{VNN} \\ \toprule[1.5pt]
\multirow{7}{*}{\rotatebox{90}{$\delta = 0.001$}} & \multirow{2}{*}{Min ($\uparrow$)} & Ind.  & $1.424\scriptstyle{\pm1.519}$ & $-0.485\scriptstyle{\pm0.296}$ & $-19.456\scriptstyle{\pm23.286}$ & $-0.460\scriptstyle{\pm0.322}$ & $0.000\scriptstyle{\pm0.043}$ & $-0.026\scriptstyle{\pm0.028}$ & $-5.927\scriptstyle{\pm5.049}$  & $0.120\scriptstyle{\pm0.301}$   \\
  &  & Joint & $\mathbf{1.875}\scriptstyle{\pm1.680}$ & $\mathbf{0.010}\scriptstyle{\pm0.016}$ & $\mathbf{-19.237}\scriptstyle{\pm23.317}$ & $\mathbf{-0.167}\scriptstyle{\pm0.191}$  & $\mathbf{0.025}\scriptstyle{\pm0.046}$ & $\mathbf{0.001}\scriptstyle{\pm0.002}$ & $\mathbf{-5.907}\scriptstyle{\pm5.040}$ & $\mathbf{0.128}\scriptstyle{\pm0.304}$  \\ 
  \cline{2-11}
& \multirow{2}{*}{Max ($\downarrow$)} & Ind.  & $2.391\scriptstyle{\pm1.898}$  & $0.517\scriptstyle{\pm0.314}$  & $-18.118\scriptstyle{\pm23.399}$  & $0.469\scriptstyle{\pm0.330}$ & $0.055\scriptstyle{\pm0.065}$ & $0.029\scriptstyle{\pm0.030}$ & $-5.471\scriptstyle{\pm4.849}$  & $0.162\scriptstyle{\pm0.315}$   \\
  &  & Joint & $\mathbf{1.940}\scriptstyle{\pm1.706}$  & $\mathbf{0.023}\scriptstyle{\pm0.019}$  & $\mathbf{-18.338}\scriptstyle{\pm23.367}$  & $\mathbf{0.176}\scriptstyle{\pm0.197}$  & $\mathbf{0.030}\scriptstyle{\pm0.049}$ & $\mathbf{0.002}\scriptstyle{\pm0.002}$ & $\mathbf{-5.491}\scriptstyle{\pm4.858}$ & $\mathbf{0.154}\scriptstyle{\pm0.311}$ \\

\cline{2-11}
& \multirow{2}{*}{Range ($\downarrow$)} & Ind. &  $0.967\scriptstyle{\pm2.431}$ & $1.002\scriptstyle{\pm0.432}$ & $1.338\scriptstyle{\pm33.011}$ & $0.929\scriptstyle{\pm0.461}$  & $0.055\scriptstyle{\pm0.078}$ & $0.055\scriptstyle{\pm0.041}$ & $0.456\scriptstyle{\pm7.000}$ & $0.042\scriptstyle{\pm0.436}$ \\
 & & Joint & $\mathbf{0.065}\scriptstyle{\pm2.394}$ & $\mathbf{0.013}\scriptstyle{\pm0.025}$ & $\mathbf{0.899}\scriptstyle{\pm33.011}$ & $\mathbf{0.343}\scriptstyle{\pm0.274}$ & 
 $\mathbf{0.005}\scriptstyle{\pm0.067}$ & $\mathbf{0.001}\scriptstyle{\pm0.003}$ & $\mathbf{0.416}\scriptstyle{\pm7.000}$ &
 $\mathbf{0.026}\scriptstyle{\pm0.435}$ \\ \cline{2-11}
 & \multicolumn{2}{c}{Improvement ($\uparrow$)} & \cellcolor{gray!30} $93.3\%$ & \cellcolor{gray!30} $98.7\%$ & \cellcolor{gray!30} $32.8\%$ & \cellcolor{gray!30} $63.1\%$ & \cellcolor{gray!30} $90.9\%$ & \cellcolor{gray!30} $98.2\%$ & \cellcolor{gray!30} $8.8\% $  & \cellcolor{gray!30} $38.1\%$
\\
\midrule

\multirow{7}{*}{\rotatebox{90}{$\delta = 0.01$}} & \multirow{2}{*}{Min ($\uparrow$)} & Ind.  & $-3.141\scriptstyle{\pm2.357}$ & $-5.165\scriptstyle{\pm3.112}$ & $-25.995\scriptstyle{\pm22.597}$ & $-4.828\scriptstyle{\pm2.923}$  & $-0.257\scriptstyle{\pm0.268}$ & $-0.287\scriptstyle{\pm0.294}$ & $-8.088\scriptstyle{\pm6.097}$ & $-0.079\scriptstyle{\pm0.305}$    \\
  &  & Joint & $\mathbf{1.058}\scriptstyle{\pm1.547}$ & $\mathbf{-0.649}\scriptstyle{\pm0.595}$ & $\mathbf{-23.917}\scriptstyle{\pm22.870}$ & $\mathbf{-2.113}\scriptstyle{\pm1.322}$ & $\mathbf{-0.044}\scriptstyle{\pm0.086}$ & $\mathbf{-0.070}\scriptstyle{\pm0.088}$ & $\mathbf{-7.906}\scriptstyle{\pm5.991}$ & $\mathbf{-0.011}\scriptstyle{\pm0.293}$    \\ 
  \cline{2-11}
& \multirow{2}{*}{Max ($\downarrow$)} & Ind.  & $6.836\scriptstyle{\pm4.243}$  & $5.197\scriptstyle{\pm3.129}$  & $-12.368\scriptstyle{\pm23.812}$  & $4.662\scriptstyle{\pm2.815}$ & $0.314\scriptstyle{\pm0.321}$ & $0.290\scriptstyle{\pm0.296}$ & $-2.966\scriptstyle{\pm4.091}$ & $0.351\scriptstyle{\pm0.433}$    \\
  &  & Joint & $\mathbf{2.635}\scriptstyle{\pm1.941}$  & $\mathbf{0.681}\scriptstyle{\pm0.600}$  & $\mathbf{-14.496}\scriptstyle{\pm23.462}$  & $\mathbf{1.935}\scriptstyle{\pm1.211}$ & $\mathbf{0.102}\scriptstyle{\pm0.115}$ & $\mathbf{0.073}\scriptstyle{\pm0.089}$ & $\mathbf{-3.144}\scriptstyle{\pm4.145}$ & $\mathbf{0.286}\scriptstyle{\pm0.381}$       \\

\cline{2-11}
& \multirow{2}{*}{Range ($\downarrow$)} & Ind. &  $9.977\scriptstyle{\pm4.854}$ & $10.362\scriptstyle{\pm4.413}$ & $13.627\scriptstyle{\pm32.827}$ & $9.490\scriptstyle{\pm4.058}$ & $0.571\scriptstyle{\pm0.418}$ & $0.577\scriptstyle{\pm0.417}$ & $5.122\scriptstyle{\pm7.342}$ & $0.430\scriptstyle{\pm0.530}$ \\
 & & Joint & $\mathbf{1.577}\scriptstyle{\pm2.482}$ & $\mathbf{1.330}\scriptstyle{\pm0.845}$ & $\mathbf{9.421}\scriptstyle{\pm32.764}$ & $\mathbf{4.048}\scriptstyle{\pm1.793}$ & $\mathbf{0.146}\scriptstyle{\pm0.144}$ & $\mathbf{0.143}\scriptstyle{\pm0.125}$ & $\mathbf{4.762}\scriptstyle{\pm7.285}$ & $\mathbf{0.297}\scriptstyle{\pm0.481}$ \\ \cline{2-11}
  &   \multicolumn{2}{c}{Improvement ($\uparrow$)} & \cellcolor{gray!30} $84.2\%$ & \cellcolor{gray!30} $87.2\%$ & \cellcolor{gray!30} $30.9\%$ & \cellcolor{gray!30} $57.3\%$ & \cellcolor{gray!30} $74.4\%$ & \cellcolor{gray!30} $75.2\%$ & \cellcolor{gray!30} $7.0\%$ &  \cellcolor{gray!30} $30.9\% $
\\ 
    \bottomrule[1.5pt]
\end{tabular}} \label{table_min_max_both}
\end{table*}

\paragraph{CIFAR10 Dataset.} 

Figures~\ref{FNN_prune_cifar}--~\ref{FNN_dist_cifar} show the results of investigating formal local implication of \glspl{DNN} trained on the CIFAR10 dataset when $\delta = 0.001$. Although the general patterns in the results of the CIFAR10 \glspl{DNN} are similar to those of the MNIST \glspl{DNN}, a few differences are observed. In Figure~\ref{FNN_quant_cifar}, the quantized network with int8 precision is more likely to be implied by the original network rather than the reverse. This suggests that the reduction in precision does not significantly impair the network’s ability to make the correct decision each time the original one does. Moreover, in Figure~\ref{FNN_dist_cifar}, distilled networks consistently exhibit a higher established implication across different temperatures. This behavior indicates that, in this set of experiments, despite fewer parameters, distilled networks tend to better make the correct decision each time the original network does than those obtained with other compaction schemes. 

The processing time of \glspl{DNN} trained on the CIFAR10 dataset is higher than that of the MNIST \glspl{DNN}, as the number of parameters is larger due to the input size. The processing time ($\mu \pm \sigma$) is $33.1 \pm 2.1$ seconds when $\delta=0.001$ and $43.9 \pm 8.1$ seconds when $\delta=0.01$ for pruned and quantized networks. The processing time of distilled networks is $15.0 \pm 0.9$ and $18.1 \pm 2.6$ seconds for $\delta=0.001$ and $\delta=0.01$, respectively.

\subsubsection{Comparison with Independent Analysis}  

Table~\ref{table_min_max_both} shows the minimum, maximum and range ($\mu \pm {\sigma}$) of \glspl{LRPR} for the scenario where we assess if compact networks imply original ones using independent and joint analyses at $\delta=0.001$ and $\delta=0.01$ for MNIST and CIFAR10 datasets. Results show that our proposed joint analysis consistently produces higher minimum values and lower maximum values, resulting in a tighter range for \glspl{LRPR}. This occurs because considering two networks in the same setting, as in joint analysis, removes unrealistic scenarios that would not occur in reality, thus preventing a minimum lower than the true minimum and a maximum higher than the true maximum.

Here, range refers to the difference between the minimum and maximum values of each method ($\text{Max} - \text{Min}$), where the range calculated by our joint analysis method is consistently lower than that of the independent analysis. Improvement is calculated using $\left( 1 - \frac{\text{Range}_{\text{joint}}}{\text{Range}_{\text{ind.}}} \right) \times 100$, which measures the percentage reduction in range from the independent to the joint analysis. For example, for MNIST with $\delta=0.01$, we have $9.977$ for the independent and $1.577$ as the joint value, the improvement is $\left( 1 - \frac{1.577}{9.977} \right) \times 100 \approx 84.2\%$. This formula standardizes the measurement of relative improvement, indicating that the range of \glspl{LRPR} in the joint analysis is $84.2\%$ narrower than in the independent analysis. This demonstrates that the joint analysis provides a tighter, more consistent range, highlighting its improved precision and reliability.

\subsubsection{Adversarially-Trained Models} 
As discussed earlier, our proposed formulation is applicable to any two neural networks. In this section, we analyze three neural networks, two of which are adversarially-trained models designed to defend against adversarial attacks, specifically \gls{PGD}, using different values of $\epsilon$~\citep{DeepPoly}. The PGD-trained networks have $\epsilon$ values of 0.1 and 0.3, denoted as \gls{PGD}1 and \gls{PGD}3, respectively. 

Table~\ref{pgd_comparison} presents the results of examining the established implications and certified accuracy of the original (non-defended), \gls{PGD}1, and \gls{PGD}3 \glspl{DNN}. In the established implication section, the column $\implies$ row indicates, for example, that in 42\% of the samples of the dataset, \gls{PGD}3 makes a correct decision whenever \gls{PGD}1 does. The results demonstrate that \gls{PGD}1 and \gls{PGD}3 consistently achieve higher established implications than their original counterpart.

These results become even more intriguing when compared to the outcomes of evaluating the robustness of the neural networks using formal verification techniques. For $\delta = 0.001$, the certified accuracy of the original, \gls{PGD}1, and \gls{PGD}3 \glspl{DNN} is $100\%$ when each network is evaluated individually using verification tools. Our formulation reveals that, although robustness evaluations might yield similar results, this does not necessarily imply that the networks behave identically. For example, under a higher perturbation of $\delta = 0.01$, the certified accuracy of the original \gls{DNN} drops to $89\%$, whereas the PGD-trained \glspl{DNN} maintain a certified accuracy of $99\%$. This demonstrates that the results produced by our formulation accurately capture the differences in the networks' behaviors. Further investigations reveal that increasing $\delta$ to 0.04 leads to a more pronounced drop in the certified accuracy of \gls{PGD}1 compared to \gls{PGD}3, with \gls{PGD}1's accuracy falling to 29\%, while \gls{PGD}3's remains at 87\%. This is also reflected in the established implication results for $\delta = 0.001$, where \gls{PGD}1 has 20\% verified \gls{LRPR} with respect to \gls{PGD}3, compared to 42\% for \gls{PGD}3 with respect to \gls{PGD}1.

\begin{table}[ht]
\caption{Comparison of established implication and certified accuracy of the original, \gls{PGD}1, and \gls{PGD}3 \glspl{DNN}. In the established implication section, the column $\implies$ row shows that, e.g., in 42\% of the samples, \gls{PGD}3 makes a correct decision whenever \gls{PGD}1 does.}
\centering
\resizebox{\columnwidth}{!}{
\begin{tabular}{cccc|ccc}
\toprule[1.5pt]
 & \multicolumn{3}{c|}{Established Implication} & \multicolumn{3}{c}{Certified Accuracy} \\ \cline{2-7}
 & Org. & PGD1 & PGD3 & $\delta=0.001$ & $\delta=0.01$ & $\delta=0.04$ \\ \midrule
Org. & - & $0\%$ & $1\%$ & $100\%$ & $89\%$ & $17\%$\\
PGD1 & $57\%$ & - & $20\%$ & $100\%$ & $99\%$ & $29\%$\\ 
PGD3 & $57\%$ & $42\%$ & - & $100\%$ & $99\%$ & $87\%$ \\ \bottomrule[1.5pt]
\end{tabular} }
\label{pgd_comparison}
\end{table}

\begin{figure*}[t]
 \centering

  \begin{subfigure}{0.225\textwidth}
\includegraphics[trim={0cm 0cm 1cm 0.5cm}, clip,width=\linewidth]{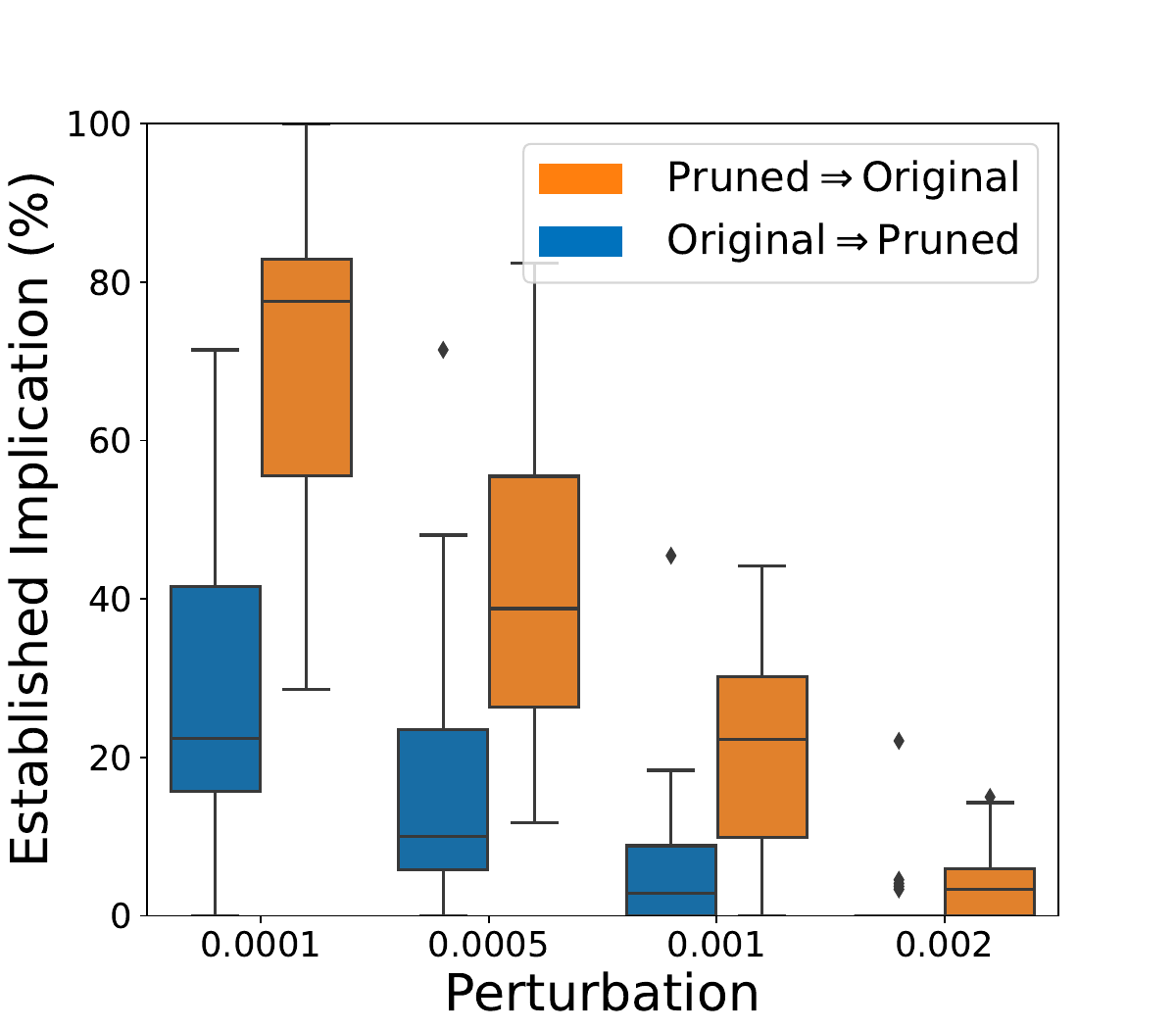}
    \caption{CHB, Pruned~\citep{ugare2022proof}}
    \label{CHB_sparse}
  \end{subfigure}
  \begin{subfigure}{0.225\textwidth}
\includegraphics[trim={0cm 0cm 1cm 0.5cm}, clip,width=\linewidth]{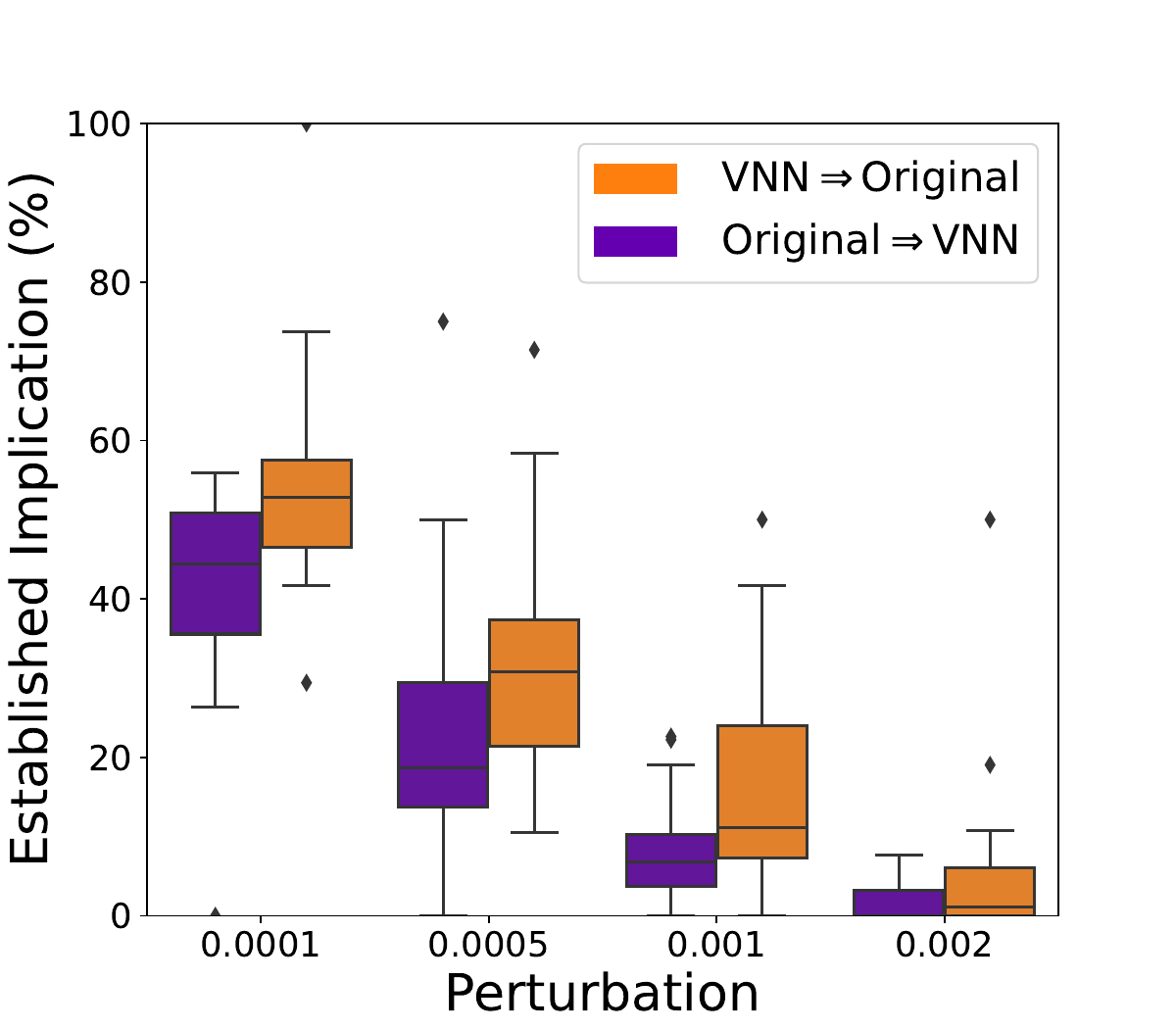}
    \caption{CHB, VNN~\citep{baninajjarvnn}}
    \label{CHB_VNN}
  \end{subfigure}
  \begin{subfigure}{0.225\textwidth}
\includegraphics[trim={0cm 0cm 1cm 0.5cm}, clip,width=\linewidth]{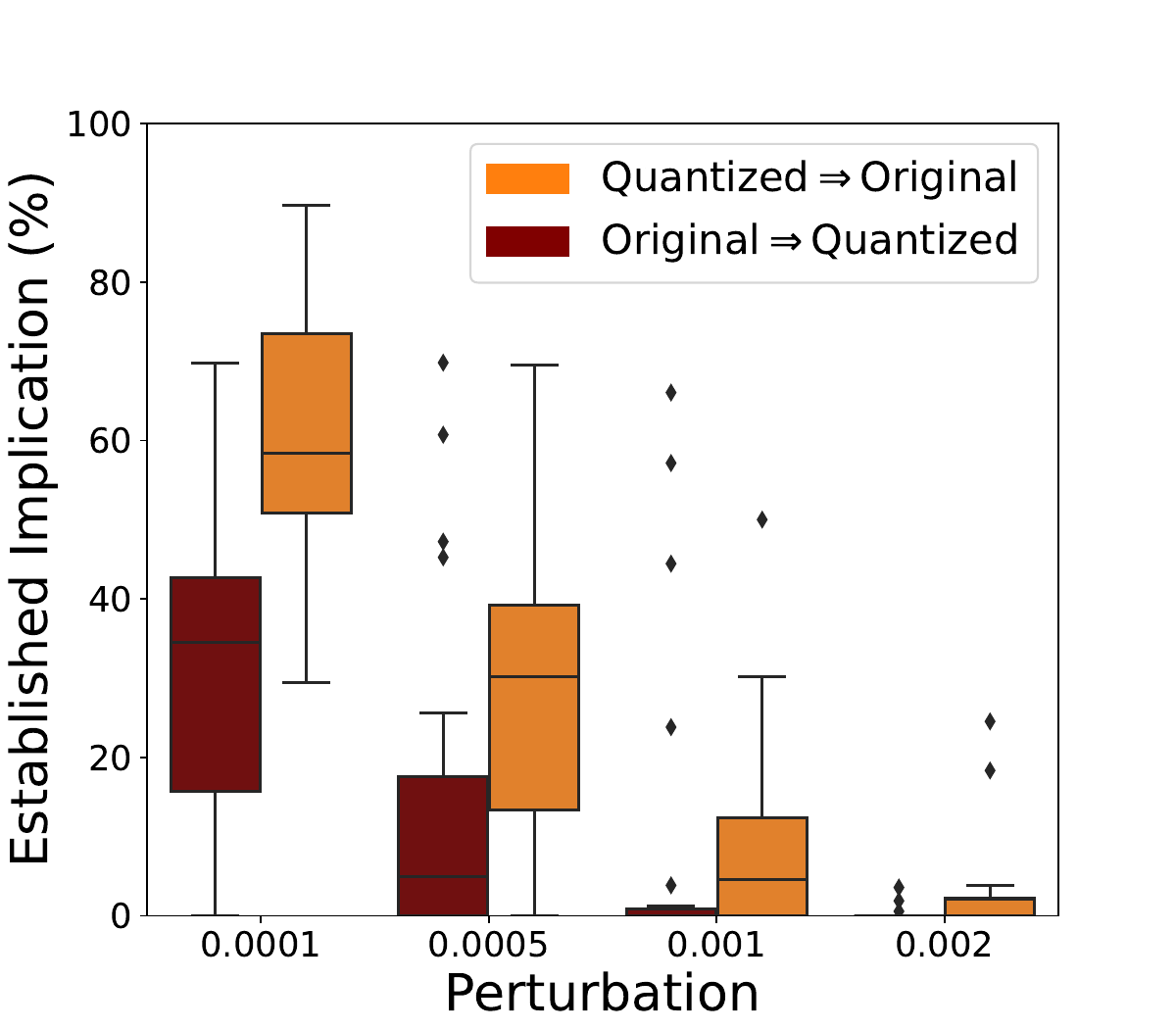}
    \caption{CHB,\! Quantized~\citep{ugare2022proof}}
    \label{CHB_int16}
  \end{subfigure}
  \begin{subfigure}{0.225\textwidth}
\includegraphics[trim={0cm 0cm 1cm 0.5cm}, clip,width=\linewidth]{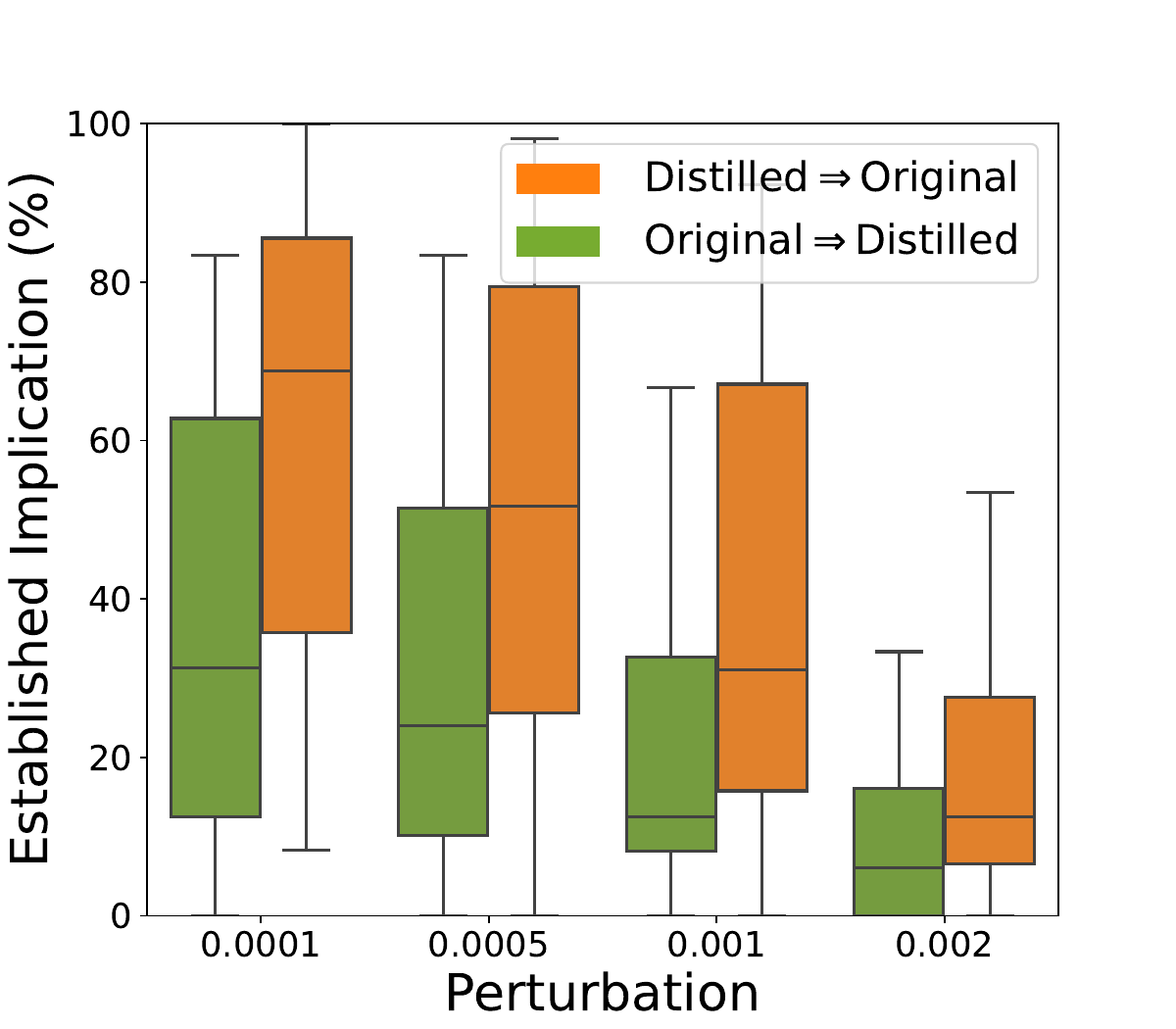}
    \caption{CHB, Distilled~\citep{hinton2015distilling}}
    \label{CHB-T5}
  \end{subfigure}

    \begin{subfigure}{0.225\textwidth}
\includegraphics[trim={0cm 0cm 1cm 0.5cm}, clip,width=\linewidth]{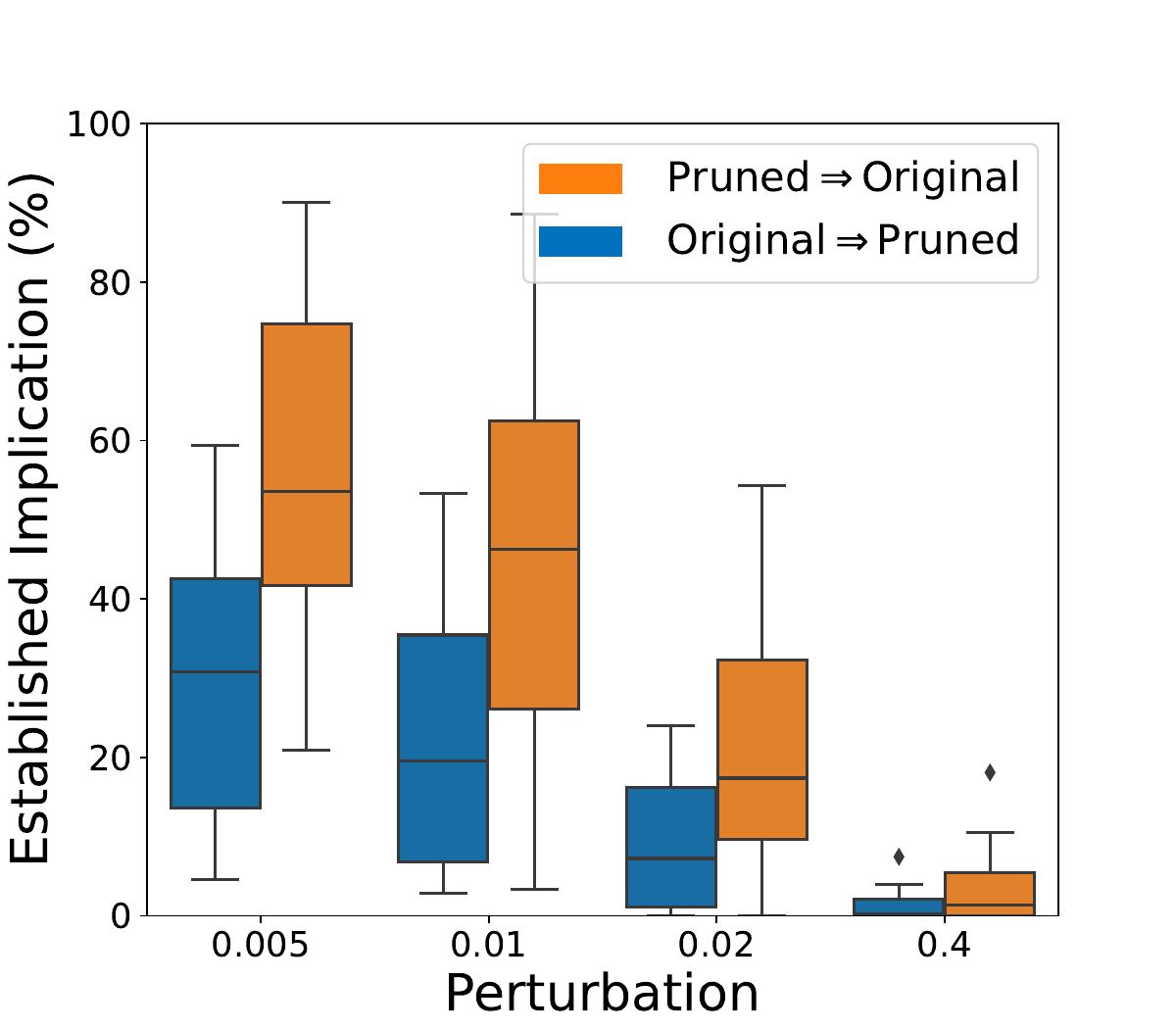}
    \caption{BIH, Pruned~\citep{ugare2022proof}}
    \label{BIH-sparse}
  \end{subfigure}
  \begin{subfigure}{0.225\textwidth}
\includegraphics[trim={0cm 0cm 1cm 0.5cm}, clip,width=\linewidth]{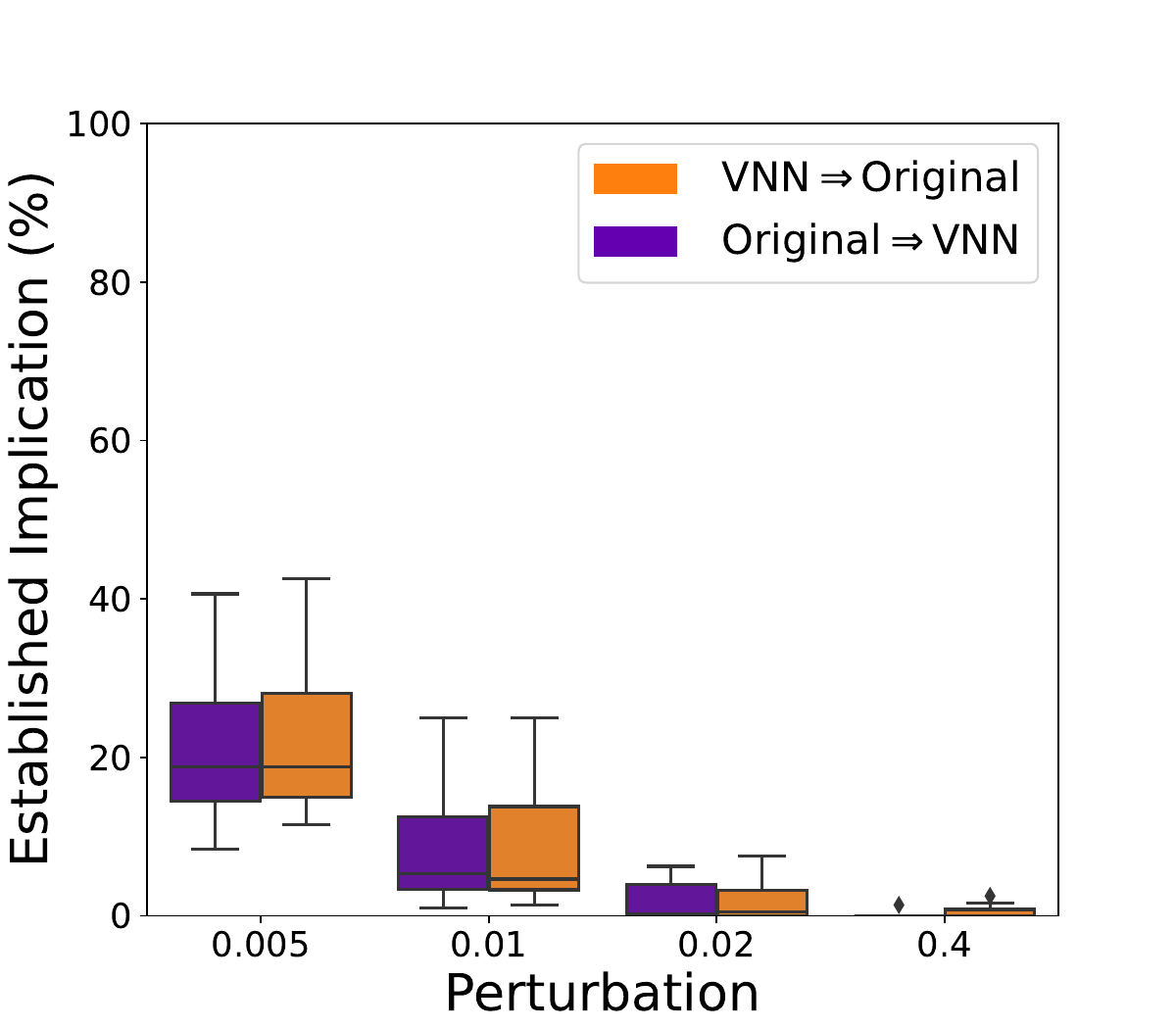}
    \caption{BIH, VNN~\citep{baninajjarvnn}}
    \label{BIH-VNN}
  \end{subfigure}
  \begin{subfigure}{0.225\textwidth}
\includegraphics[trim={0cm 0cm 1cm 0.5cm}, clip,width=\linewidth]{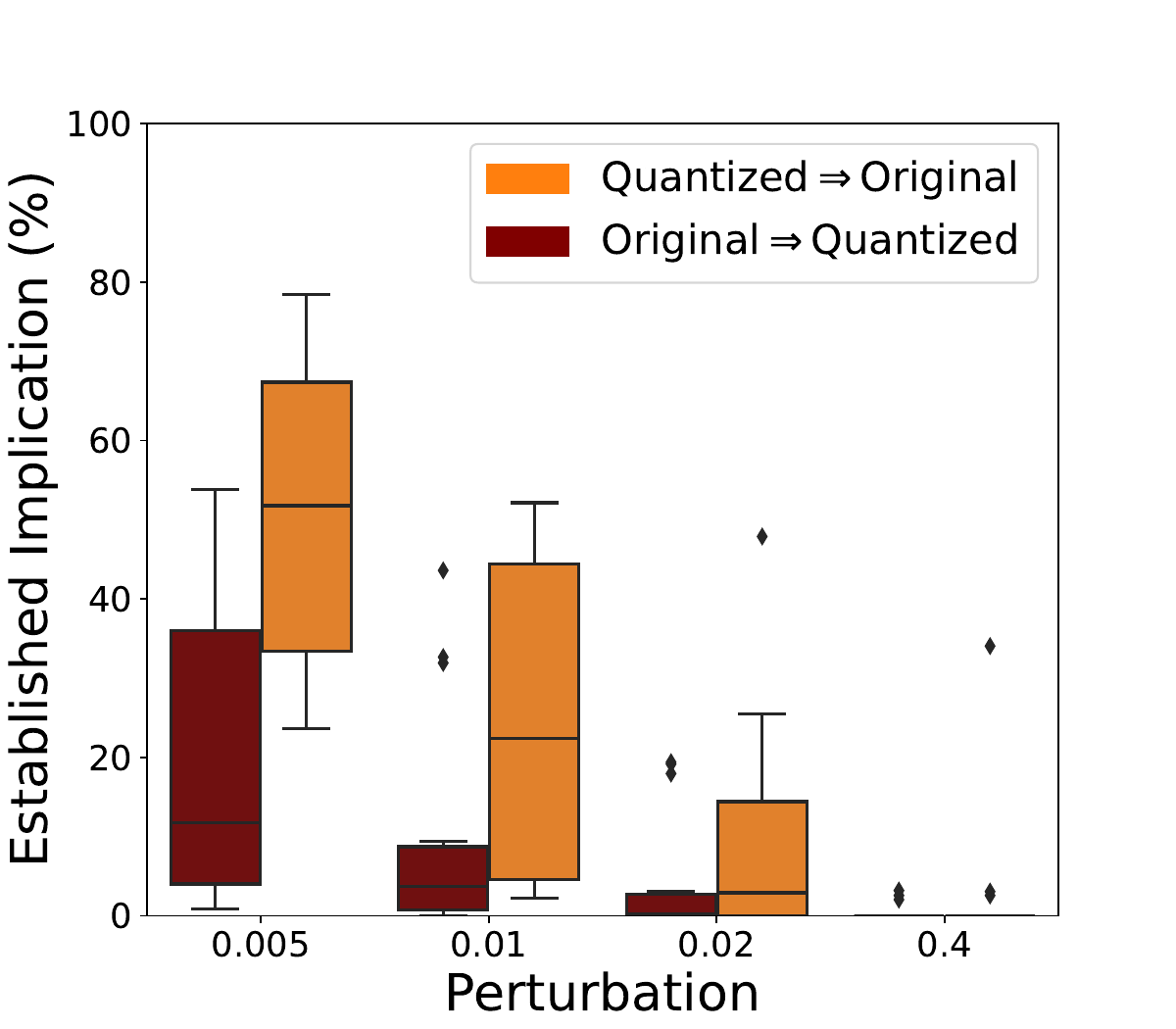}
    \caption{BIH, Quantized~\citep{ugare2022proof}}
    \label{bih-quant16}
  \end{subfigure}
  \begin{subfigure}{0.225\textwidth}
\includegraphics[trim={0cm 0cm 1cm 0.5cm}, clip,width=\linewidth]{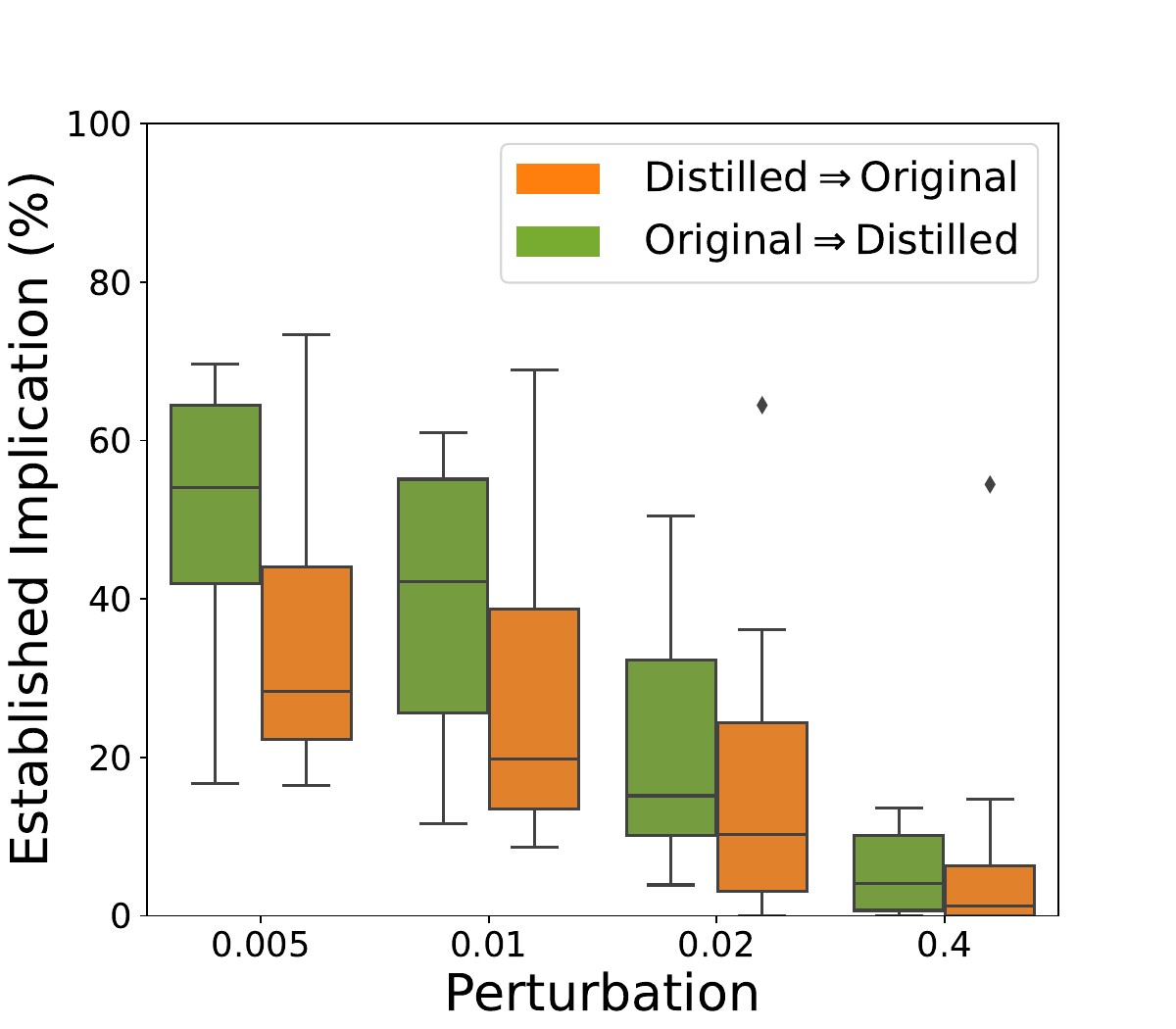}
    \caption{BIH, Distilled~\citep{hinton2015distilling}}
    \label{bih-distil5}
  \end{subfigure} 
 \caption{Box plots illustrate the established implication of convolutional \glspl{DNN} trained on all patients in the CHB-MIT~\citep{shoeb_chb-mit_2010} and MIT-BIH~\citep{mit-bih} datasets for Original and Compact networks. For each patient in the dataset, we evaluate the implication between an Original and a Compact network using the patient’s own data and aggregate the results across all patients to present them in the box plots.} 
  \label{CHB_3}
\end{figure*}

\subsubsection{Two Real-World Medical Datasets}

In this section, we explore the established implication of \glspl{DNN} trained on two real-world medical datasets, the CHB-MIT and MIT-BIH.

\paragraph{CHB-MIT Dataset:} We explore the established implication of convolutional \glspl{DNN} trained on the CHB-MIT dataset to categorize EEG signals of patients with epileptic seizures as captured in Figure~\ref{CHB_sparse}--~\ref{CHB-T5}. Here, the x-axis shows different perturbation values applied to the input of a pair of original and compact networks. The general pattern of the established implication of pruned, quantized, and distilled networks is that we observe original \glspl{DNN} have a higher established implication than their compact counterparts. Moreover, the number of established cases decreases with increasing perturbation. This can be caused by an actual decrease in \gls{LRPR} over a neighborhood, or by exacerbated over-approximation generated by the formulation. However, Figure~\ref{CHB_VNN} shows that the average established implication of \glspl{VNN} is comparable to that of their original counterparts.

\paragraph{{MIT-BIH dataset}:} In this section, we assess the established implication of convolutional \glspl{DNN} trained on the MIT-BIH dataset to categorize ECG signals from patients with cardiac arrhythmia, as demonstrated in Figures~\ref{BIH-sparse}--~\ref{bih-distil5}. Similar to \glspl{DNN} trained on the CHB-MIT dataset, the number of established samples drops as perturbation increases, either due to reduced \gls{LRPR} across a range of perturbed inputs or increased over-approximation generated by the formulation. The behavior of \gls{MBP}-pruned and quantized networks is also similar to that of CHB-MIT \glspl{DNN}, with the established implication of original networks being higher than that of their corresponding pruned and quantized ones. However, the results of \gls{VNN}-pruned networks are slightly different, as the established implication is closer to that of their original counterparts. Moreover, distilled networks display a different pattern, where their established implication is higher than that of their corresponding original networks.

\section{Related Work}
\label{sec:related}

Evaluating and comparing the performance of networks is crucial in machine learning. 
For instance, both verification and adversarial techniques have been used to
compare robustness of networks and their compacted
versions.
The previous studies by~\citep{li2023can}
and~\citep{jordao2021effect} for pruning, and by
~\citep{duncan2020relative} for quantization used established
verification techniques~\citep{huang2017safety,wang2021beta} to separately assess local
robustness of a network and its compacted version.

On the other hand, the work in ~\citep{wang2018adversarial} applies
two white box attacks including Fast Gradient Sign Method
(FGSM)~\citep{goodfellow2014explaining} and
\gls{PGD}~\citep{madry2017towards} to assess local robustness.
None of the previous approaches can formally establish, even in the presence 
of adversarial examples, that a network
makes a correct decision each time the other network does.

In this context, ReLUDiff~\citep{paulsen2020reludiff} and NeuroDiff~\citep{paulsen2020neurodiff} present valuable techniques for analyzing functional differences between neural networks, focusing on verifying whether two models behave identically. As noted by~\citet{paulsen2020reludiff}, many tools focus on single-network behavior, limiting their ability to verify relational properties. In contrast, we establish an implication property that ensures one network is at least as correct as another. Instead of neuron-wise tracking like ReLUDiff and NeuroDiff, our method directly optimizes the final-layer difference, allowing comparison between different architectures.

~\citet{kleine2020verifying} use MILP over clustered input regions to verify symmetric neural network equivalence, but does not capture directional correctness needed for comparing original and compact models. Unlike~\citet{narodytska2018verifying}, who verify single binarized networks via SAT, and~\citet{eleftheriadis2022neural}, who perform symmetric equivalence checking with SMT over the entire input space, our approach targets local directional guarantees between networks, enabling asymmetric verification.

\section{Conclusions}
\label{sec:concl}

In this work, we propose a formulation to compare two networks in relation to each other over an entire input region. Specifically, we establish the foundation for formal local implication between two networks, i.e., $\net{2} \locimplies \net{1}$, within an  entire input region $\sphere{}{}$. In this context, network $\net{1}$ consistently makes a correct decision every time network $\net{2}$ does in the entire input region $\sphere{}{}$. The proposed formulation is relevant in the context of several application domains, e.g., for comparing a trained network and its corresponding compact (e.g., pruned, quantized, distilled) network. We evaluate our formulation using the MNIST, CIFAR10, and two real-world medical datasets, to show its relevance.




\begin{ack}
This work is partially supported by the Wallenberg AI, Autonomous Systems and Software Program (WASP) funded by the Knut and Alice Wallenberg Foundation and by the European Union (EU) Interreg Program. We would like to also thank Kamran Hosseini for his initial input on an earlier draft of this work.

\end{ack}



\bibliography{mybibfile}

\begin{thebibliography}{28}
\providecommand{\natexlab}[1]{#1}
\providecommand{\url}[1]{\texttt{#1}}
\expandafter\ifx\csname urlstyle\endcsname\relax
  \providecommand{\doi}[1]{doi: #1}\else
  \providecommand{\doi}{doi: \begingroup \urlstyle{rm}\Url}\fi

\bibitem[Baghersalimi et~al.(2024)Baghersalimi, Amirshahi, Forooghifar, Teijeiro, Aminifar, and Atienza~Alonso]{baghersalimi2024m2skd}
S.~Baghersalimi, A.~Amirshahi, F.~Forooghifar, T.~Teijeiro, A.~Aminifar, and D.~Atienza~Alonso.
\newblock M2skd: Multi-to-single knowledge distillation of real-time epileptic seizure detection for low-power wearable systems.
\newblock \emph{ACM Transactions on Intelligent Systems and Technology}, 2024.

\bibitem[Baninajjar et~al.(2023)Baninajjar, Hosseini, Rezine, and Aminifar]{baninajjar2023safedeep}
A.~Baninajjar, K.~Hosseini, A.~Rezine, and A.~Aminifar.
\newblock Safedeep: A scalable robustness verification framework for deep neural networks.
\newblock In \emph{ICASSP 2023-2023 IEEE International Conference on Acoustics, Speech and Signal Processing (ICASSP)}, pages 1--5. IEEE, 2023.

\bibitem[Baninajjar et~al.(2024)Baninajjar, Rezine, and Aminifar]{baninajjarvnn}
A.~Baninajjar, A.~Rezine, and A.~Aminifar.
\newblock Vnn: Verification-friendly neural networks with hard robustness guarantees.
\newblock In \emph{Forty-first International Conference on Machine Learning (ICML)}, 2024.

\bibitem[Chiang et~al.(2020)Chiang, Ni, Abdelkader, Zhu, Studer, and Goldstein]{chiang2020certified}
P.-y. Chiang, R.~Ni, A.~Abdelkader, C.~Zhu, C.~Studer, and T.~Goldstein.
\newblock Certified defenses for adversarial patches.
\newblock \emph{Proceedings of International Conference on Learning Representations (ICLR)}, pages 1--17, 2020.

\bibitem[Duncan et~al.(2020)Duncan, Komendantskaya, Stewart, and Lones]{duncan2020relative}
K.~Duncan, E.~Komendantskaya, R.~Stewart, and M.~Lones.
\newblock Relative robustness of quantized neural networks against adversarial attacks.
\newblock In \emph{2020 International Joint Conference on Neural Networks (IJCNN)}, pages 1--8. IEEE, 2020.

\bibitem[Ehlers(2017)]{Ehler:linear}
R.~Ehlers.
\newblock Formal verification of piece-wise linear feed-forward neural networks.
\newblock In D.~D'Souza and K.~Narayan~Kumar, editors, \emph{Automated Technology for Verification and Analysis}, pages 269--286, Cham, 2017. Springer International Publishing.
\newblock ISBN 978-3-319-68167-2.

\bibitem[Eleftheriadis et~al.(2022)Eleftheriadis, Kekatos, Katsaros, and Tripakis]{eleftheriadis2022neural}
C.~Eleftheriadis, N.~Kekatos, P.~Katsaros, and S.~Tripakis.
\newblock On neural network equivalence checking using smt solvers.
\newblock In \emph{International Conference on Formal Modeling and Analysis of Timed Systems}, pages 237--257. Springer, 2022.

\bibitem[Goldberger et~al.(2000)Goldberger, Amaral, Glass, Hausdorff, Ivanov, Mark, Mietus, Moody, Peng, and Stanley]{mit-bih}
A.~L. Goldberger, L.~A. Amaral, L.~Glass, J.~M. Hausdorff, P.~C. Ivanov, R.~G. Mark, J.~E. Mietus, G.~B. Moody, C.-K. Peng, and H.~E. Stanley.
\newblock {MIT}-{BIH} {Arrhythmia} {Database}, 2000.
\newblock URL \url{https://www.physionet.org/content/mitdb/1.0.0/}.

\bibitem[Goodfellow et~al.(2014)Goodfellow, Shlens, and Szegedy]{goodfellow2014explaining}
I.~J. Goodfellow, J.~Shlens, and C.~Szegedy.
\newblock Explaining and harnessing adversarial examples.
\newblock \emph{arXiv preprint arXiv:1412.6572}, 2014.

\bibitem[{Gurobi Optimization, LLC}(2023)]{gurobi}
{Gurobi Optimization, LLC}.
\newblock {Gurobi Optimizer Reference Manual}, 2023.
\newblock URL \url{https://www.gurobi.com}.

\bibitem[Hinton et~al.(2015)Hinton, Vinyals, and Dean]{hinton2015distilling}
G.~Hinton, O.~Vinyals, and J.~Dean.
\newblock Distilling the knowledge in a neural network.
\newblock \emph{arXiv preprint arXiv:1503.02531}, 2015.

\bibitem[Huang et~al.(2017)Huang, Kwiatkowska, Wang, and Wu]{huang2017safety}
X.~Huang, M.~Kwiatkowska, S.~Wang, and M.~Wu.
\newblock Safety verification of deep neural networks.
\newblock In \emph{Computer Aided Verification: 29th International Conference, CAV 2017, Heidelberg, Germany, July 24-28, 2017, Proceedings, Part I 30}, pages 3--29. Springer, 2017.

\bibitem[Jordao and Pedrini(2021)]{jordao2021effect}
A.~Jordao and H.~Pedrini.
\newblock On the effect of pruning on adversarial robustness.
\newblock In \emph{Proceedings of the IEEE/CVF International Conference on Computer Vision}, pages 1--11, 2021.

\bibitem[Kleine~B{\"u}ning et~al.(2020)Kleine~B{\"u}ning, Kern, and Sinz]{kleine2020verifying}
M.~Kleine~B{\"u}ning, P.~Kern, and C.~Sinz.
\newblock Verifying equivalence properties of neural networks with relu activation functions.
\newblock In \emph{International Conference on Principles and Practice of Constraint Programming}, pages 868--884. Springer, 2020.

\bibitem[Krizhevsky(2009)]{krizhevsky2009learning}
A.~Krizhevsky.
\newblock Learning multiple layers of features from tiny images.
\newblock pages 32--33, 2009.

\bibitem[LeCun(1998)]{mnist}
Y.~LeCun.
\newblock The mnist database of handwritten digits.
\newblock \emph{http://yann. lecun.com/exdb/mnist/}, 1998.

\bibitem[LI et~al.(2023)LI, Chen, Li, Li, and Wang]{li2023can}
Z.~LI, T.~Chen, L.~Li, B.~Li, and Z.~Wang.
\newblock Can pruning improve certified robustness of neural networks?
\newblock \emph{Transactions on Machine Learning Research}, 2023.
\newblock ISSN 2835-8856.

\bibitem[Madry et~al.(2018)Madry, Makelov, Schmidt, Tsipras, and Vladu]{madry2017towards}
A.~Madry, A.~Makelov, L.~Schmidt, D.~Tsipras, and A.~Vladu.
\newblock Towards deep learning models resistant to adversarial attacks.
\newblock In \emph{International Conference on Learning Representations}, 2018.

\bibitem[Narodytska et~al.(2018)Narodytska, Kasiviswanathan, Ryzhyk, Sagiv, and Walsh]{narodytska2018verifying}
N.~Narodytska, S.~Kasiviswanathan, L.~Ryzhyk, M.~Sagiv, and T.~Walsh.
\newblock Verifying properties of binarized deep neural networks.
\newblock In \emph{Proceedings of the AAAI Conference on Artificial Intelligence}, volume~32, 2018.

\bibitem[Paulsen et~al.(2020{\natexlab{a}})Paulsen, Wang, and Wang]{paulsen2020reludiff}
B.~Paulsen, J.~Wang, and C.~Wang.
\newblock Reludiff: Differential verification of deep neural networks.
\newblock In \emph{Proceedings of the ACM/IEEE 42nd International Conference on Software Engineering}, pages 714--726, 2020{\natexlab{a}}.

\bibitem[Paulsen et~al.(2020{\natexlab{b}})Paulsen, Wang, Wang, and Wang]{paulsen2020neurodiff}
B.~Paulsen, J.~Wang, J.~Wang, and C.~Wang.
\newblock Neurodiff: scalable differential verification of neural networks using fine-grained approximation.
\newblock In \emph{Proceedings of the 35th IEEE/ACM International Conference on Automated Software Engineering}, pages 784--796, 2020{\natexlab{b}}.

\bibitem[Shoeb(2010)]{shoeb_chb-mit_2010}
A.~Shoeb.
\newblock {CHB}-{MIT} {Scalp} {EEG} {Database}, 2010.
\newblock URL \url{https://physionet.org/content/chbmit/}.

\bibitem[Singh et~al.(2019)Singh, Gehr, P{\"u}schel, and Vechev]{DeepPoly}
G.~Singh, T.~Gehr, M.~P{\"u}schel, and M.~Vechev.
\newblock An abstract domain for certifying neural networks.
\newblock \emph{Proceedings of the ACM on Programming Languages (PACMPL)}, 3:\penalty0 1--30, 2019.

\bibitem[Sopic et~al.(2018{\natexlab{a}})Sopic, Aminifar, Aminifar, and Atienza~Alonso]{Sopic:myocard}
D.~Sopic, A.~Aminifar, A.~Aminifar, and D.~Atienza~Alonso.
\newblock Real-time event-driven classification technique for early detection and prevention of myocardial infarction on wearable systems.
\newblock \emph{IEEE Transactions on Biomedical Circuits and Systems}, 12\penalty0 (5):\penalty0 982--992, 2018{\natexlab{a}}.

\bibitem[Sopic et~al.(2018{\natexlab{b}})Sopic, Aminifar, and Atienza]{8351728}
D.~Sopic, A.~Aminifar, and D.~Atienza.
\newblock e-glass: A wearable system for real-time detection of epileptic seizures.
\newblock In \emph{IEEE International Symposium on Circuits and Systems (ISCAS)}, pages 1--5, 2018{\natexlab{b}}.

\bibitem[Ugare et~al.(2022)Ugare, Singh, and Misailovic]{ugare2022proof}
S.~Ugare, G.~Singh, and S.~Misailovic.
\newblock Proof transfer for fast certification of multiple approximate neural networks.
\newblock \emph{Proceedings of the ACM on Programming Languages}, 6\penalty0 (OOPSLA1):\penalty0 1--29, 2022.

\bibitem[Wang et~al.(2018)Wang, Ding, Huang, Cao, and Lui]{wang2018adversarial}
L.~Wang, G.~W. Ding, R.~Huang, Y.~Cao, and Y.~C. Lui.
\newblock Adversarial robustness of pruned neural networks.
\newblock 2018.

\bibitem[Wang et~al.(2021)Wang, Zhang, Xu, Lin, Jana, Hsieh, and Kolter]{wang2021beta}
S.~Wang, H.~Zhang, K.~Xu, X.~Lin, S.~Jana, C.-J. Hsieh, and J.~Z. Kolter.
\newblock Beta-crown: Efficient bound propagation with per-neuron split constraints for neural network robustness verification.
\newblock \emph{Advances in Neural Information Processing Systems}, 34:\penalty0 29909--29921, 2021.

\end{thebibliography}

\newpage
\appendix
\onecolumn
\section{Appendix} \label{appendix}

\subsection{Proofs from Section \ref{sec:method}}
\label{lab:proofs}

\begin{lemma}
  \label{lem:ln:appendix}
  Let $(\class{i},{\class{j}})$ be a pair of
  classes of compatible \glspl{DNN} \net{1} and \net{2}.
  Assume common input  $\xlayer{0}=\ylayer{0}$.
  Suppose \net{1} has $N_1+1$ layers 
  $\xlayer{0}, \ldots \xlayer{N_1+1}$ and \net{2}
  has $N_2+1$ layers $\ylayer{0}, \ldots \ylayer{N_2+1}$. 
  Then: 
\[
\ln\left(\rpr{\net{1}}{\net{2}}{\xlayer{0}}(\class{i},{\class{j}})\right) 
=(\subxlayer{{\class{i}}}{N_1} - \subxlayer{\class{j}}{N_1})-(\subylayer{{\class{i}}}{N_2} - \subylayer{\class{j}}{N_2})
\]
\end{lemma}
\begin{proof}
  By applying the $\ln$ function on the definition of $\rpr{\net{1}}{\net{2}}{\xlayer{0}}{(\class{i},\class{j})}$:  
\[
\begin{array}{ll}
  \ln\left(\rpr{\net{1}}{\net{2}}{\xlayer{0}}(\class{i},{\class{j}})\right) = &  \ln\left(\frac{\left(\sm_{\class{i}}(\xlayer{N_1})\right) \cdot \left(\sm_{{\class{j}}}(\ylayer{N_2})\right)}{\left(\sm_{{\class{j}}}(\xlayer{N_1})\right) \cdot \left(\sm_{\class{i}}(\ylayer{N_2})\right)}\right)
    = \ln\left( \frac{\frac{\e^{{\subxlayer{{\class{i}}}{N_1}}}}{\sum_{u=1}^{n^{\net{1}}_{N_1}}\e^{{\subxlayer{u}{N_1}}}} \cdot \frac{\e^{{\subylayer{{\class{j}}}{N_2}}}}{\sum_{u=1}^{n^{\net{2}}_{N_2}}\e^{{\subylayer{u}{N_2}}}}}{\frac{\e^{{\subxlayer{{\class{j}}}{N_1}}}}{\sum_{u=1}^{n^{\net{1}}_{N_1}}\e^{{\subxlayer{u}{N_1}}}} \cdot \frac{\e^{{\subylayer{{\class{i}}}{N_2}}}}{\sum_{u=1}^{n^{\net{2}}_{N_2}}\e^{{\subylayer{u}{N_2}}}}} \right) \\ \\
   = \ln\left( \frac{\e^{{\subxlayer{{\class{i}}}{N_1}}} \cdot \e^{{\subylayer{{\class{j}}}{N_2}}}}{\e^{{\subxlayer{{\class{j}}}{N_1}}} \cdot \e^{{\subylayer{{\class{i}}}{N_2}}}} \right)
   = &  {\subxlayer{{\class{i}}}{N_1}} + {\subylayer{{\class{j}}}{N_2}} - ( {\subxlayer{{\class{j}}}{N_1}} + {\subylayer{{\class{i}}}{N_2}} )
   = {\subxlayer{{\class{i}}}{N_1}} - {\subxlayer{{\class{j}}}{N_1}} - ({\subylayer{{\class{i}}}{N_2}} - {\subylayer{{\class{j}}}{N_2}}) 
\end{array}
\]
\end{proof}

\begin{corollary}
  \label{cor:guaranteed:appendix}
  Let $(\class{i},{\class{j}})$ be a pair of classes of
  compatible \glspl{DNN} \net{1} and \net{2}
  (with resp. $N_1+1$ and $N_2+1$ layers).
  Assume common input $\txlayer{0}=\tylayer{0}$ and
  perturbation $\delta$.
  For $\xlayer{0}=\ylayer{0}$ with $\xlayer{0}\in\sphere{\txlayer{0}}{\delta}$,
  let $\xlayer{0}, \ldots \xlayer{N_1+1}$
  be layers in \net{1} and $\ylayer{0}, \ldots
  \ylayer{N_2+1}$ be layers in \net{2}.
  Then: 
  \[
\begin{array}{l}
  \ln\left(\lrpr{\net{1}}{\net{2}}{\txlayer{0}}{\delta}{\class{i},\class{j}}\right) ~ = ~
  {min_{\xlayer{0}\in\sphere{\txlayer{0}}{\delta}} \left((\subxlayer{\class{i}}{N_1} - \subxlayer{\class{j}}{N_1}) -  (\subylayer{\class{i}}{N_2} - \subylayer{\class{j}}{N_2})\right) } \\
  \\
\ln\left(\urpr{\net{1}}{\net{2}}{\txlayer{0}}{\delta}{\class{i},\class{j}}\right)  ~ = ~
{max_{\xlayer{0}\in\sphere{\txlayer{0}}{\delta}} \left((\subxlayer{\class{i}}{N_1} - \subxlayer{\class{j}}{N_1}) -  (\subylayer{\class{i}}{N_2} - \subylayer{\class{j}}{N_2})\right)}
   \end{array}
\]
\end{corollary}
\begin{proof}
  By applying the $\ln$ function on $\lrpr{\net{1}}{\net{2}}{\txlayer{0}}{\delta}{\class{i},\class{j}}$ and using Lemma \ref{lem:ln:appendix}:  
\[
\begin{array}{ll}
  \ln\left(\lrpr{\net{1}}{\net{2}}{\txlayer{0}}{\delta}{\class{i},\class{j}}\right)  &
  =~min\left\{{\ln\left(\Pi^{{\net{1}}\mid{\net{2}}}_{\xlayer{0}}(\class{i},\class{j})\right)}~|~{\xlayer{0}\in\sphere{\txlayer{0}}{\delta}}\right\} \\ \\
  & =~
  min\left\{{\left((\subxlayer{\class{i}}{N_1} - \subxlayer{\class{j}}{N_1}) -  (\subylayer{\class{i}}{N_2} - \subylayer{\class{j}}{N_2})\right)}~
  |~{\xlayer{0}\in\sphere{\txlayer{0}}{\delta}}\right\} \\ \\
  & =~
    {min_{\xlayer{0}\in\sphere{\txlayer{0}}{\delta}} \left((\subxlayer{\class{i}}{N_1} - \subxlayer{\class{j}}{N_1}) -  (\subylayer{\class{i}}{N_2} - \subylayer{\class{j}}{N_2})\right) } \\ \\
\end{array}
\]
\[
\begin{array}{ll}    
  \ln\left(\urpr{\net{1}}{\net{2}}{\txlayer{0}}{\delta}{\class{i},\class{j}}\right)  
  & =~ max\left\{{\ln\left(\Pi^{{\net{1}}\mid{\net{2}}}_{\xlayer{0}}(\class{i},\class{j})\right)}~|~{\xlayer{0}\in\sphere{\txlayer{0}}{\delta}}\right\} \\ \\
  & =~max\left\{{\left((\subxlayer{\class{i}}{N_1} - \subxlayer{\class{j}}{N_1}) -  (\subylayer{\class{i}}{N_2} - \subylayer{\class{j}}{N_2})\right)}~
  |~{\xlayer{0}\in\sphere{\txlayer{0}}{\delta}}\right\} \\ \\
  & =~ {max_{\xlayer{0}\in\sphere{\txlayer{0}}{\delta}} \left((\subxlayer{\class{i}}{N_1} - \subxlayer{\class{j}}{N_1}) -  (\subylayer{\class{i}}{N_2} - \subylayer{\class{j}}{N_2})\right) } \\
\end{array}
\] \end{proof}

\begin{theorem}
  \label{theo:relaxed_bounds:appendix}
  Let $(\class{i},{\class{j}})$ be a pair of classes of compatible
  \glspl{DNN} \net{1} and \net{2}.  Assume a neighborhood
  $\sphere{\txlayer{0}}{\delta}$ and let
  $\relaxprob{\net{1}}{\net{2}}{\txlayer{0}}{\delta}(\class{i},\class{j})$
  (resp. $\relaxprob{\net{2}}{\net{1}}{\txlayer{0}}{\delta}(\class{i},\class{j})$)
  be a solution to the relaxed minimization problem corresponding to
  \gls{LRPR} of \net{1} w.r.t. \net{2} (resp. \net{2} w.r.t. \net{1}). Then:
  {\small
  \[
  \begin{array}{lllll}
    \relaxprob{\net{1}}{\net{2}}{\txlayer{0}}{\delta}(\class{i},\class{j}) 
    & \leq & 
    \ln\left(\lrpr{\net{1}}{\net{2}}{\txlayer{0}}{\delta}{\class{i},\class{j}}\right)
    & &
    \\
    \\
    \textrm{ and }
    & &
    \ln\left(\urpr{\net{1}}{\net{2}}{\txlayer{0}}{\delta}{\class{i},\class{j}}\right)
    & \leq &  
    -\relaxprob{\net{2}}{\net{1}}{\txlayer{0}}{\delta}(\class{i},\class{j})
  \end{array}
  \]
  }
\end{theorem}
\begin{proof}
 As mentioned in Section~\ref{sec:method} any lower bound obtained for the relaxed problem is guaranteed to
be smaller than a solution for the original minimization problem. Since $
 min\left\{\lnbigpi{\net{1}}{\net{2}}{\txlayer{0}}{\delta}{\class{i},\class{j}}\right\}
 \geq
 \relaxprob{\net{1}}{\net{2}}{\txlayer{0}}{\delta}(\class{i},\class{j})$,
 we get:
   \[
   \begin {array}{l}
   min\left\{\lnbigpi{\net{1}}{\net{2}}{\txlayer{0}}{\delta}{\class{i},\class{j}}\right\}  \\
   =\ln\left(\lrpr{\net{1}}{\net{2}}{\txlayer{0}}{\delta}{\class{i},\class{j}}\right) \geq \relaxprob{\net{1}}{\net{2}}{\txlayer{0}}{\delta}(\class{i},\class{j})
   \end {array}{}
   \]
 And similarly in a symmetric manner:
 \[
 \begin{array}{l}

    min\left\{\lnbigpi{\net{2}}{\net{1}}{\txlayer{0}}{\sm}{\class{i},\class{j}}\right\} \\
    =min\left\{\bigpiopenedxy{\net{2}}{\net{1}}{\txlayer{0}}{\sm}{\class{i},\class{j}}\right\}  
    =min\left\{-\bigpiopenedyx{\net{1}}{\net{2}}{\txlayer{0}}{\sm}{\class{i},\class{j}}\right\} \\
    = -max\left\{\bigpiopenedyx{\net{1}}{\net{2}}{\txlayer{0}}{\sm}{\class{i},\class{j}}\right\} 
    =-max\left\{\lnbigpi{\net{1}}{\net{2}}{\txlayer{0}}{\sm}{\class{i},\class{j}}\right\} \\
    = -\ln\left(\urpr{\net{1}}{\net{2}}{\txlayer{0}}{\sm}{\class{i},\class{j}}\right) \geq \relaxprob{\net{2}}{\net{1}}{\txlayer{0}}{\sm}(\class{i},\class{j})
 \end{array}{}
  \]

 Then we can deduce:

   \[
  \begin{array}{lllll}
    \relaxprob{\net{1}}{\net{2}}{\txlayer{0}}{\delta}(\class{i},\class{j}) 
    & \leq & 
    \ln\left(\lrpr{\net{1}}{\net{2}}{\txlayer{0}}{\delta}{\class{i},\class{j}}\right)
    & &
    \\
    \\
    \textrm{ and }
    & &
    \ln\left(\urpr{\net{1}}{\net{2}}{\txlayer{0}}{\delta}{\class{i},\class{j}}\right)
    & \leq &  
    -\relaxprob{\net{2}}{\net{1}}{\txlayer{0}}{\delta}(\class{i},\class{j})
  \end{array}
  \]
\end{proof}

\begin{corollary}
Let $(\class{i},{\class{j}})$ be a pair of classes of compatible
  \glspl{DNN} \net{1} and \net{2}.  Assume a neighborhood
  $\sphere{\txlayer{0}}{\delta}$ and let
  $\relaxprob{\net{1}}{\net{2}}{\txlayer{0}}{\delta}(\class{i},\class{j})$
  be a solution to the relaxed minimization problem corresponding to
  \gls{LRPR} of \net{1} w.r.t. \net{2}. If $\relaxprob{\net{1}}{\net{2}}{\txlayer{0}}{\delta}(\class{i},\class{j}) > 0$, then for all perturbed inputs $\xlayer{0} \in \sphere{\txlayer{0}}{\delta}$ for which $\net{2}$ correctly classifies $\xlayer{0}$, then $\net{1}$ also correctly classifies $\xlayer{0}$. That is, $\net{2} \locimplies{} \net{1}$ on the entire neighborhood $\sphere{\txlayer{0}}{\delta}$.
\end{corollary}
\begin{proof}

Suppose $\relaxprob{\net{1}}{\net{2}}{\txlayer{0}}{\delta}(\class{i},\class{j}) > 0$ on a neighborhood around a sample $\txlayer{0}$. Four possible scenarios are possible:
\begin{itemize}
    \item If a sample $\xlayer{0}$ in the neighborhood of $\txlayer{0}$ is classified incorrectly by both networks. We still have $\pr{\net{1}}{\xlayer{0}}(\class{i},\class{j}) > \pr{\net{2}}{\xlayer{0}}(\class{i},\class{j})$ for $\xlayer{0}$, which means that $\net{1}$ is closer to correctly classifying $\xlayer{0}$ than $\net{2}$. So, you can trust $\net{1}$ if you did trust $\net{2}$.

    \item If a sample $\xlayer{0}$ in the neighborhood of $\txlayer{0}$ is classified correctly by $\net{1}$, i.e., $\pr{\net{1}}{\xlayer{0}}(\class{i},\class{j}) > 0$, and incorrectly by $\net{2}$, i.e., $\pr{\net{2}}{\xlayer{0}}(\class{i},\class{j}) < 0$, then again you can trust $\net{1}$ if you did trust $\net{2}$ as $\net{1}$ is correct despite $\net{2}$ making an incorrect prediction.

    \item If a sample $\xlayer{0}$ in the neighborhood of $\txlayer{0}$ is classified incorrectly by $\net{1}$ and correctly by $\net{2}$, then $\net{1}$ is doing worse than $\net{2}$. Since $\net{1}$ classifies the $\xlayer{0}$ incorrectly, then $\pr{\net{1}}{\xlayer{0}}(\class{i},\class{j}) < 0$. Also, $\net{2}$ classifies the sample $\xlayer{0}$ correctly implies $\pr{\net{2}}{\xlayer{0}}(\class{i},\class{j}) > 0$. As a result, the lower bound $\relaxprob{\net{1}}{\net{2}}{\txlayer{0}}{\delta}(\class{i},\class{j})$  on the neighborhood has to be negative, which is excluded by our assumption $\relaxprob{\net{1}}{\net{2}}{\txlayer{0}}{\delta}(\class{i},\class{j}) > 0$ .

    \item If sample $\xlayer{0}$ in the neighborhood is classified correctly by both networks with $\relaxprob{\net{1}}{\net{2}}{\txlayer{0}}{\delta}(\class{i},\class{j}) > 0$, then we know that the ratio of probabilities associated to the correct decision for $\xlayer{0}$ by $\net{1}$ is larger than the ratio of probabilities associated by $\net{2}$ for the same correct decision on $\xlayer{0}$. Hence, you can trust $\net{1}$ if you trusted $\net{2}$.
\end{itemize}

\end{proof}

\subsection{A Sound Over-Approximation of \glspl{DNN} Behavior}
\label{appendix:relaxation}

Solving the original minimization problem is not trivial.
Indeed, the activation functions result in nonlinear constraints for
Equations~(\ref{N1})~and~(\ref{N2}).
Our analysis targets \gls{ReLU} layers, it can be generalized to
accommodate any nonlinear activation function that can be represented
in a piece-wise linear form \citep{Ehler:linear}.
\gls{ReLU} functions are the most widely used activation
functions in \glspl{DNN}.
Recall the last layer is a softmax layer, but we are only interested
in the possible values of its inputs.
We explain in the following how to over-approximate the values
computed at each layer using linear inequalities. The goal is to make
possible the computation of a tight lower bound for the minimization problem
from Section~\ref{sec:method}.

A \gls{ReLU} compounds two linear segments, resulting in a piece-wise
linear function. Consider \(\hat{x}_i^{(k)} = {\mW}_{i,:}^{(k)}
\vx^{(k-1)} + {b}_i^{(k)}\), the value of the $i^{th}$ neuron in the
$k^{th}$ layer before applying the activation function.
The output $x_i^{(k)}$ of the \gls{ReLU} of $\hat{x}_i^{(k)}$ is $\hat{x}_i^{(k)}$
if $\hat{x}_i^{(k)} \geq 0 $ and $0$ otherwise.
When considering a $\delta$-neighborhood as inputs, each neuron
\(\hat{x}_i^{(k)}\) gets lower and upper bounds, denoted as
\(\hat{\underline{x}}_i^{(k)}\) and \(\overline{\hat{x}}_i^{(k)}\),
respectively.
Applying \gls{ReLU} to each neuron \(\hat{x}_i^{(k)}\) results in the
neuron being always active 
when both lower and upper bounds are
positive (i.e., \gls{ReLU} coincides with the identity relation),
and always inactive when both are negative (i.e., \gls{ReLU} coincides with zero).
There is a third situation where lower and upper bounds have different signs.
To adapt \gls{ReLU} to our optimization framework, we
consider as in \citep{Ehler:linear} the minimum convex area bounded by
\(\hat{\underline{x}}_i^{(k)}\) and
\(\overline{\hat{x}}_i^{(k)}\).
The convex is given by the three
inequalities:
\begin{gather*}
x^{(k)}_{i} \leq \overline{\hat{x}}^{(k)}_{i}.\dfrac{\hat{x}^{(k)}_{i} - \hat{\underline{x}}^{(k)}_{i}}{\overline{\hat{x}}^{(k)}_{i} - \hat{\underline{x}}^{(k)}_{i}} ,\;\;\;
x^{(k)}_{i} \geq \hat{x}^{(k)}_{i}, \;\;\;
x^{(k)}_{i} \geq 0.
\label{eqrelu}
\end{gather*}

Lower and upper bounds of each neuron can be calculated by propagating
through the network, starting from the input layer w.r.t. the
perturbation \(\delta\). In fact, our proposed framework can manage various layers including,
 but not limited to, convolution, zero-padding, max-pooling, permute,
 and flattening layers.
 For instance, a max-pooling layer with a pool size of $p_k$ can be approximated
 with $p_k + 1$ inequalities as follows. 
 Let $J=\{(i-1)p_k+1, \dots, ip_k \}$, use: 
\begin{equation*}
\begin{gathered}
{x}^{(k)}_{i} \geq {x}_{j}^{(k-1)}, \forall j \in J, \\ 
\sum_{j\in J}{x}_{j}^{(k-1)} \geq {x}^{(k)}_{i} + \sum_{j\in J} \underline{x}_{j}^{(k-1)} - \max_{j\in J} \underline{x}_{j}^{(k-1)}.
\end{gathered}
\end{equation*}

Other nonlinear layers used in Equations~(\ref{N1})~and~(\ref{N2})
can also be over-approximated using linear inequalities.

\subsection{Joint vs Independent Analysis}

\begin{proof}
 Recall:
 {
 \begin{itemize}
 \item \(
\minprob{\net{1}}{\net{2}}{\txlayer{2}}{\delta}(\class{i},\class{j})
=   min\left\{
    {\left((\subxlayer{\class{i}}{N_1} - \subxlayer{\class{j}}{N_1}) -  (\subylayer{\class{i}}{N_2} - \subylayer{\class{j}}{N_2})\right)}
    ~|~{\xlayer{0}\in\sphere{\txlayer{0}}{\delta}}
    \right\}
\)
 \item \(
\minprob{\net{1}}{\zeronet}{\txlayer{0}}{\delta}(\class{i},\class{j})
=   min\left\{
    {\left((\subxlayer{\class{i}}{N_1} - \subxlayer{\class{j}}{N_1})\right)}
    ~|~{\xlayer{0}\in\sphere{\txlayer{0}}{\delta}}
    \right\}
\)
 \item \(
\minprob{\zeronet}{\net{2}}{\txlayer{0}}{\delta}(\class{i},\class{j})
=   min\left\{
    {\left(-(\subylayer{\class{i}}{N_2} - \subylayer{\class{j}}{N_2})\right)}
    ~|~{\ylayer{0}\in\sphere{\txlayer{0}}{\delta}}
    \right\}
    \)
\end{itemize}
}

Let \(\subxlayer{\class{i}}{N_1}, \subxlayer{\class{j}}{N_1}, \subylayer{\class{i}}{N_2}, \subylayer{\class{j}}{N_2}\)
be the logits values obtained in the solution \(\minprob{\net{1}}{\net{2}}{\txlayer{2}}{\delta}(\class{i},\class{j})\).

Let \({\subxlayer{\class{i}}{N_1}}', {\subxlayer{\class{j}}{N_1}}'\)
be the logits values obtained in the solution \(\minprob{\net{1}}{\zeronet}{\txlayer{2}}{\delta}(\class{i},\class{j})\).

Let \({\subylayer{\class{i}}{N_2}}', {\subylayer{\class{j}}{N_2}}'\)
be the logits values obtained in the solution \(\minprob{\zeronet}{\net{2}}{\txlayer{2}}{\delta}(\class{i},\class{j})\).

By definitions: 
\begin{itemize}
\item \( ({\subxlayer{\class{i}}{N_1}}' - {\subxlayer{\class{j}}{N_1}}') \leq (\subxlayer{\class{i}}{N_1} - \subxlayer{\class{j}}{N_1})\)
\item \( -({\subylayer{\class{i}}{N_2}}' - {\subylayer{\class{j}}{N_2}}') \leq -(\subylayer{\class{i}}{N_2} - \subylayer{\class{j}}{N_2})\)
\end{itemize}

Hence:
\[
({\subxlayer{\class{i}}{N_1}}' - {\subxlayer{\class{j}}{N_1}}') 
-({\subylayer{\class{i}}{N_2}}' - {\subylayer{\class{j}}{N_2}}')
\leq 
(\subxlayer{\class{i}}{N_1} - \subxlayer{\class{j}}{N_1})
 - (\subylayer{\class{i}}{N_2} - \subylayer{\class{j}}{N_2})
\]
and
\[
\minprob{\net{1}}{\zeronet}{\txlayer{0}}{\delta}(\class{i},\class{j})+\minprob{\zeronet}{\net{2}}{\txlayer{0}}{\delta}(\class{i},\class{j})
         \leq 
     \minprob{\net{1}}{\net{2}}{\txlayer{0}}{\delta}(\class{i},\class{j})
     \]
 \end{proof}
 
\subsection{Transitivity Property of \glspl{LRPR}}

\begin{proof}
The proof follows directly from Theorem \ref{theo:relaxed_bounds:appendix}.
\end{proof}

\subsection{Additional Tables and Figures for Section ~\ref{sec:eval}}

\subsubsection{Information on Neural Networks} \label{info_NN}

\paragraph{Original Networks.} For the MNIST and CIFAR10 datasets, we use \glspl{DNN} from~\citep{ugare2022proof}, which have undergone robust training as outlined in~\citep{chiang2020certified}. The fully-connected \glspl{DNN} share the same structure, consisting of seven dense layers with 200 neurons each. The convolutional \gls{DNN} trained on the MNIST dataset includes two convolutional layers, each preceded by a zero-padding layer, followed by five dense layers, each with 256 neurons. The convolutional \gls{DNN} trained on the CIFAR10 dataset includes two additional pairs of convolutional and zero-padding layers compared to the convolutional \gls{DNN} for MNIST. The experimental results for the convolutional \glspl{DNN} are presented in the following sections. Patients in the CHB-MIT and MIT-BIH datasets have personalized convolutional \glspl{DNN}~\citep{baninajjarvnn}. For each patient in the CHB-MIT dataset, the \gls{DNN} has 2048 input neurons, two convolution layers followed by max-pooling layers with 3 and 5 filters, kernel sizes of 100 and 200, and a dense layer with 40 neurons. The accuracy ($\mu \pm {\sigma}$) is $85.7\% \pm 14.8\%$. For each patient in the MIT-BIH dataset, the \gls{DNN} has an input layer with 320 neurons, a convolution layer with a 64-size kernel and 3 filters, and a dense layer with 40 neurons. The accuracy is  $92.2\% \pm 9.1\%$.

\paragraph{Compact Networks.} Here, we explain the architecture and design of compact networks. Our experiments involve pruning, quantization, and knowledge distillation. These techniques derive compact \glspl{DNN} enabling energy-efficient inference on limited resources, and improving generalization and interoperability.

\textbf{Pruned Networks} are created through a pruning procedure applied to \glspl{DNN}, where certain weights and biases are selectively nullified. For the MNIST and CIFAR10 datasets, we use pruned networks generated by~\citep{ugare2022proof} via post-training pruning. Each pruned network removes the smallest weights/biases in each layer, a process called \gls{MBP}, resulting in nine pruned networks with pruning rates ranging from 10\% to 90\%. For the CHB-MIT and MIT-BIH datasets, we apply \gls{MBP} pruning by setting values below 10\% of the maximum weight/bias to zero. The resulting accuracies are $84.1\% \pm 17.2\%$ and $90.7\% \pm 10.9\%$, respectively.

\textbf{\glspl{VNN}} are generated by optimizing weights and biases to maintain their functionality while reducing the number of non-zero weights~\citep{baninajjarvnn}. For the MNIST and CIFAR10 datasets, we use \glspl{VNN} generated by~\citep{baninajjarvnn}. On the CHB-MIT and MIT-BIH datasets, \glspl{VNN} from~\citep{baninajjarvnn} achieve accuracies of $82.5\% \pm 9.6\%$ and $92.0\% \pm 9.1\%$, respectively.

\textbf{Quantized Networks} are obtained by reducing the precision of network weights, converting them from 32-bit floating-point to lower precision. The MNIST and CIFAR10 networks are quantized using post-training float16, int16, int8, and int4 methods~\citep{ugare2022proof}. The same quantization is applied to the CHB-MIT and MIT-BIH networks, and the accuracy of the quantized networks is presented in Table~\ref{acc_all_quant}.

\textbf{Distilled Networks} are compact networks trained via knowledge distillation to mimic teacher networks' behavior~\citep{hinton2015distilling}. The architecture differs from the original networks, and the temperature parameter influences distillation complexity. We evaluate nine temperatures, i.e., $T=1,\dots,9$, and produce distilled networks for the MNIST and CIFAR10 datasets following~\citep{hinton2015distilling}, as~\citep{ugare2022proof} did not provide distilled networks for these datasets. For all distilled networks, we use a single layer with 20 neurons. 
The accuracy of distilled networks is found in Table~\ref{acc_all_dist}.

\begin{table}[b]
\caption{Accuracy of quantized networks trained on CHB-MIT and MIT-BIH datasets.}\label{acc_all_quant}
  \centering
  \scalebox{0.75}{
  \begin{tabular}{ c c c c c}
  \toprule
   & float16 & int16 & int8 & int4  \\ \midrule
  CHB-MIT & $85.7 \pm 14.8$ & $81.6 \pm 14.8$ & $80.9 \pm 15.0$  & $85.1 \pm 15.1$ \\
  MIT-BIH & $92.2 \pm 9.1$ & $91.9 \pm 10.1$ & $91.8 \pm 10.1$ & $90.9 \pm 10.7$ \\
  \bottomrule
\end{tabular}
  }
\end{table}

\begin{table}[t]
\caption{Accuracy of distilled networks trained on CHB-MIT and MIT-BIH datasets.} \label{acc_all_dist}
  \centering
  \scalebox{0.75}{
  \centering
  \begin{tabular}{c c c c c c c c c c c}
  \toprule
   & T1 & T2 & T3 & T4 & T5 & T6 & T7 & T8 & T9 \\ \midrule
  CHB-MIT & $73.4 \pm 9.5$ & $72.5 \pm 11.2$ & $73.0 \pm 9.4$ & $71.0 \pm 9.0$ & $71.6 \pm 8.7$ &  $71.5 \pm 9.6$ & $73.5 \pm 9.7$ & $73.5 \pm 8.6$ & $71.5 \pm 9.5$ \\
  MIT-BIH & $91.9 \pm 6.7$ & $90.6 \pm 9.2$ & $90.0 \pm 11.2$ & $89.1 \pm 12.7$ & $89.7 \pm 11.2$ & $90.3 \pm 10.2$ & $90.6 \pm 10.3$ & $90.6 \pm 10.2$ & $90.2 \pm 10.0$ \\
  \bottomrule
\end{tabular}
}
  
\end{table}

\subsubsection{Additional Experiments and Figures} \label{app_extra_exp}

\paragraph{MNIST and CIFAR10 Datasets.} Figure~\ref{CNN_photo_relative_robustness} presents the results of formal local implication for convolutional \glspl{DNN} trained on the MNIST and CIFAR10 datasets with $\delta = 0.001$. The figures illustrate distinct patterns for different compact networks. Figures~\ref{conv_scatter} and~\ref{photo_scatter} demonstrate the minimum and maximum values of \glspl{LRPR} achieved by our method with joint analysis, compared to those obtained through independent analysis of the networks when $\delta = 0.01$.

\begin{figure}[ht]
  \centering  
  \begin{subfigure}{0.32\textwidth}
    \includegraphics[width=\linewidth]{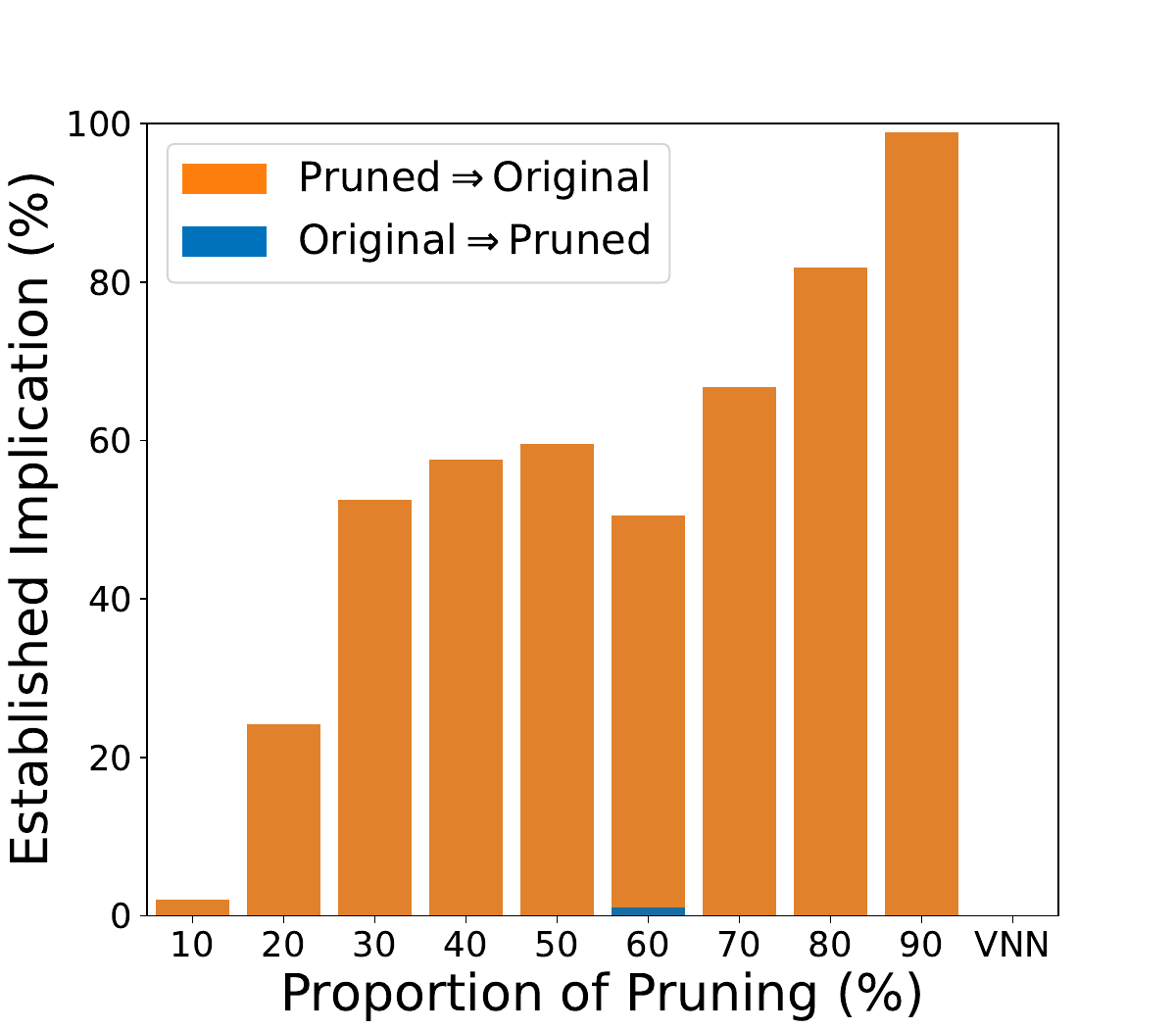}
    \caption{MNIST, Pruned~\citep{ugare2022proof}, ~\citep{baninajjarvnn}}
    \label{CNN_pruned_mnist}
  \end{subfigure}
  \begin{subfigure}{0.32\textwidth}
    \includegraphics[width=\linewidth]{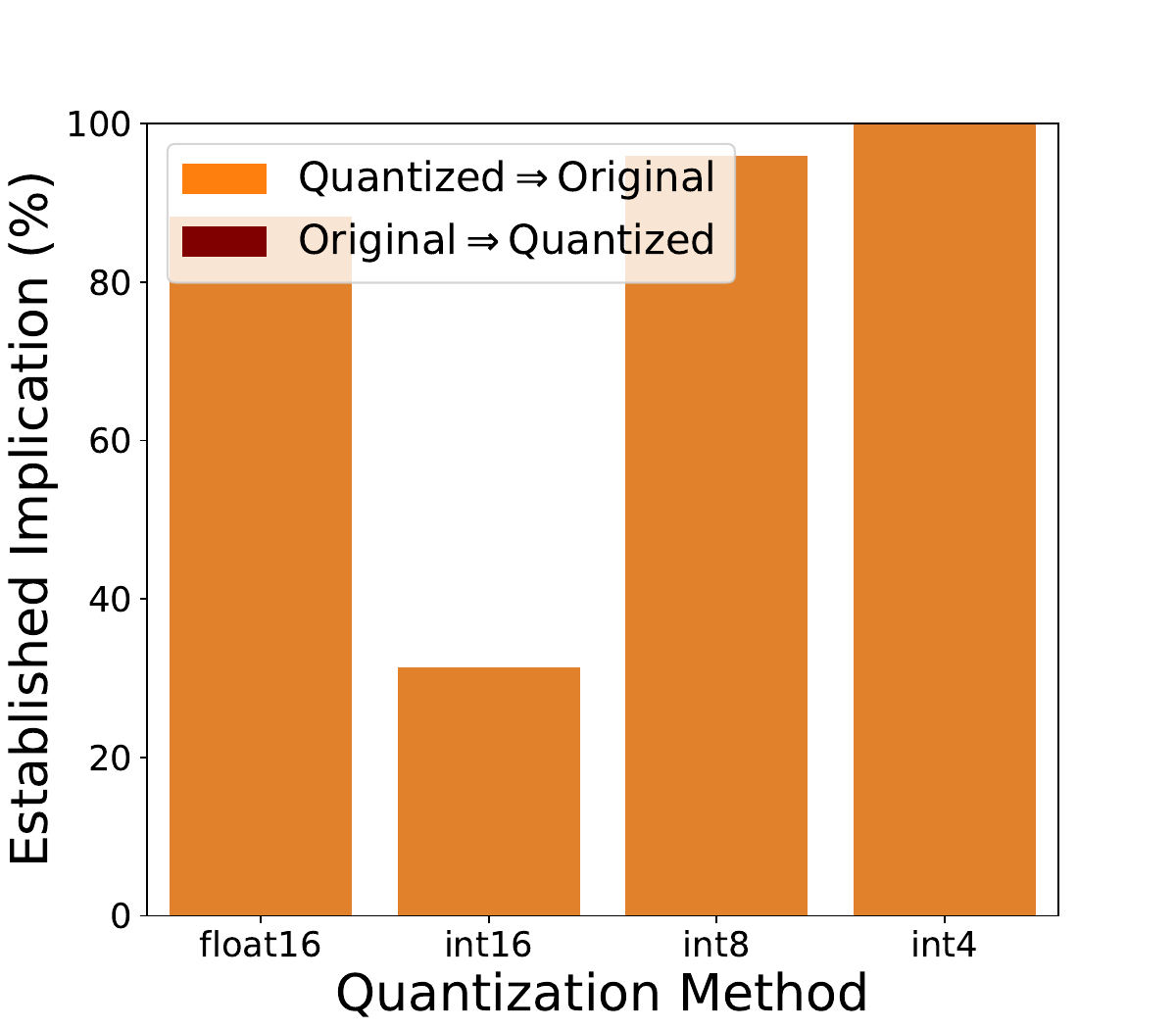}
    \caption{MNIST, Quantized  \citep{ugare2022proof}}
    \label{CNN_quant_mnist}
  \end{subfigure}
  \begin{subfigure}{0.32\textwidth}
    \includegraphics[width=\linewidth]{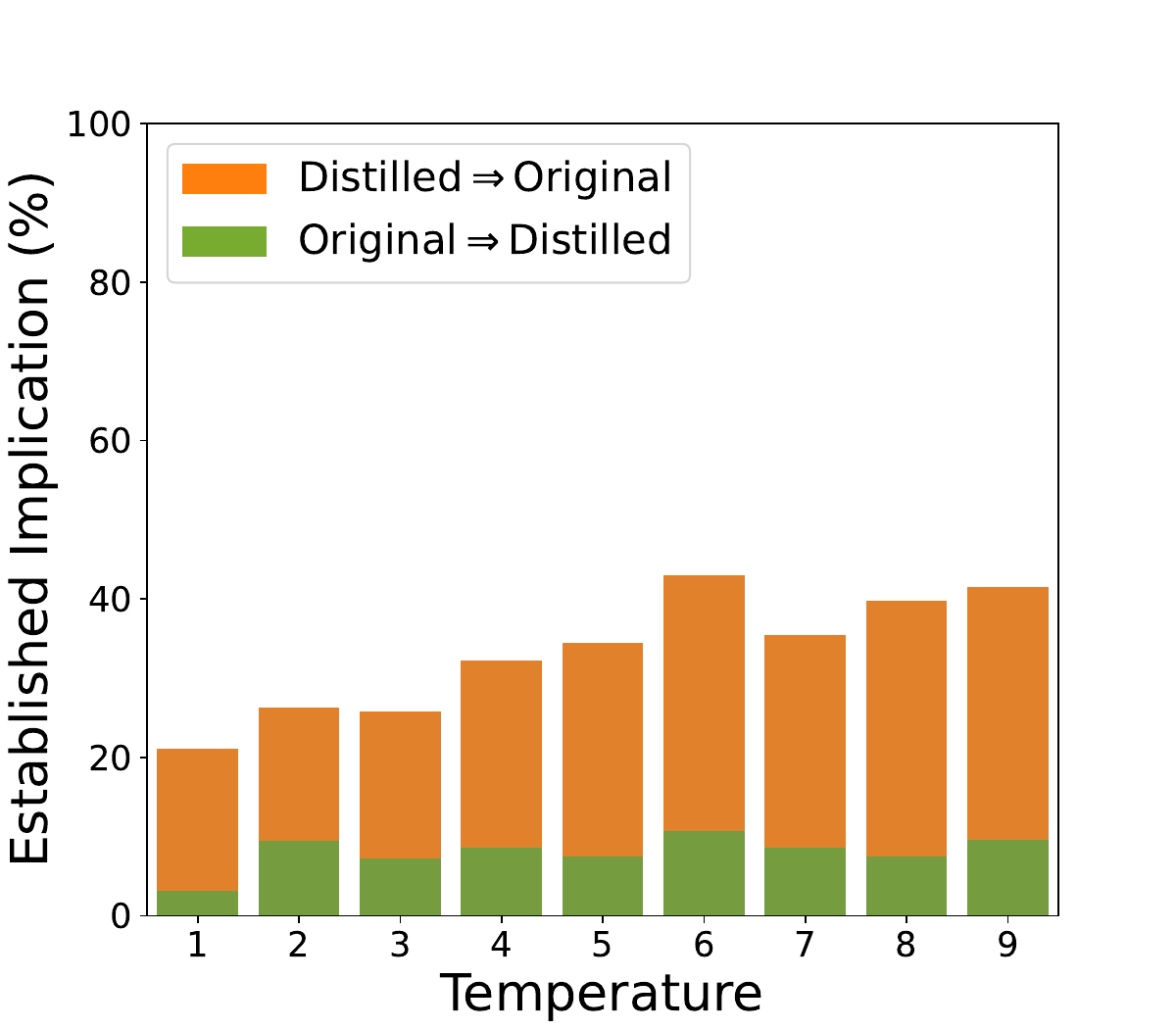}
    \caption{MNIST, Distilled \citep{hinton2015distilling}}
    \label{CNN_dist_mnist}
  \end{subfigure}

\begin{subfigure}{0.32\textwidth}
    \includegraphics[width=\linewidth]{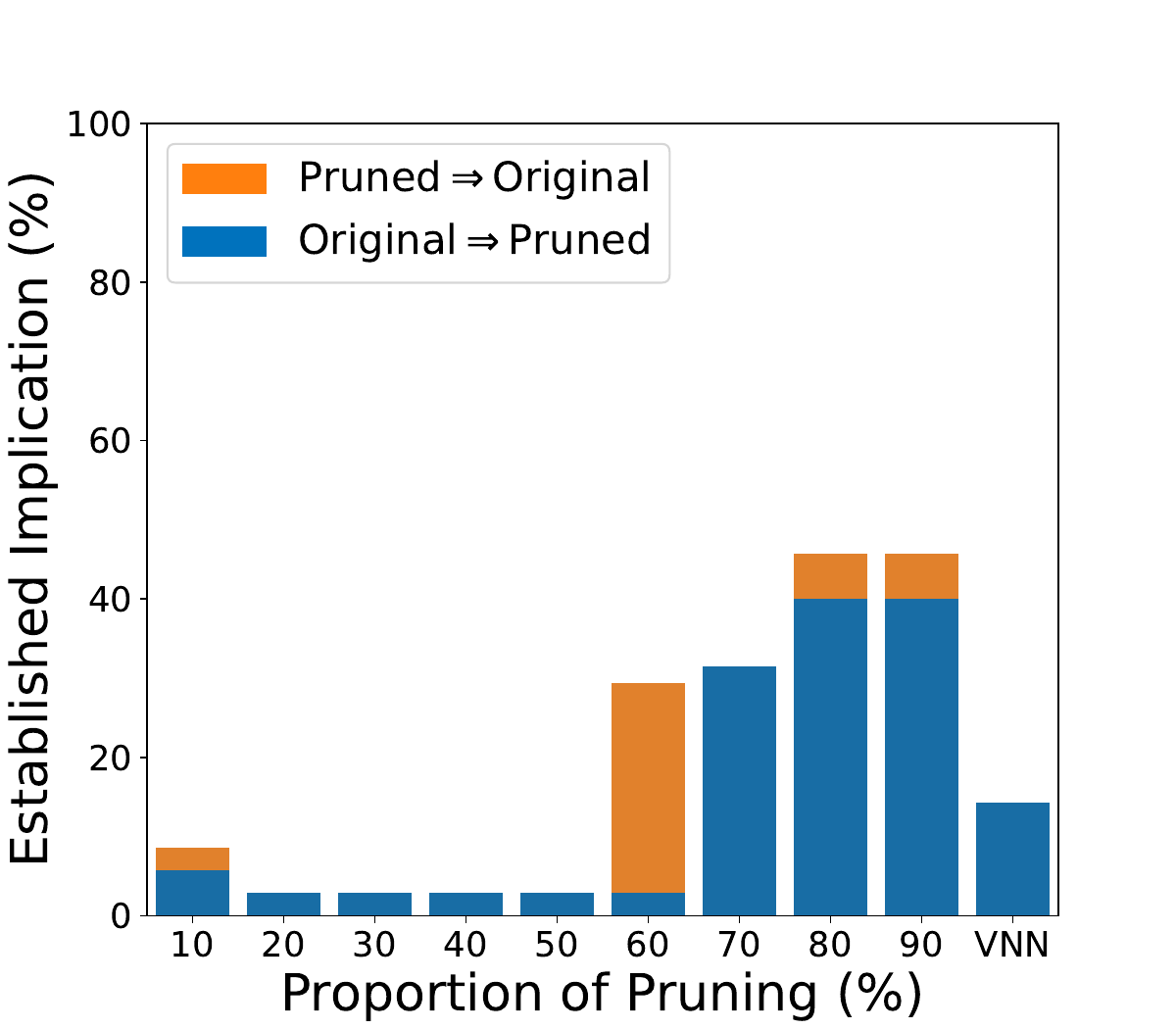}
    \caption{CIFAR10, Pruned~\citep{ugare2022proof}, ~\citep{baninajjarvnn}}
    \label{CNN_pruned_cifar}
  \end{subfigure}
  \begin{subfigure}{0.32\textwidth}
    \includegraphics[width=\linewidth]{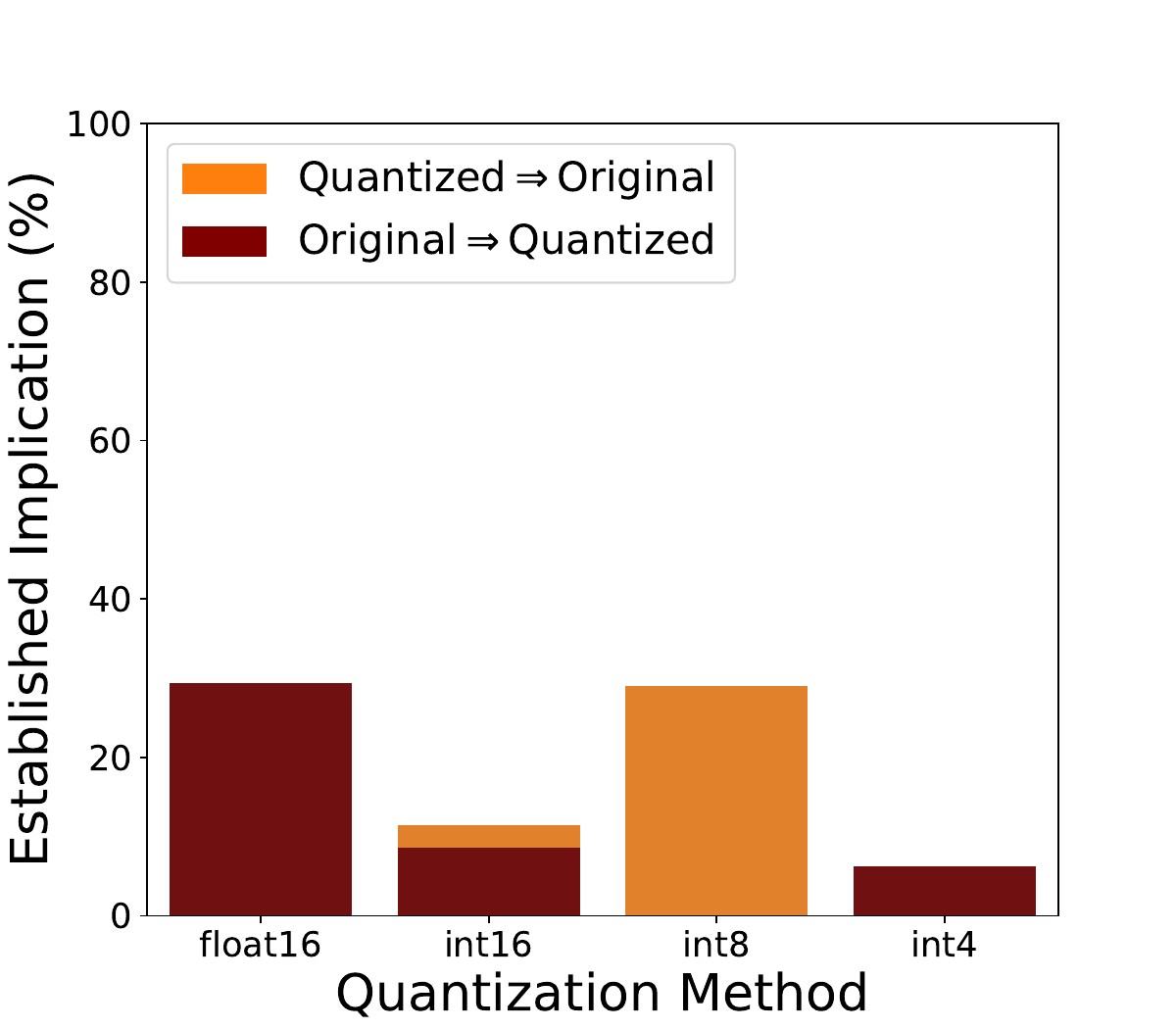}
    \caption{CIFAR10, Quantized \citep{ugare2022proof}}
    \label{CNN_quant_cifar}
  \end{subfigure}
  \begin{subfigure}{0.32\textwidth}
    \includegraphics[width=\linewidth]{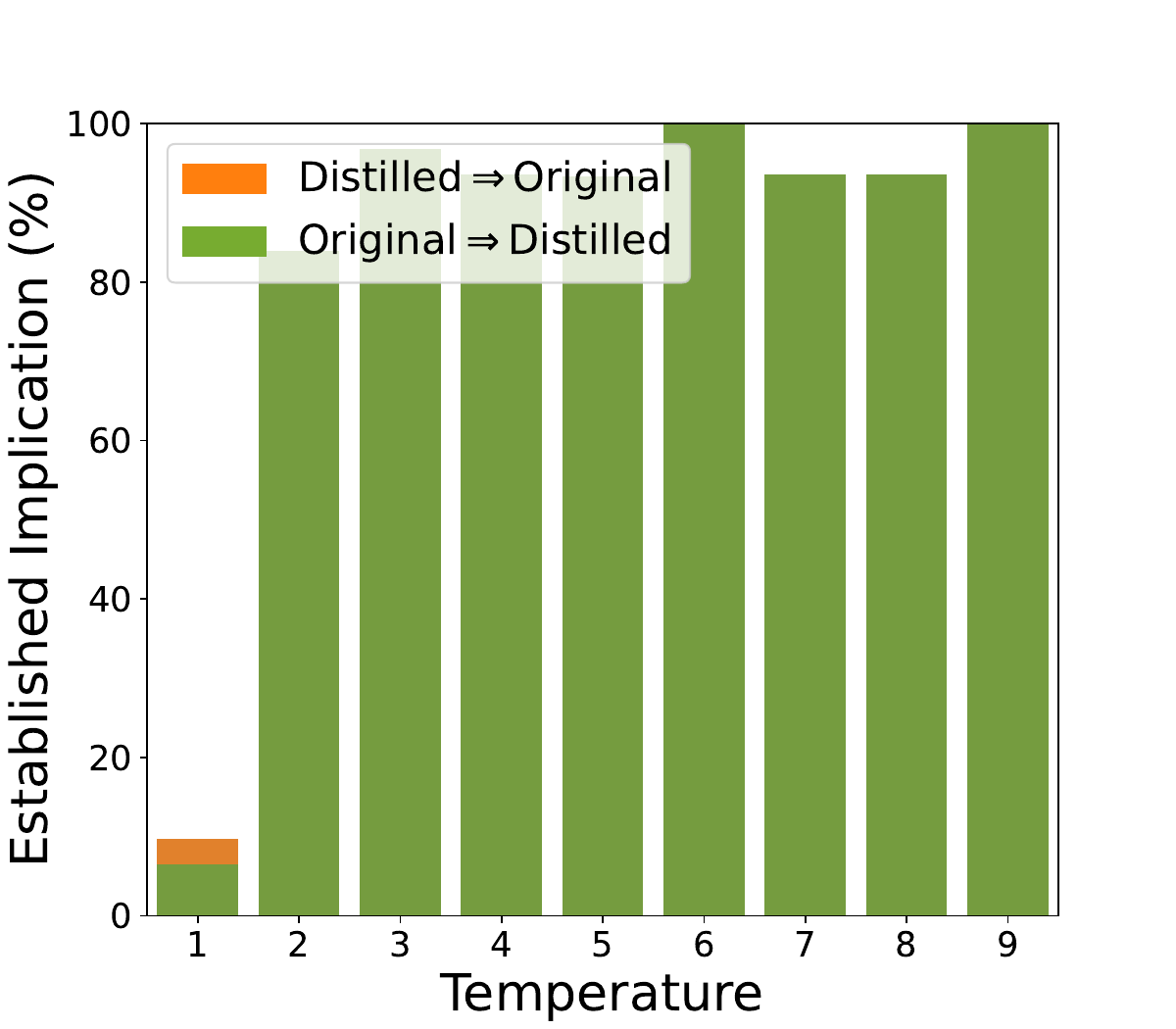}
    \caption{CIFAR10, Distilled  \citep{hinton2015distilling}}
    \label{CNN_dist_cifar}
  \end{subfigure}
  \caption{Stacked bar plots illustrate the established implication of convolutional \glspl{DNN} trained on the MNIST and CIFAR10 datasets with $\delta = 0.001$. The y-axis represents the percentage of samples in the dataset, showing the proportion for which the Compact network implies the Original network (Compact $\implies$ Original) and vice versa. 
 }
  \label{CNN_photo_relative_robustness}
\end{figure}

\begin{figure}[t]
 \centering

\begin{subfigure}{0.24\textwidth}
\includegraphics[trim={0cm 0cm 1cm 0.5cm}, clip,width=\linewidth]{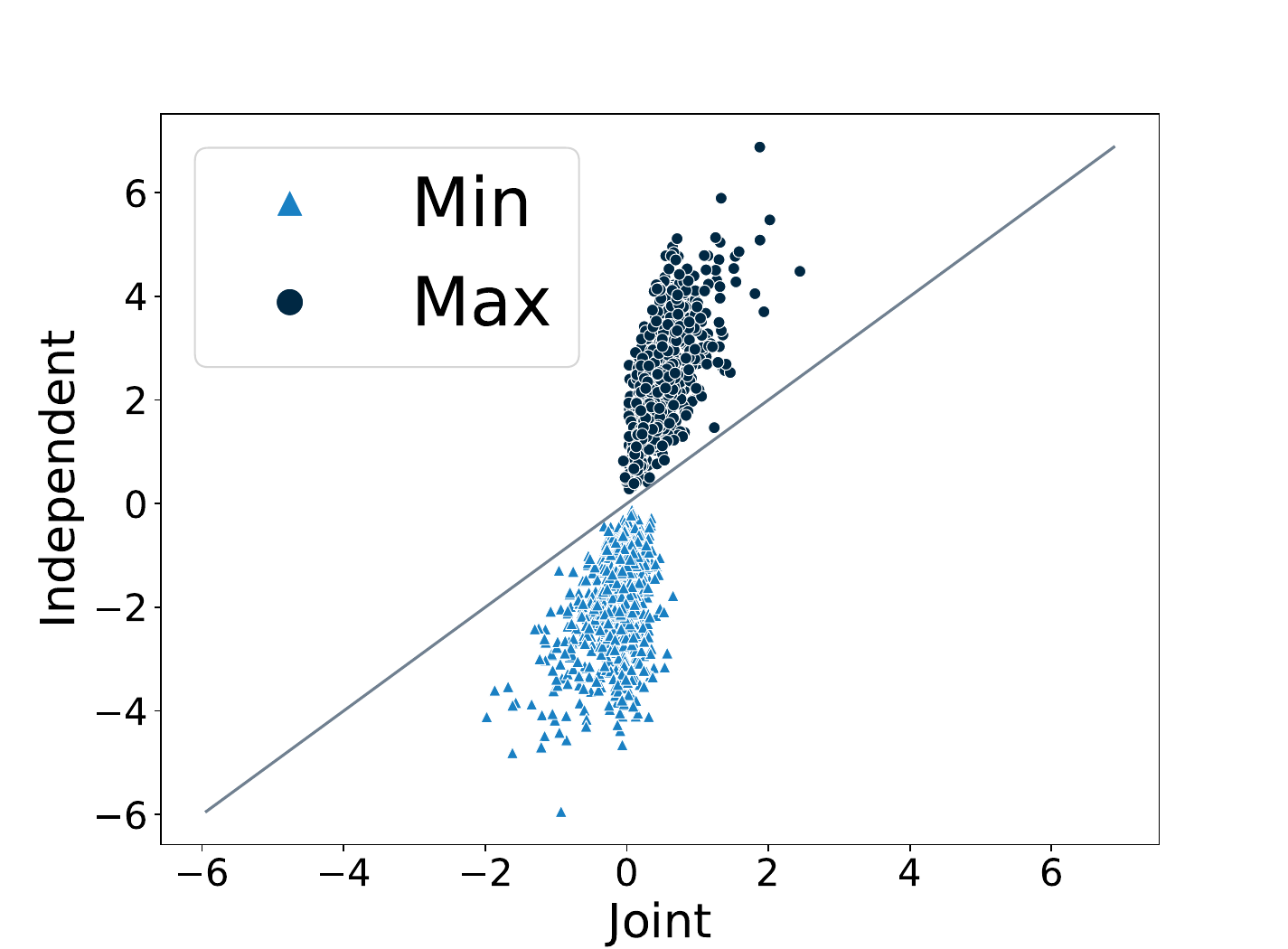}
    \caption{MNIST, Pruned \citep{ugare2022proof}}
    \label{prune50_CNN_mnist}
  \end{subfigure}
\begin{subfigure}{0.24\textwidth}
\includegraphics[trim={0cm 0cm 1cm 0.5cm}, clip,width=\linewidth]{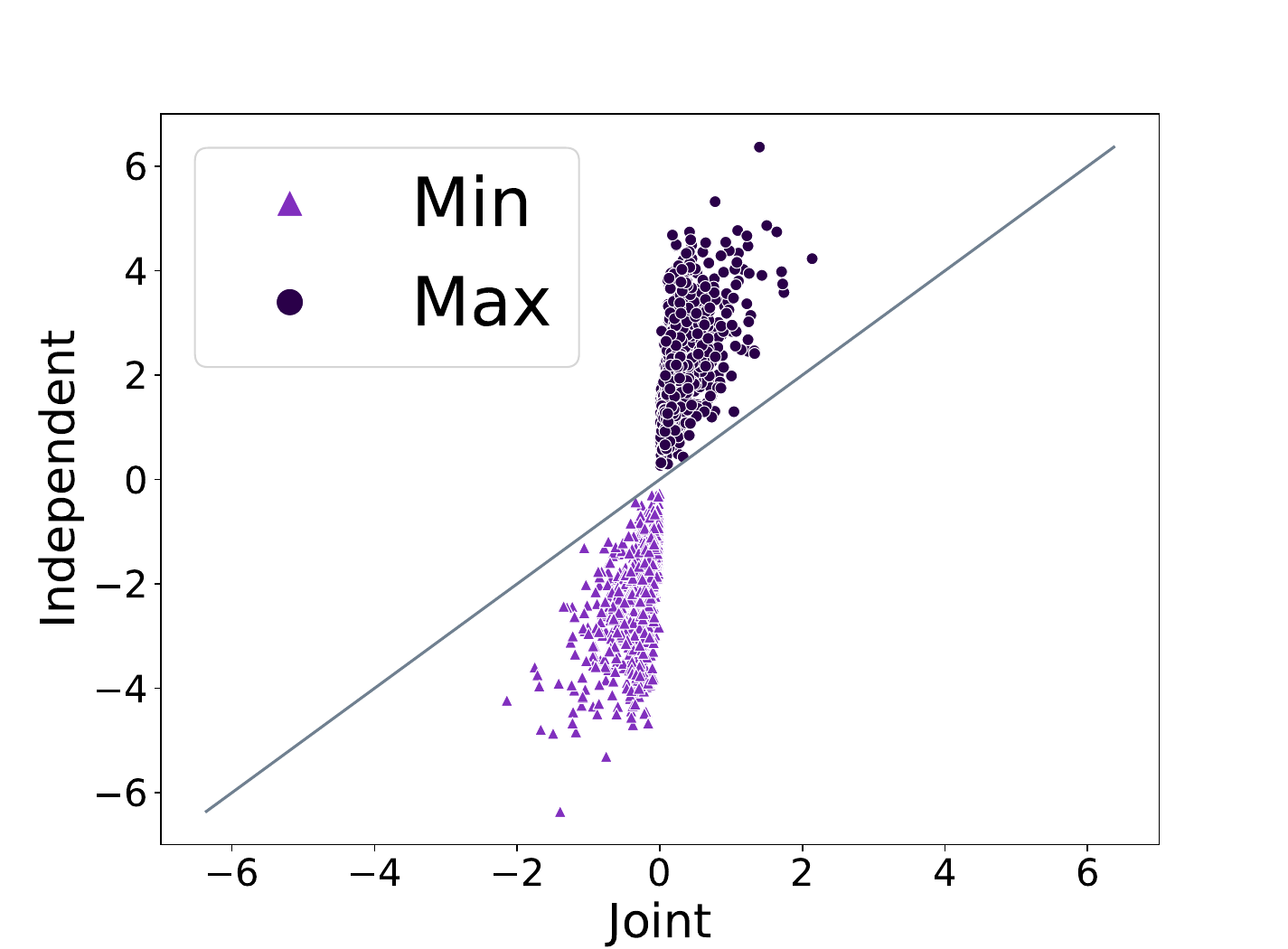}
    \caption{MNIST, \gls{VNN} \citep{baninajjarvnn}}
    \label{VNN_CNN_mnist}
  \end{subfigure}
  \begin{subfigure}{0.24\textwidth}
\includegraphics[trim={0cm 0cm 1cm 0.5cm}, clip,width=\linewidth]{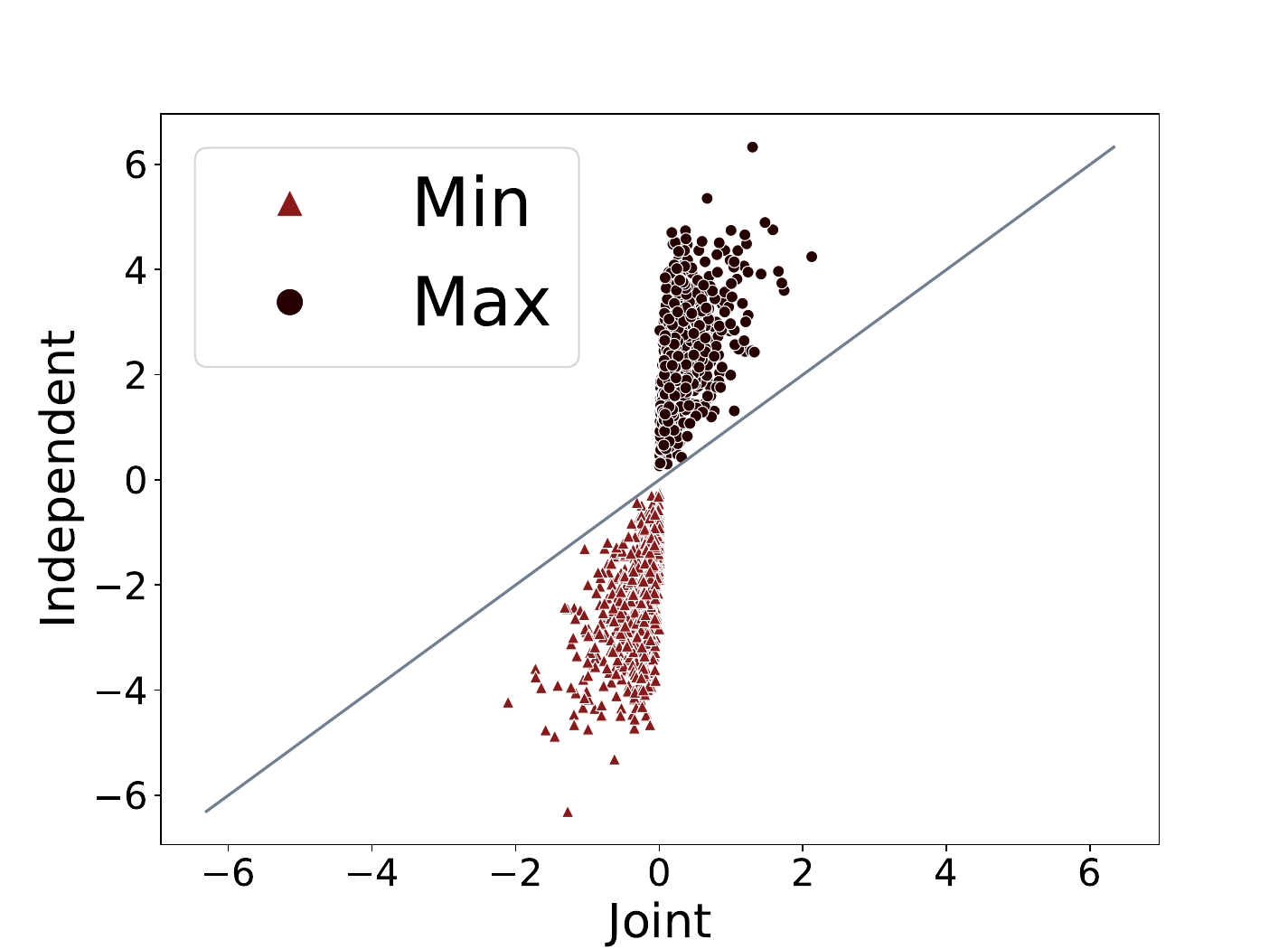}
    \caption{MNIST, Quantized \citep{ugare2022proof}}
    \label{quant16_CNN_mnist}
  \end{subfigure}
\begin{subfigure}{0.24\textwidth}
\includegraphics[trim={0cm 0cm 1cm 0.5cm}, clip,width=\linewidth]{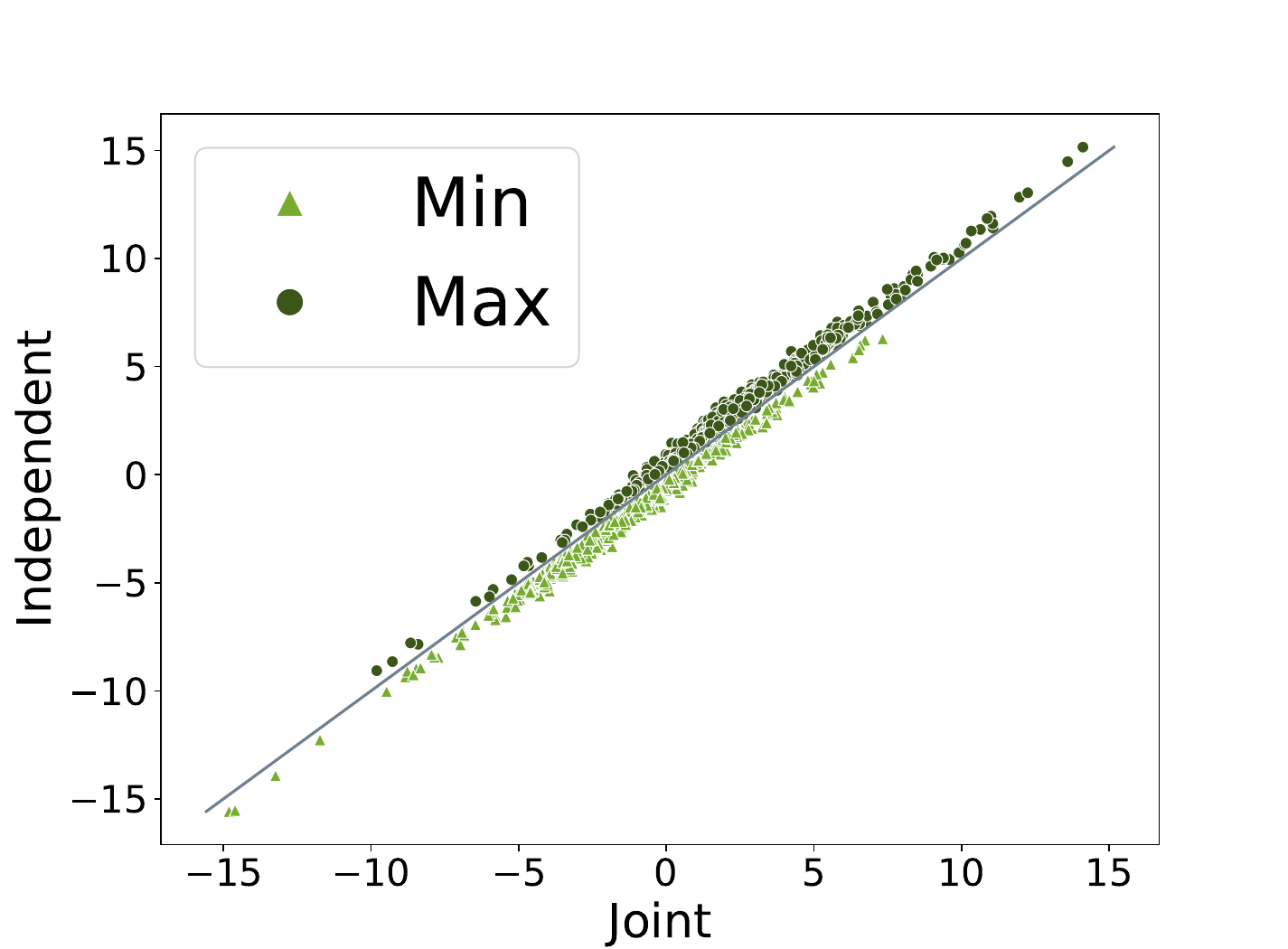}
    \caption{MNIST, Distilled \citep{hinton2015distilling}}
    \label{dist5_CNN_mnist}
  \end{subfigure}

\begin{subfigure}{0.24\textwidth}
\includegraphics[trim={0cm 0cm 1cm 0.5cm}, clip,width=\linewidth]{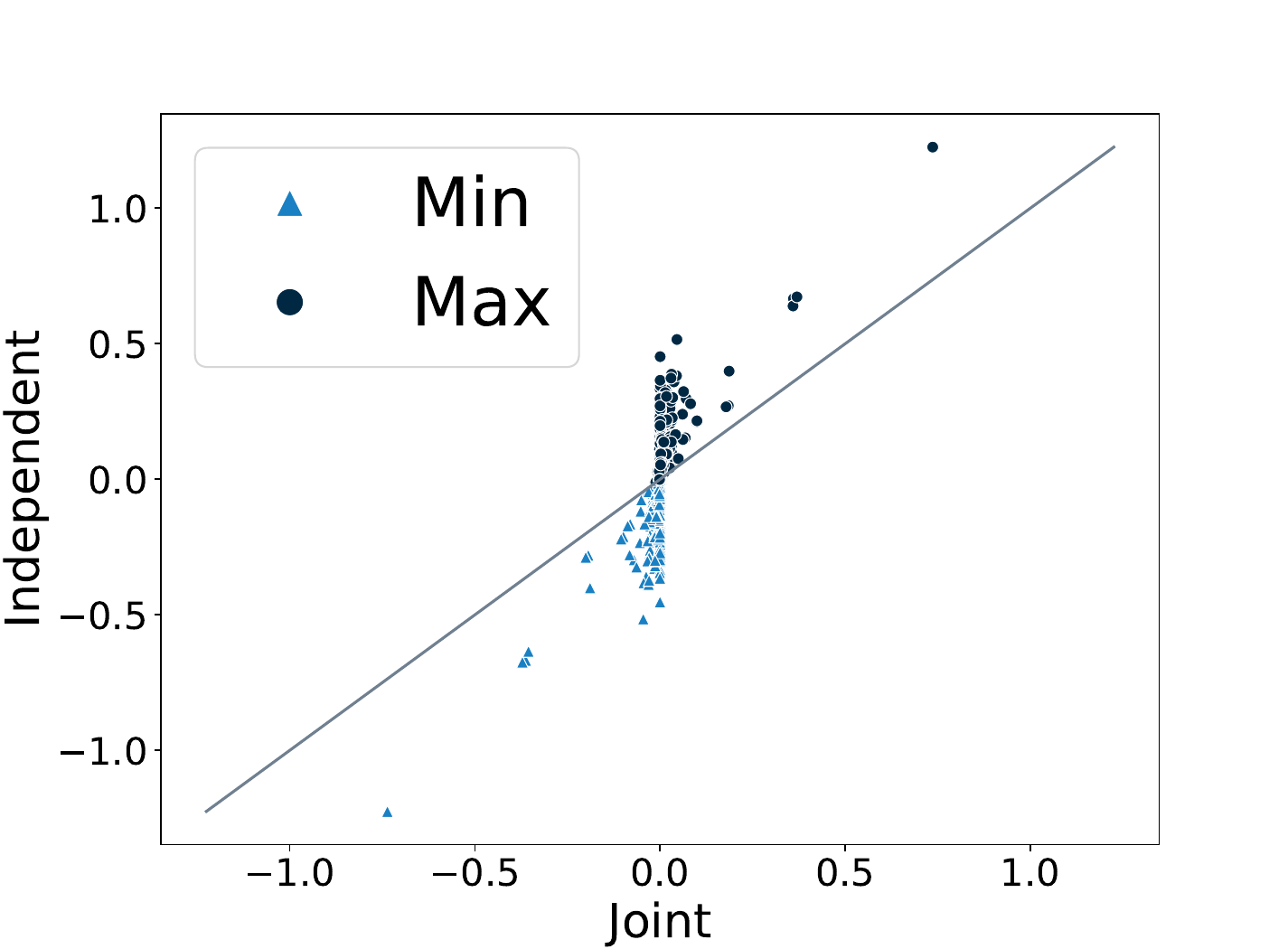}
    \caption{CIFAR10, Pruned \citep{ugare2022proof}}
    \label{prune50_CNN_cifar}
  \end{subfigure}
\begin{subfigure}{0.24\textwidth}
\includegraphics[trim={0cm 0cm 1cm 0.5cm}, clip,width=\linewidth]{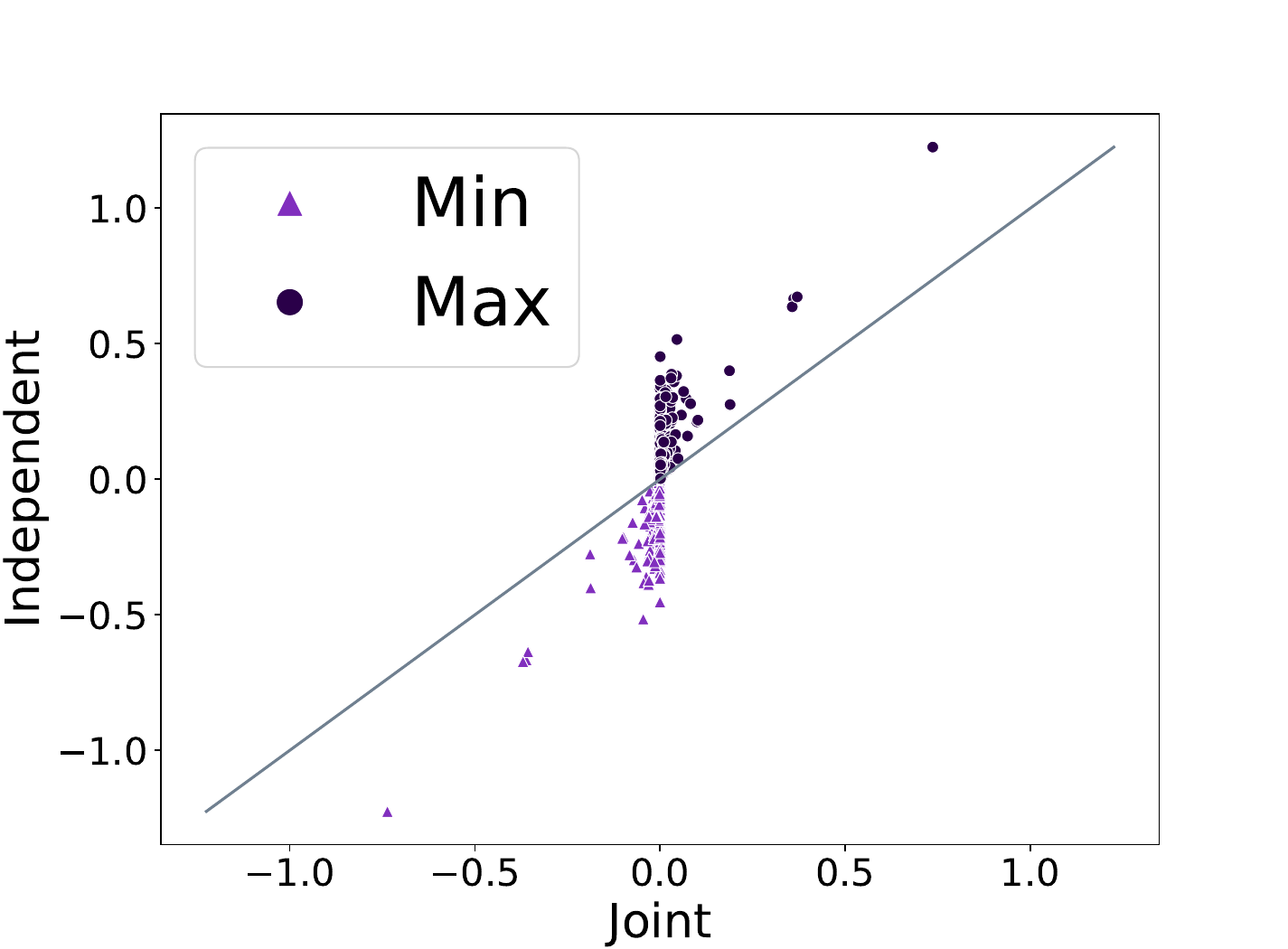}
    \caption{CIFAR10, \gls{VNN} \citep{baninajjarvnn}}
    \label{VNN_CNN_cifar}
  \end{subfigure}
  \begin{subfigure}{0.24\textwidth}
\includegraphics[trim={0cm 0cm 1cm 0.5cm}, clip,width=\linewidth]{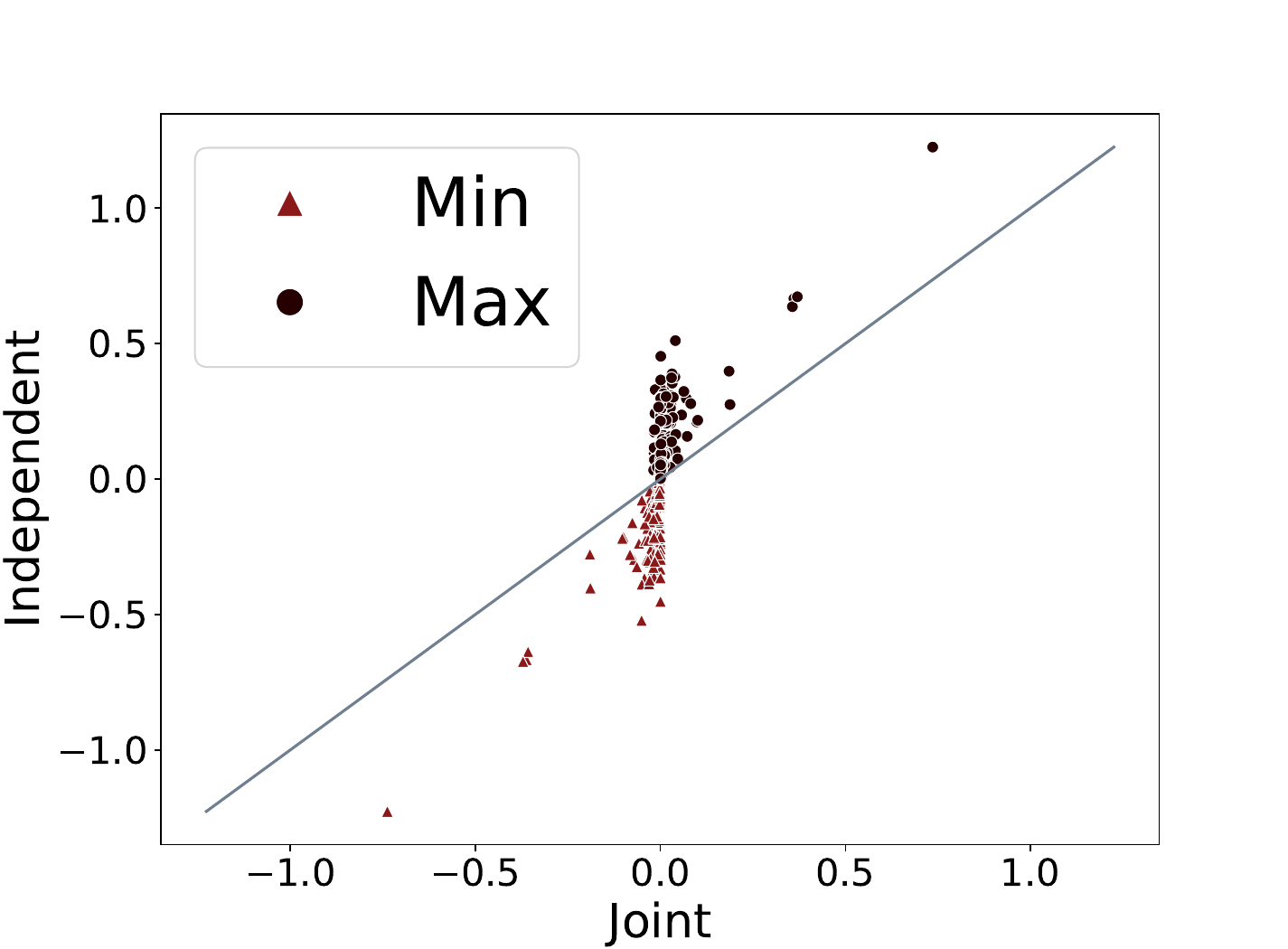}
    \caption{CIFAR10, Quantized \citep{ugare2022proof}}
    \label{quant16_CNN_cifar}
  \end{subfigure}
\begin{subfigure}{0.24\textwidth}
\includegraphics[trim={0cm 0cm 1cm 0.5cm}, clip,width=\linewidth]{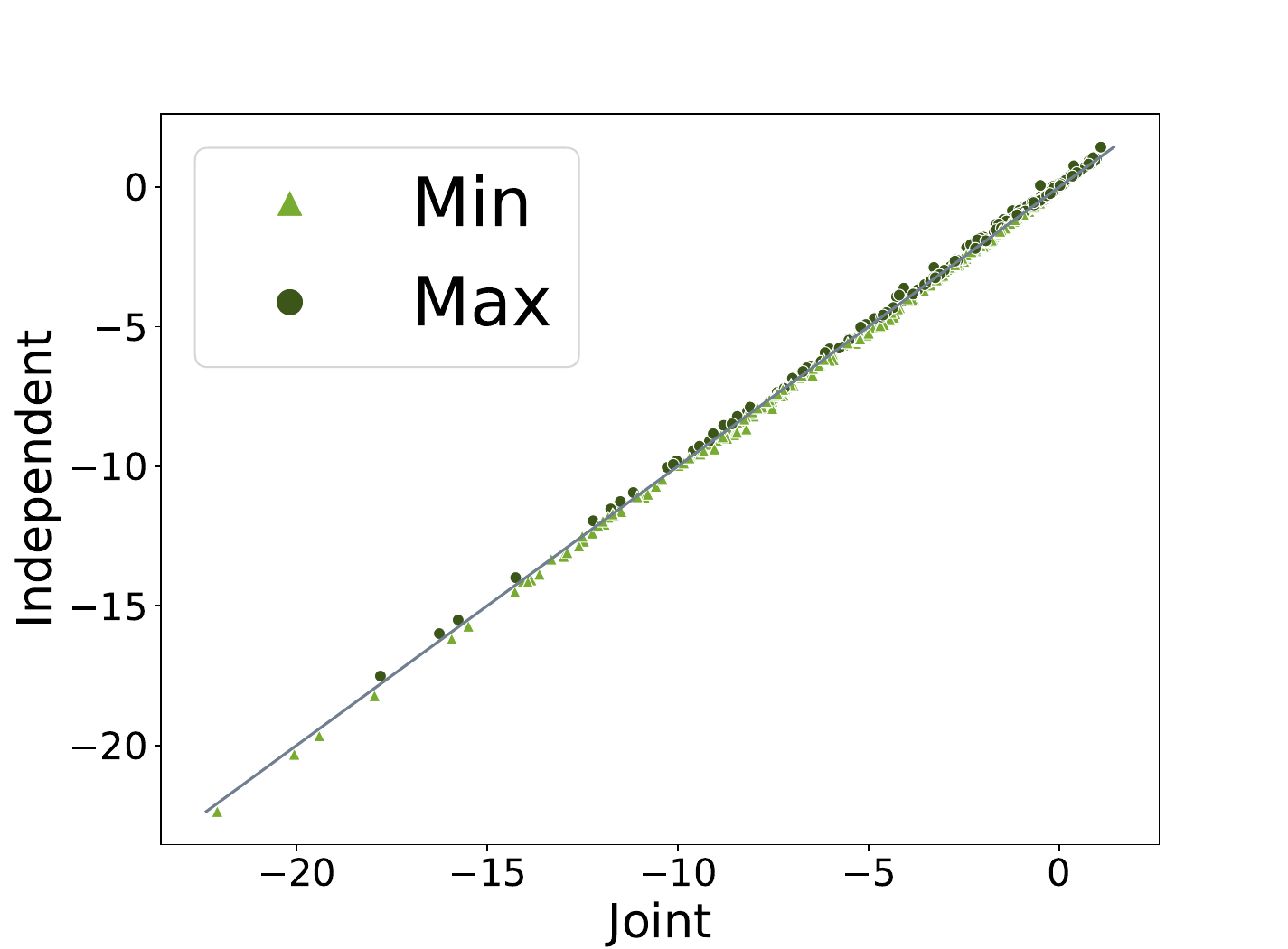}
    \caption{CIFAR10, Distilled \citep{hinton2015distilling}}
    \label{dist5_CNN_cifar}
  \end{subfigure}
  
 \caption{The minimum and maximum \glspl{LRPR} obtained with joint analysis are compared to those from independent analysis, where Compact networks imply Original networks (Compact $\implies$ Original), with $\delta = 0.01$, on convolutional \glspl{DNN}.}
  \label{conv_scatter}
\end{figure}

\paragraph{CHB-MIT and MIT-BIH Datasets.} We assess the established implication of quantized and distilled networks derived from convolutional \glspl{DNN} trained on the CHB-MIT and MIT-BIH datasets, as shown in Figures~\ref{CHB__BIH_quantized_appendix} and~\ref{CHB_BIH_distilled_appendix}. These networks have different precision and temperature settings compared to those discussed in Section~\ref{sec:eval}.


\begin{figure*}[t]
 \centering
  \begin{subfigure}{0.238\textwidth}
\includegraphics[trim={0cm 0cm 1cm 0.5cm}, clip,width=\linewidth]{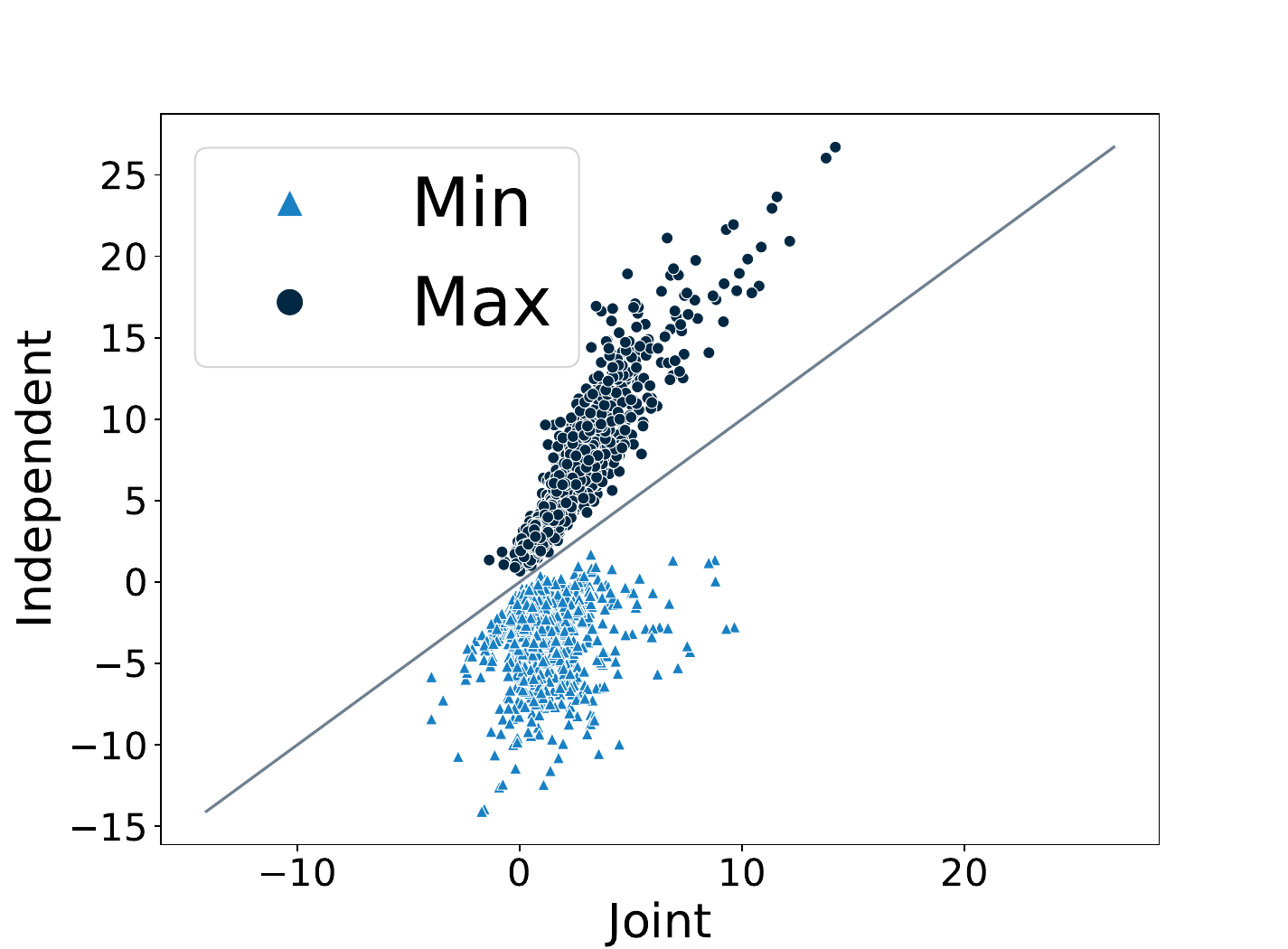}
    \caption{MNIST, Pruned \citep{ugare2022proof}}
    \label{prune50_FNN_mnist}
  \end{subfigure}
\begin{subfigure}{0.238\textwidth}
\includegraphics[trim={0cm 0cm 1cm 0.5cm}, clip,width=\linewidth]{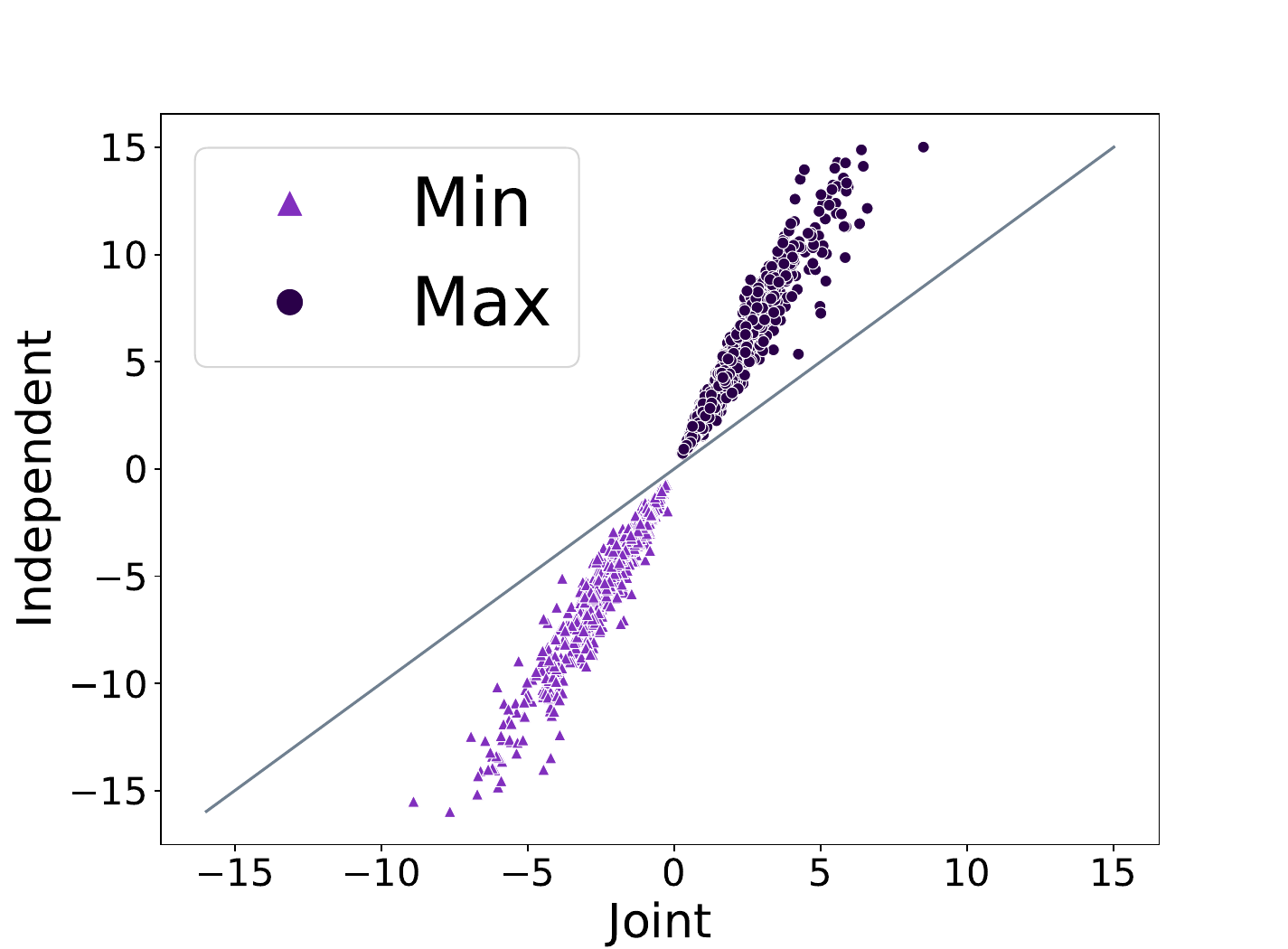}
    \caption{MNIST, \gls{VNN} \citep{baninajjarvnn}}
    \label{VNN_FNN_mnist}
  \end{subfigure}
  \begin{subfigure}{0.238\textwidth}
\includegraphics[trim={0cm 0cm 1cm 0.5cm}, clip,width=\linewidth]{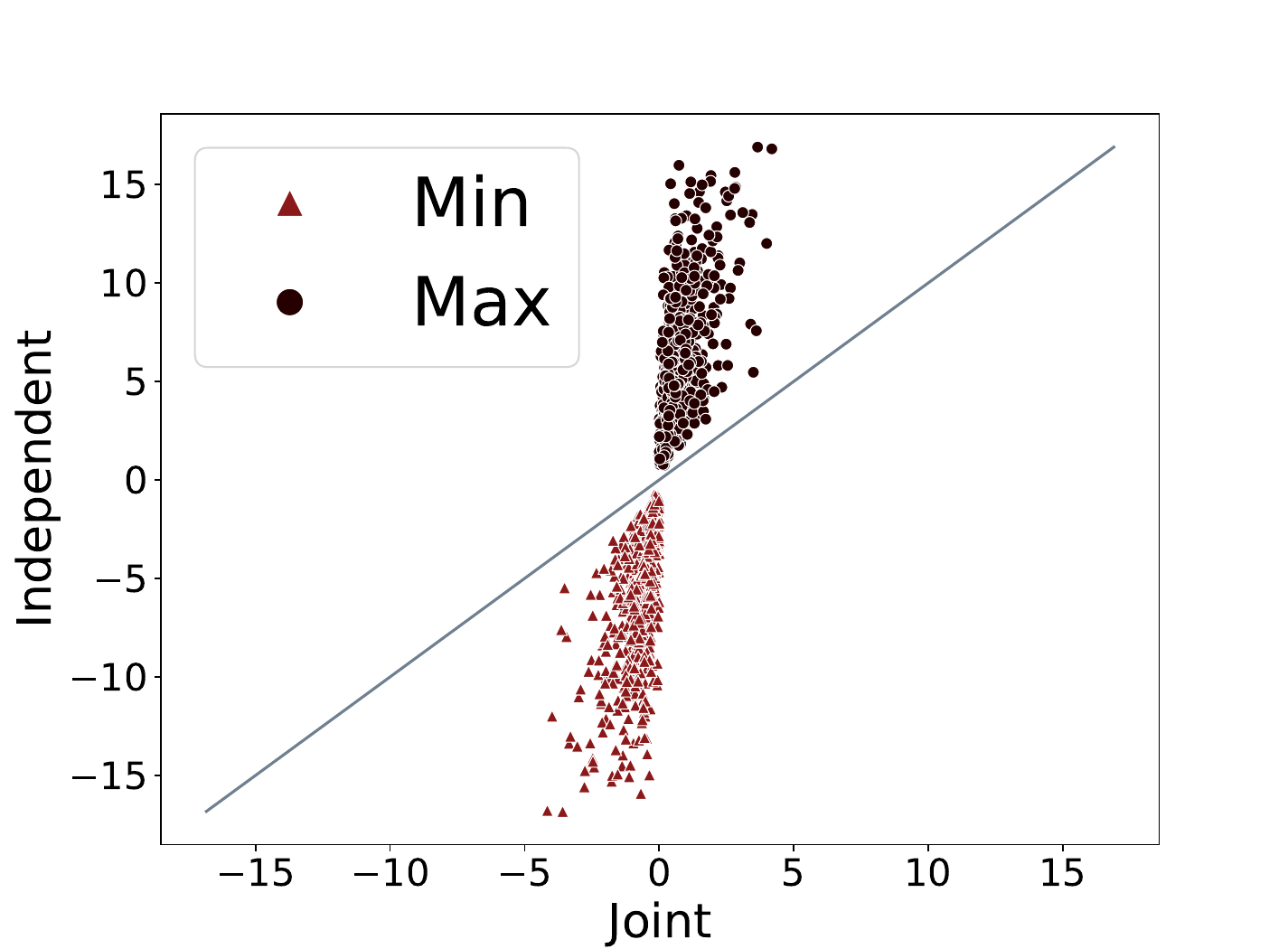}
    \caption{MNIST, Quantized \citep{ugare2022proof}}
    \label{quant16_FNN_mnist}
  \end{subfigure}
\begin{subfigure}{0.238\textwidth}
\includegraphics[trim={0cm 0cm 1cm 0.5cm}, clip,width=\linewidth]{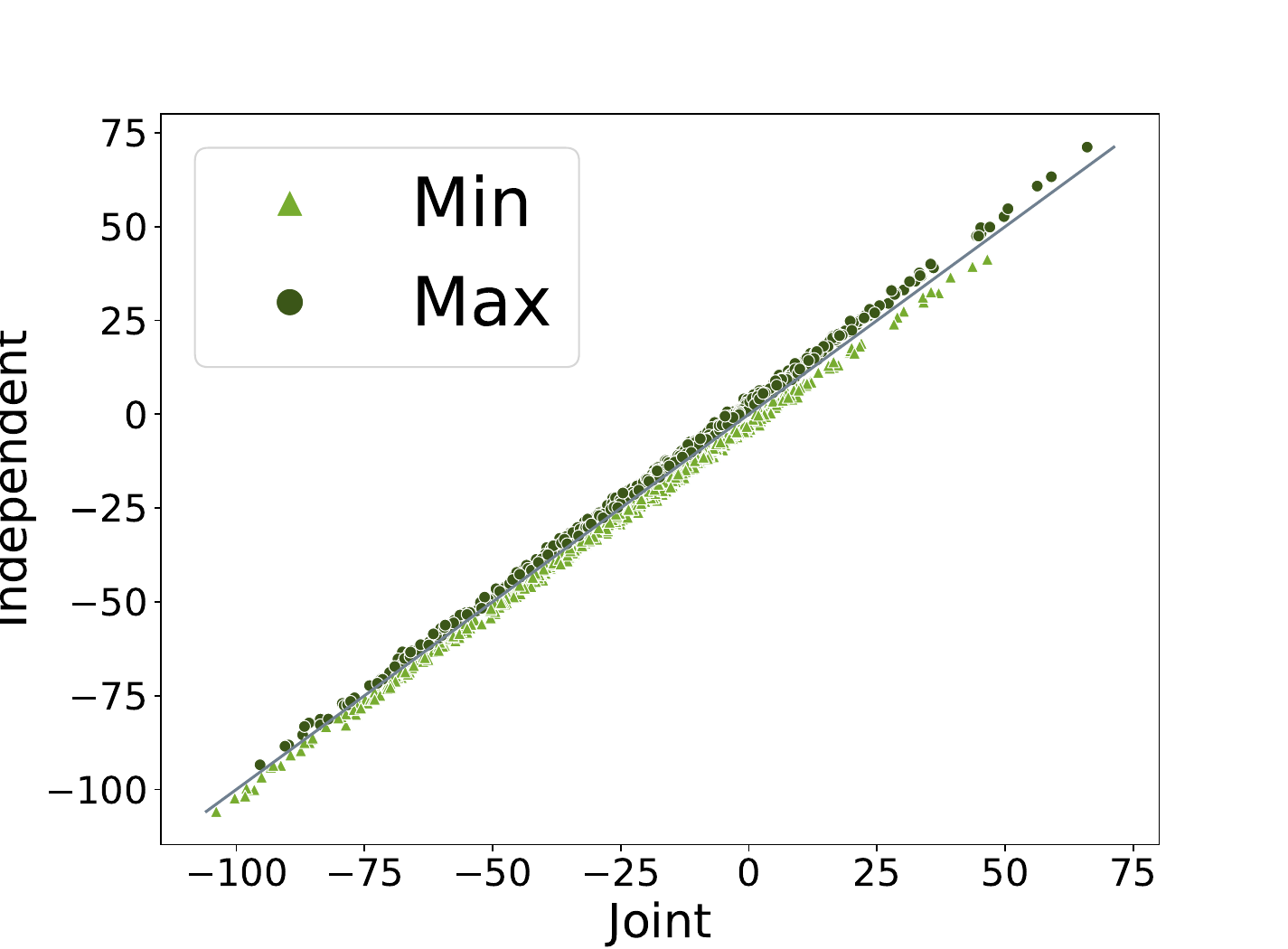}
    \caption{MNIST, Distilled \citep{hinton2015distilling}}
    \label{dist5_FNN_mnist}
  \end{subfigure}

  \begin{subfigure}{0.238\textwidth}
\includegraphics[trim={0cm 0cm 1cm 0.5cm}, clip,width=\linewidth]{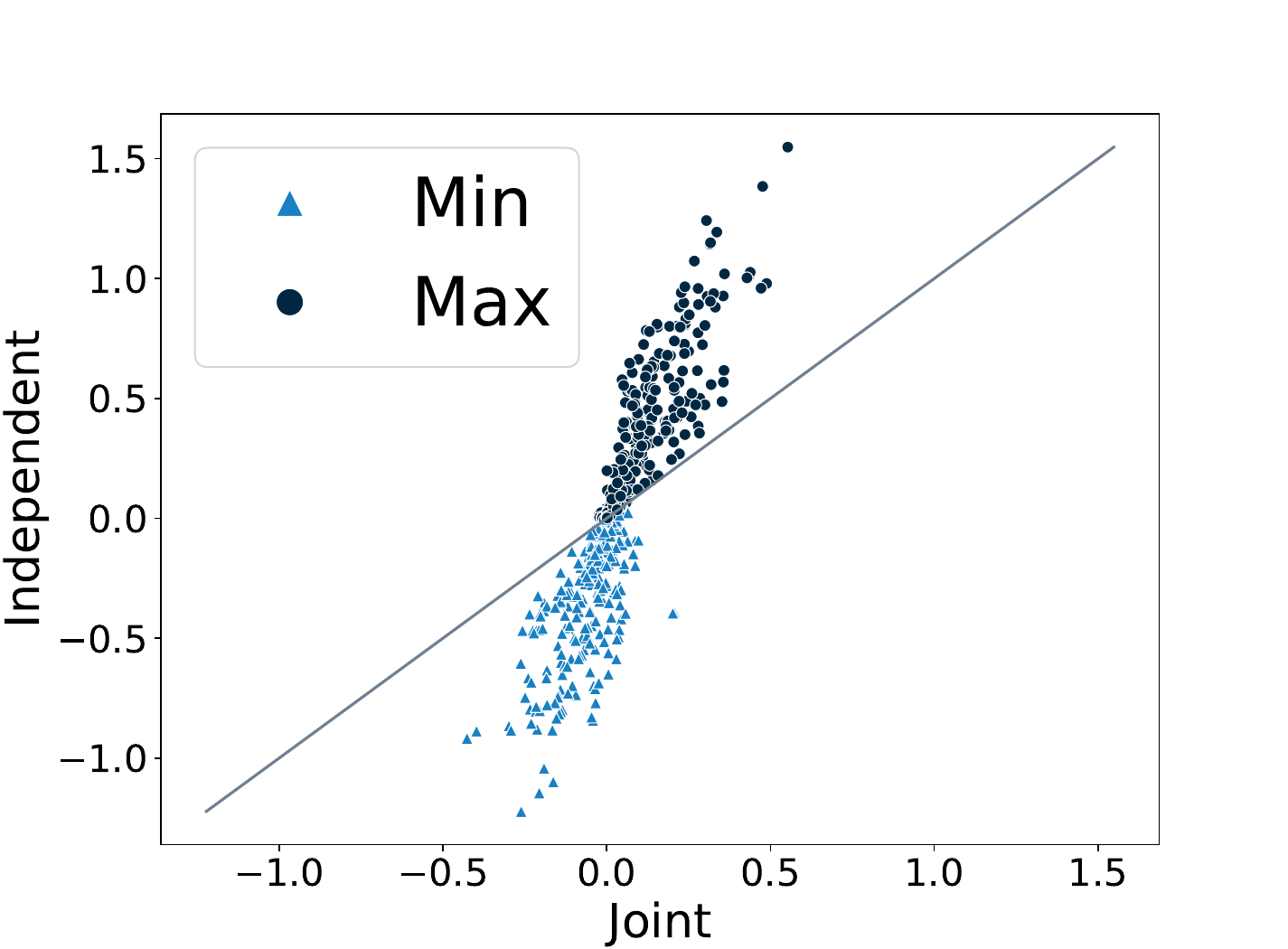}
    \caption{CIFAR10, Pruned \citep{ugare2022proof}}
    \label{prune50_FNN_cifar}
  \end{subfigure}
\begin{subfigure}{0.238\textwidth}
\includegraphics[trim={0cm 0cm 1cm 0.5cm}, clip,width=\linewidth]{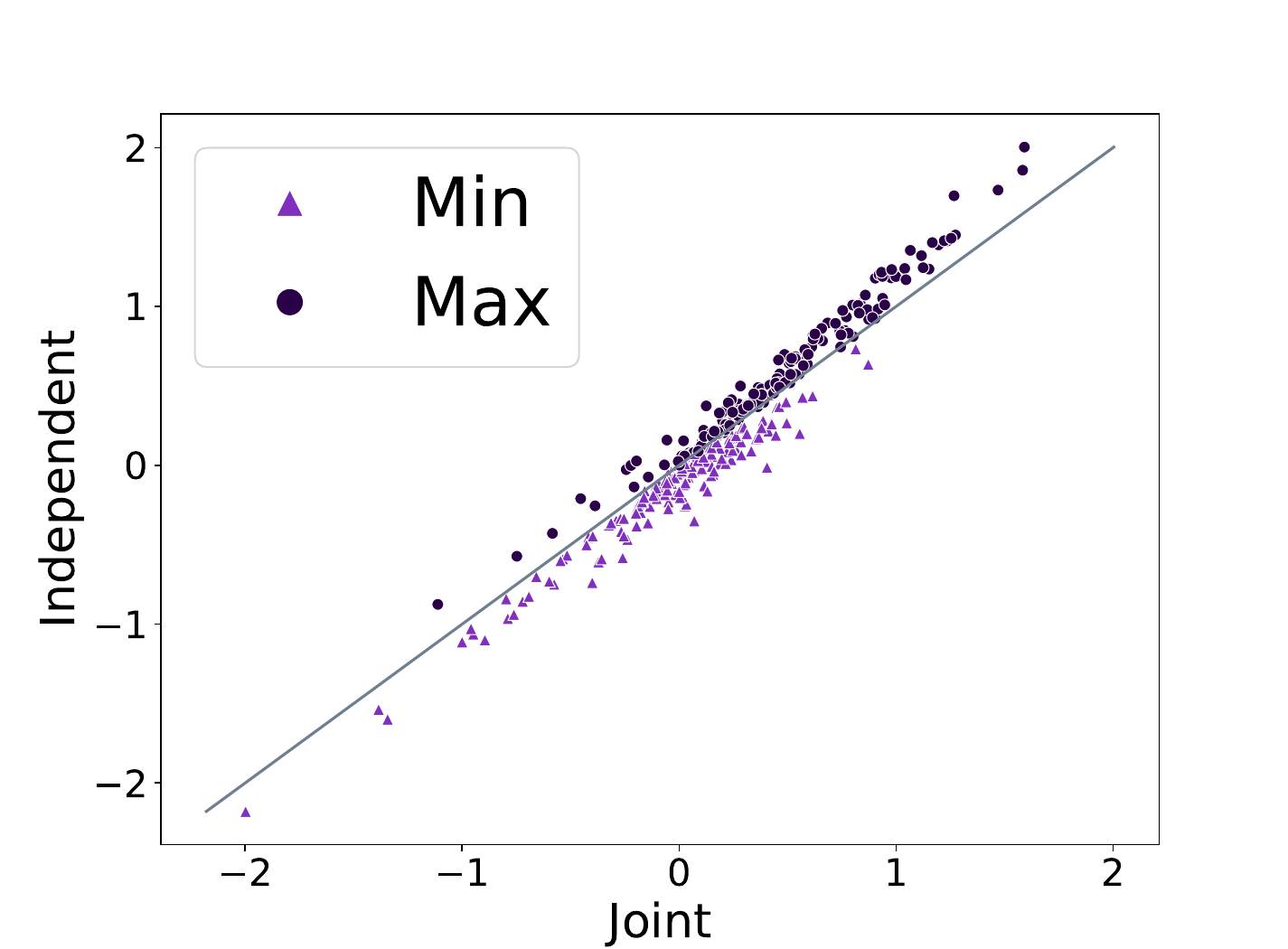}
    \caption{CIFAR10, \gls{VNN} \citep{baninajjarvnn}}
    \label{VNN_FNN_cifar}
  \end{subfigure}
  \begin{subfigure}{0.238\textwidth}
\includegraphics[trim={0cm 0cm 1cm 0.5cm}, clip,width=\linewidth]{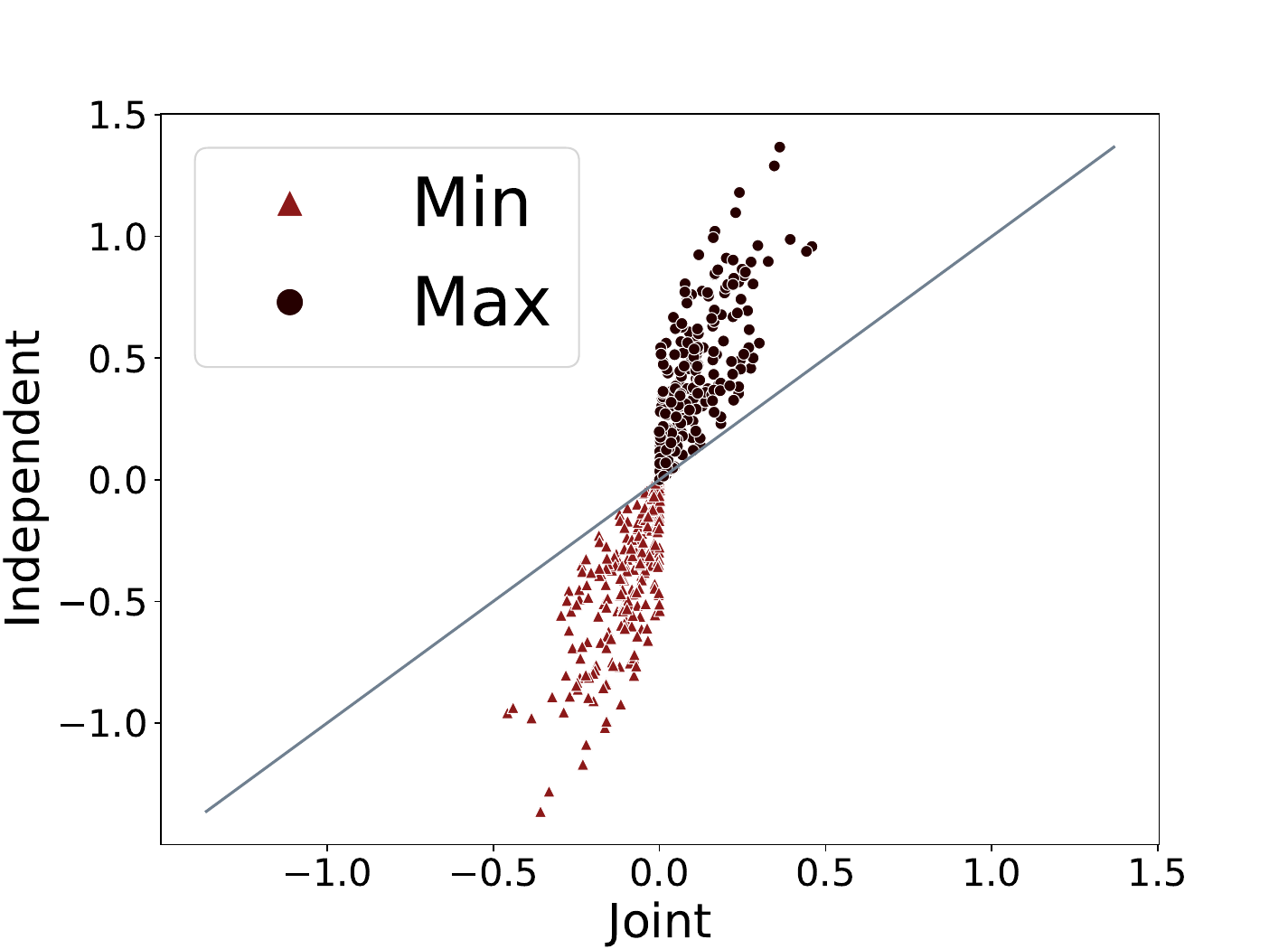}
    \caption{CIFAR10, Quantized \citep{ugare2022proof}}
    \label{quant16_FNN_cifar}
  \end{subfigure}
\begin{subfigure}{0.238\textwidth}
\includegraphics[trim={0cm 0cm 1cm 0.5cm}, clip,width=\linewidth]{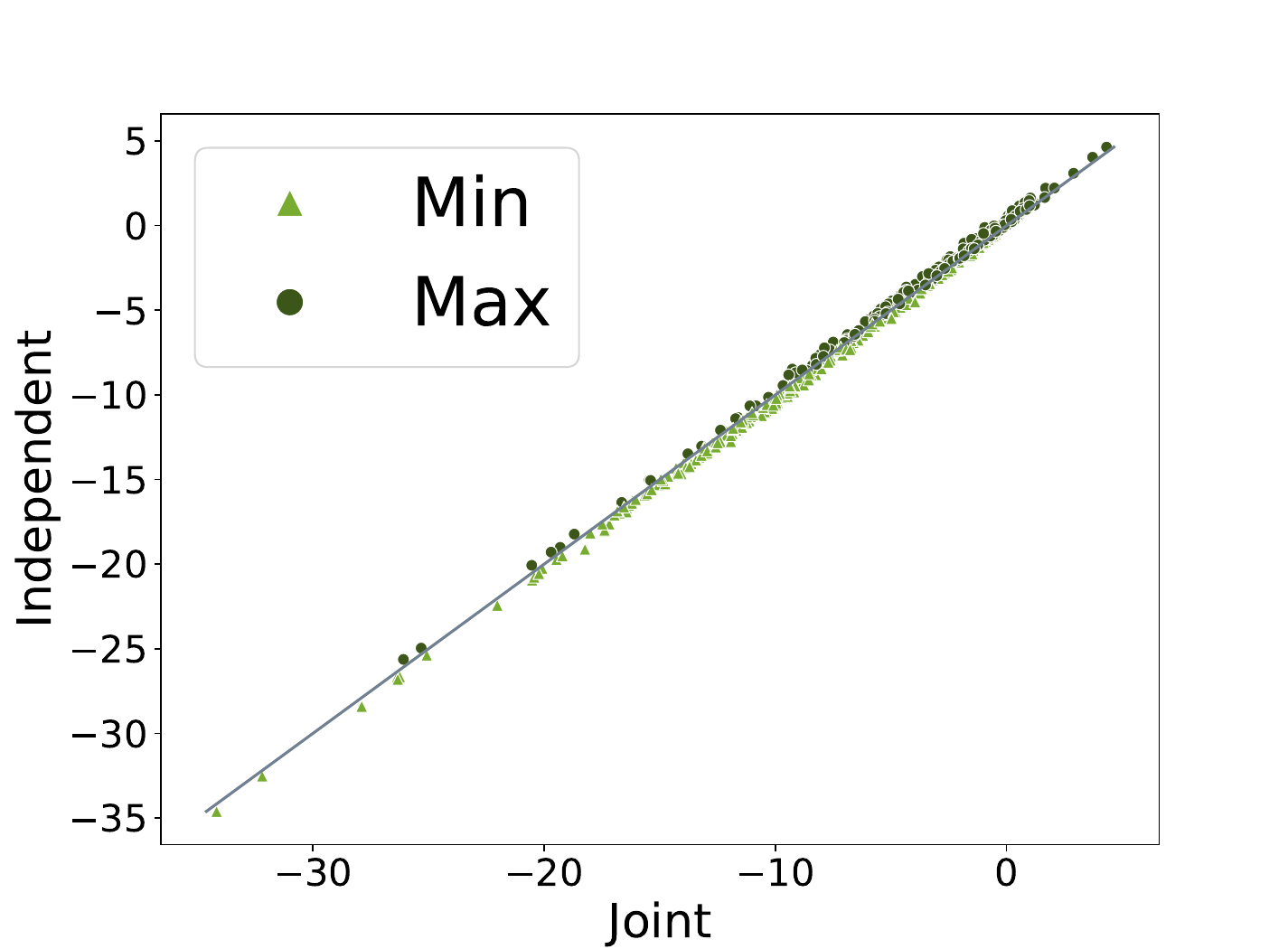}
    \caption{CIFAR10, Distilled \citep{hinton2015distilling}}
    \label{dist5_FNN_cifar}
  \end{subfigure}

 \caption{The minimum and maximum \glspl{LRPR} obtained with joint analysis are compared to those from independent analysis, where Compact networks imply Original networks (Compact $\implies$ Original), with $\delta = 0.01$, on fully-connected \glspl{DNN}.}
  \label{photo_scatter}
\end{figure*}

\begin{figure}[ht]
  \centering

  \begin{subfigure}{0.32\textwidth}
    \includegraphics[width=\linewidth]{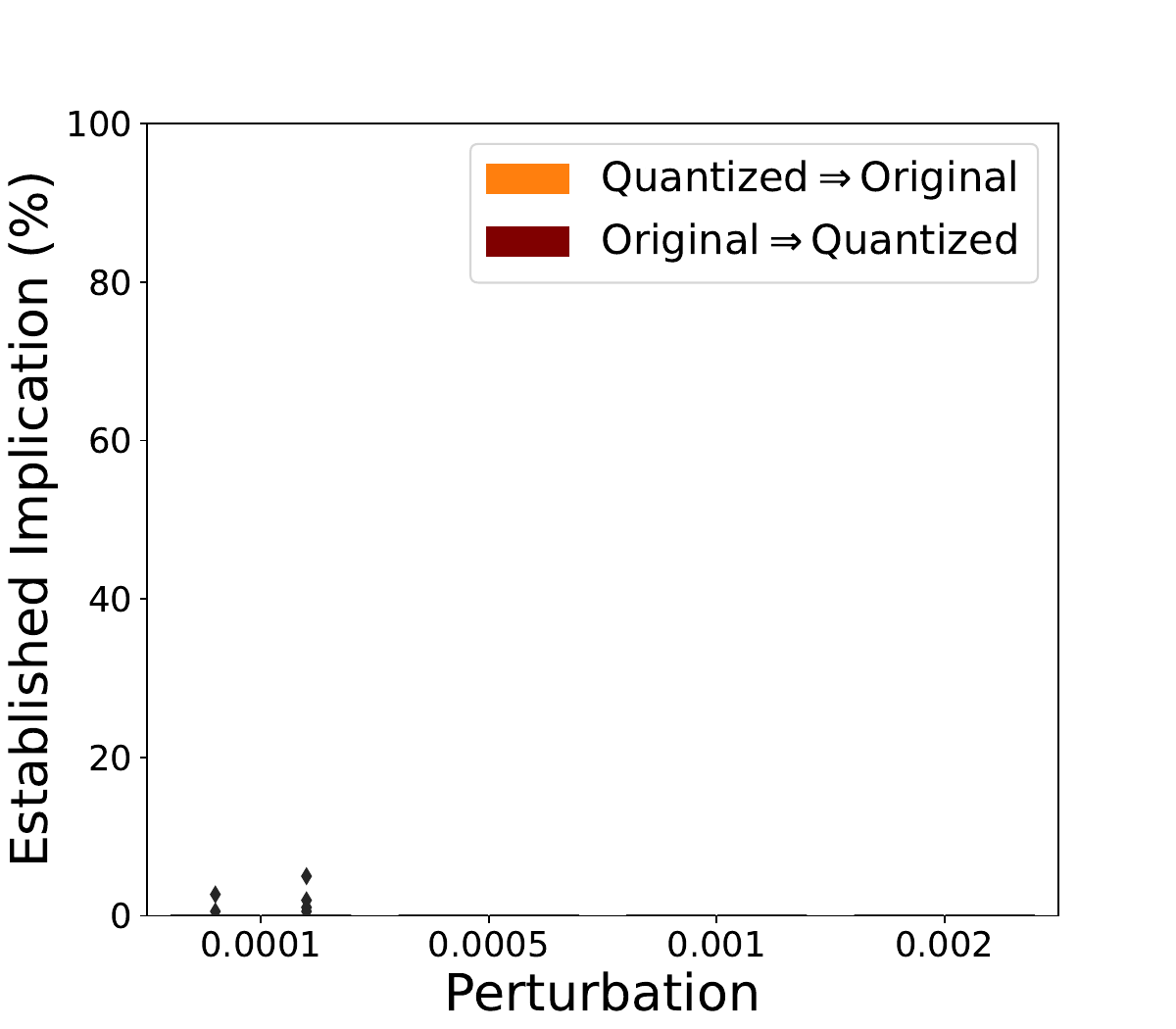}
    \caption{CHB, float}
  \end{subfigure}\label{CHB float}
  \begin{subfigure}{0.32\textwidth}
    \includegraphics[width=\linewidth]{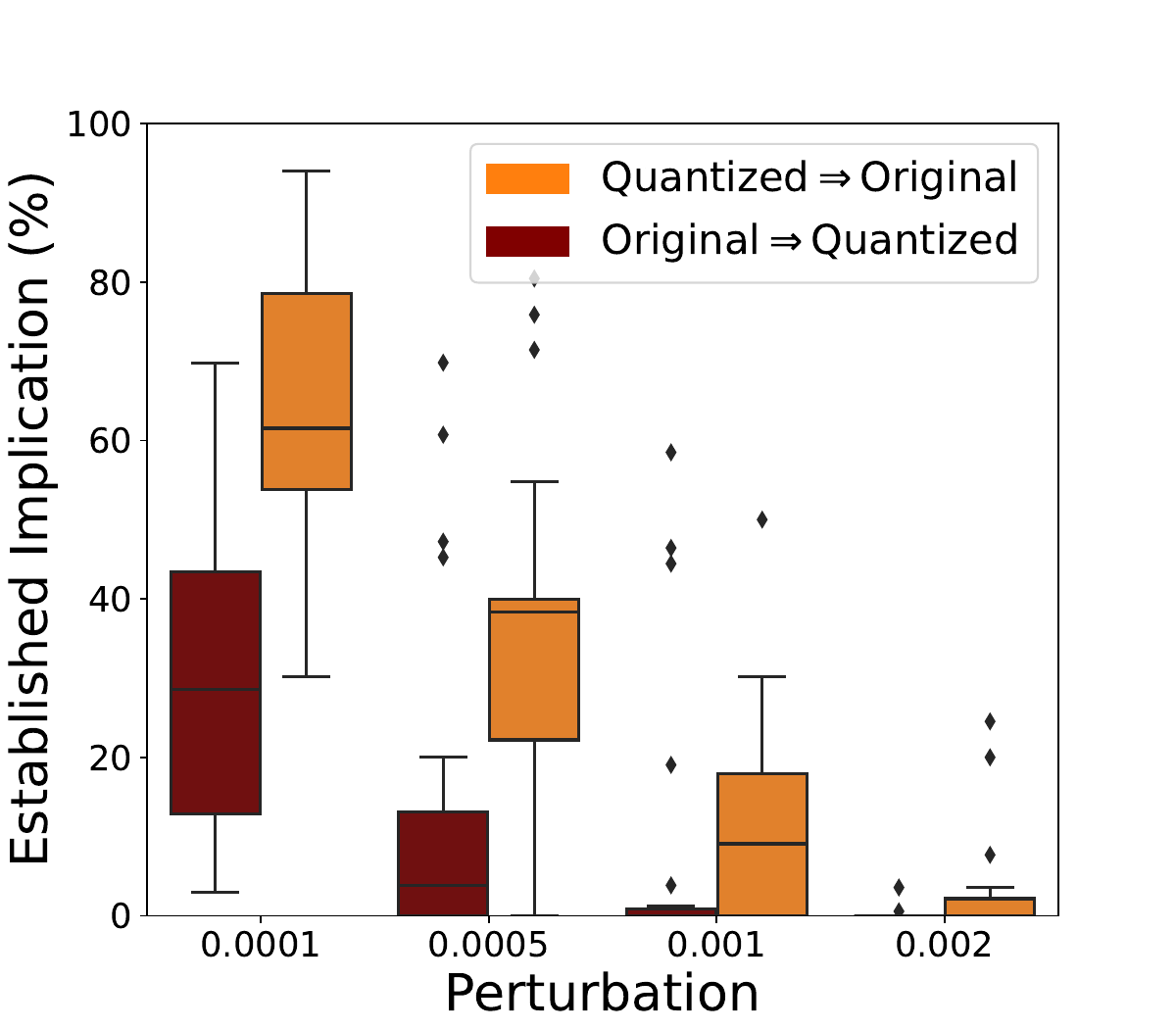}
    \caption{CHB, int8}
  \end{subfigure}\label{CHB int8}
  \begin{subfigure}{0.32\textwidth}
    \includegraphics[width=\linewidth]{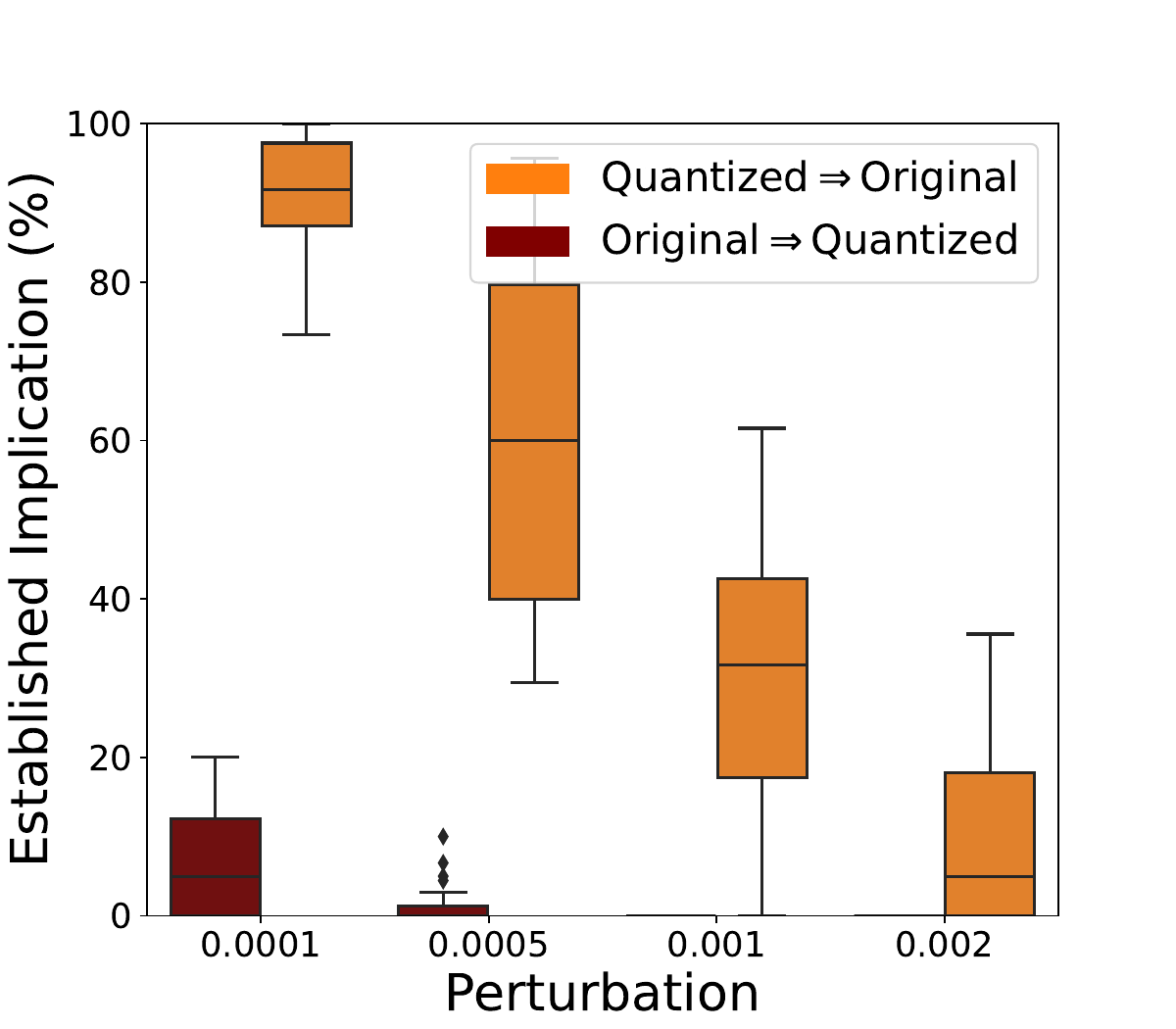}
    \caption{CHB, int4}
  \end{subfigure}\label{CHB int4}

  \begin{subfigure}{0.32\textwidth}
    \includegraphics[width=\linewidth]{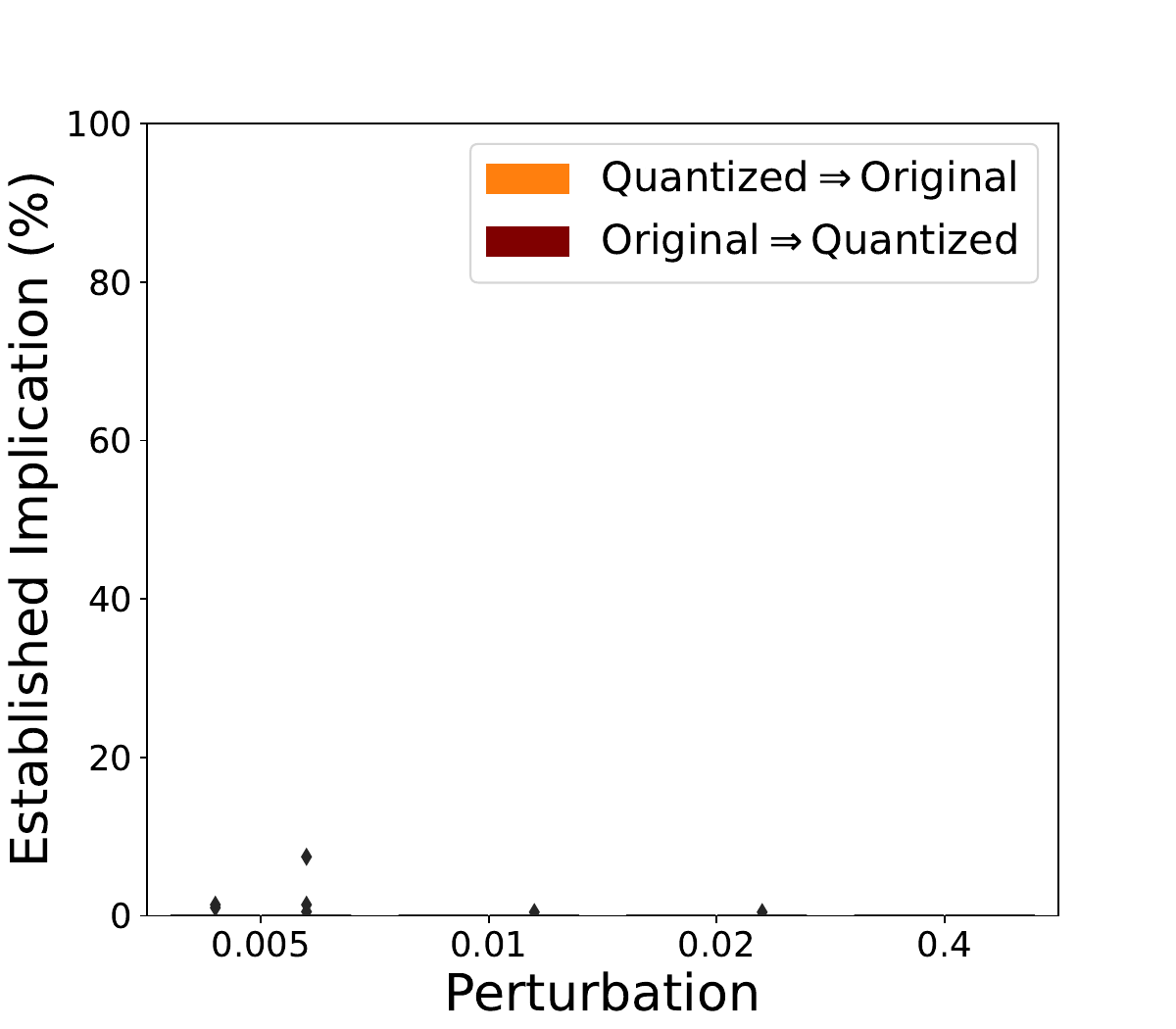}
    \caption{BIH, float}
  \end{subfigure}\label{BIH, float}
  \begin{subfigure}{0.32\textwidth}
    \includegraphics[width=\linewidth]{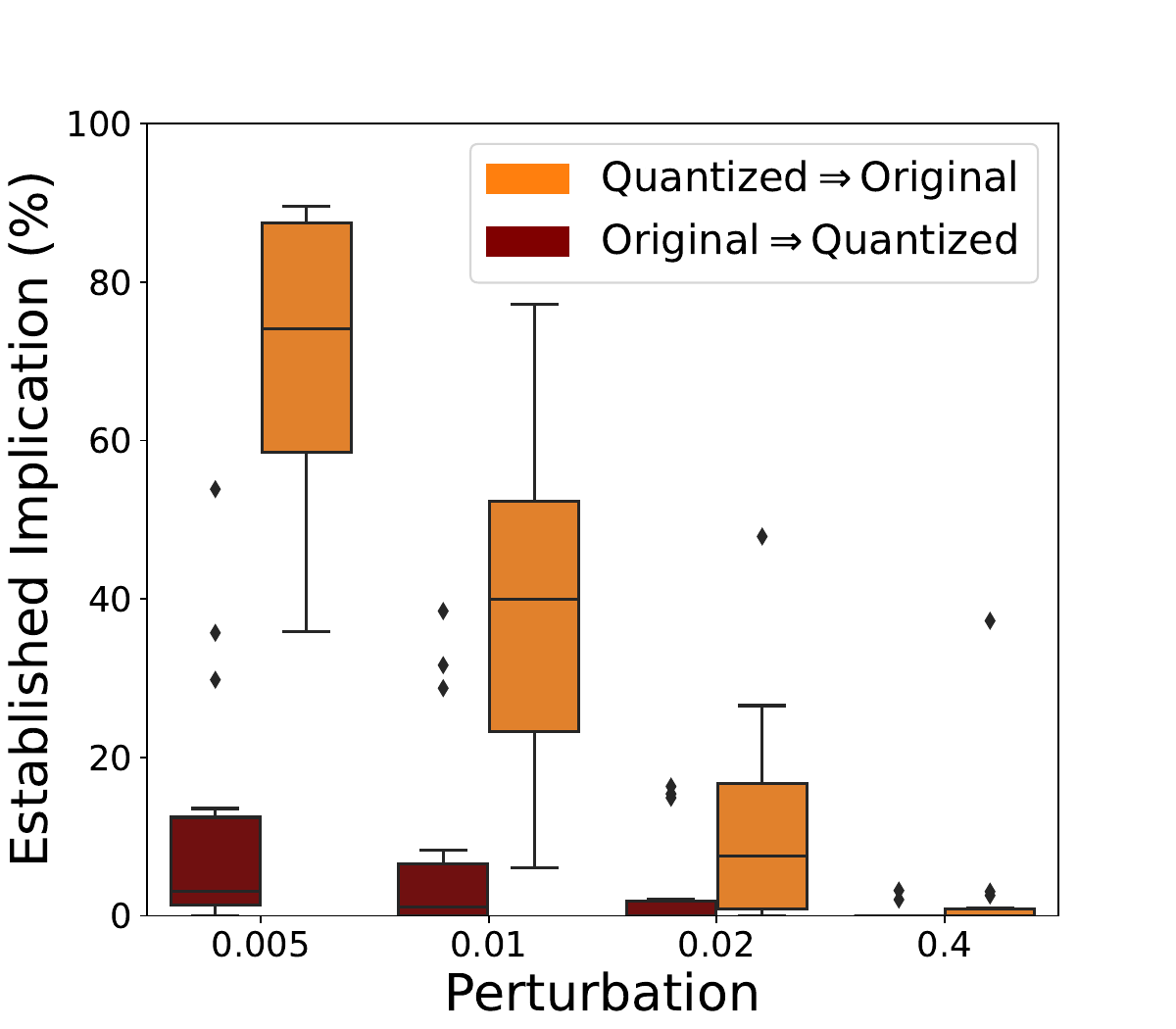}
    \caption{BIH, int8}
  \end{subfigure}\label{BIH int8}
  \begin{subfigure}{0.32\textwidth}
    \includegraphics[width=\linewidth]{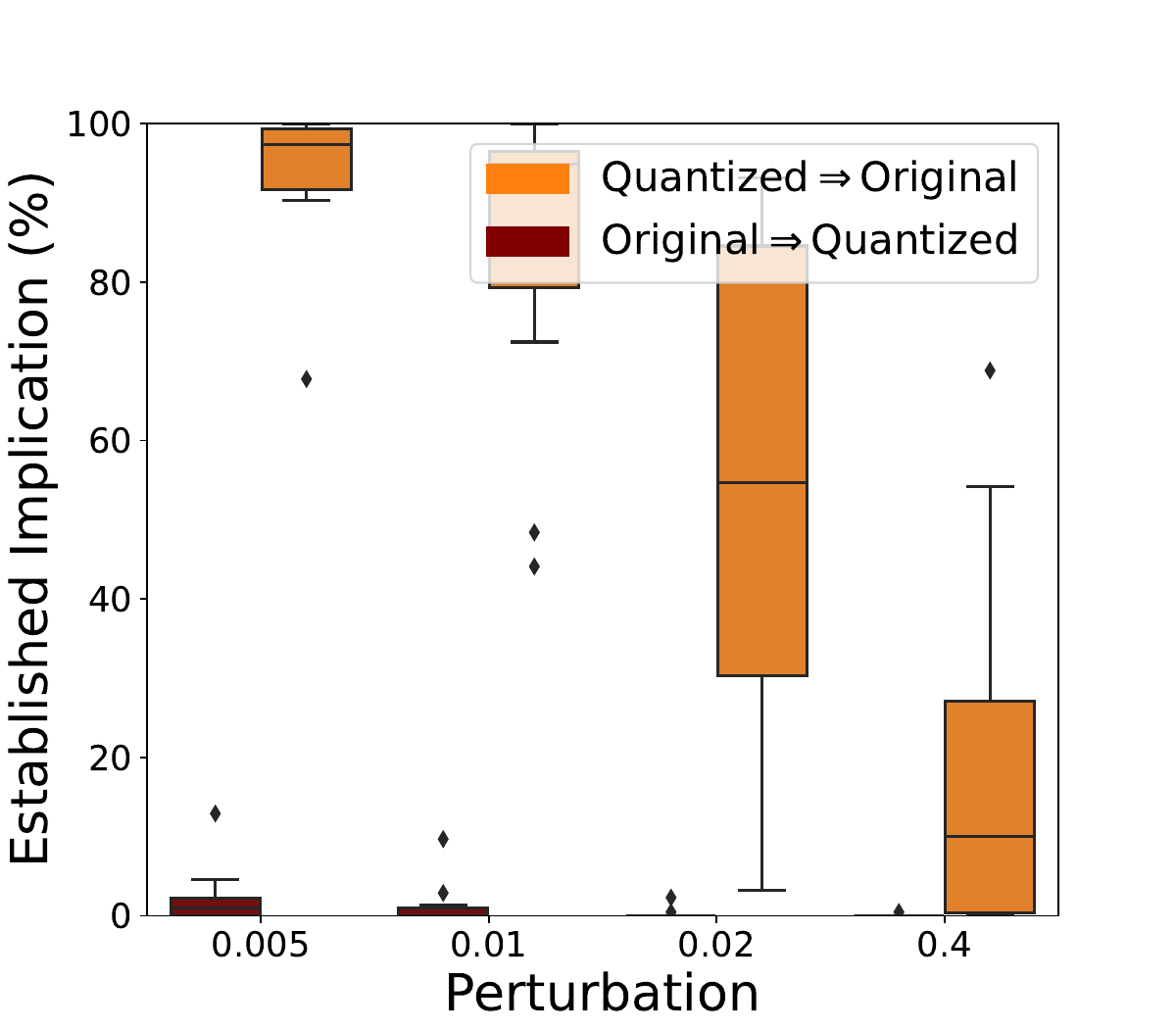}
    \caption{BIH, int4}
  \end{subfigure}\label{BIH int4}

  \caption{Box plots illustrate the established implication of convolutional \glspl{DNN} trained on all patients in the CHB-MIT and MIT-BIH datasets for Original and Quantized networks. For each patient in the dataset, we evaluate the implication between an Original and a Quantized network using the patient’s own data and aggregate the results across all patients to present them in the box plots.}
  \label{CHB__BIH_quantized_appendix}
\end{figure}

\begin{figure}[ht]
  \centering

  \begin{subfigure}{0.24\textwidth}
    \includegraphics[width=\linewidth]{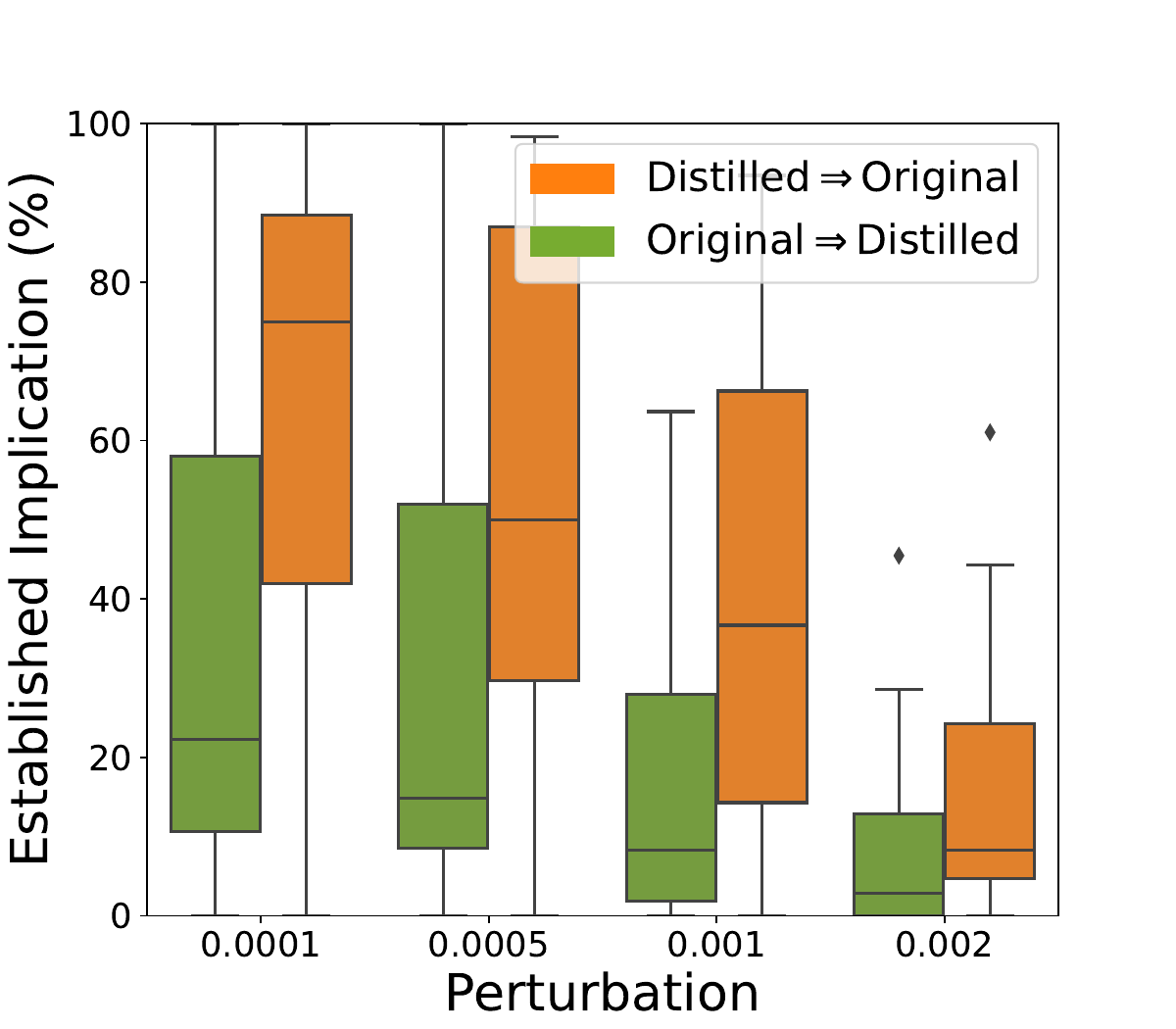}
    \caption{CHB, T = 1}
  \end{subfigure}
  \hfill
  \begin{subfigure}{0.24\textwidth}
    \includegraphics[width=\linewidth]{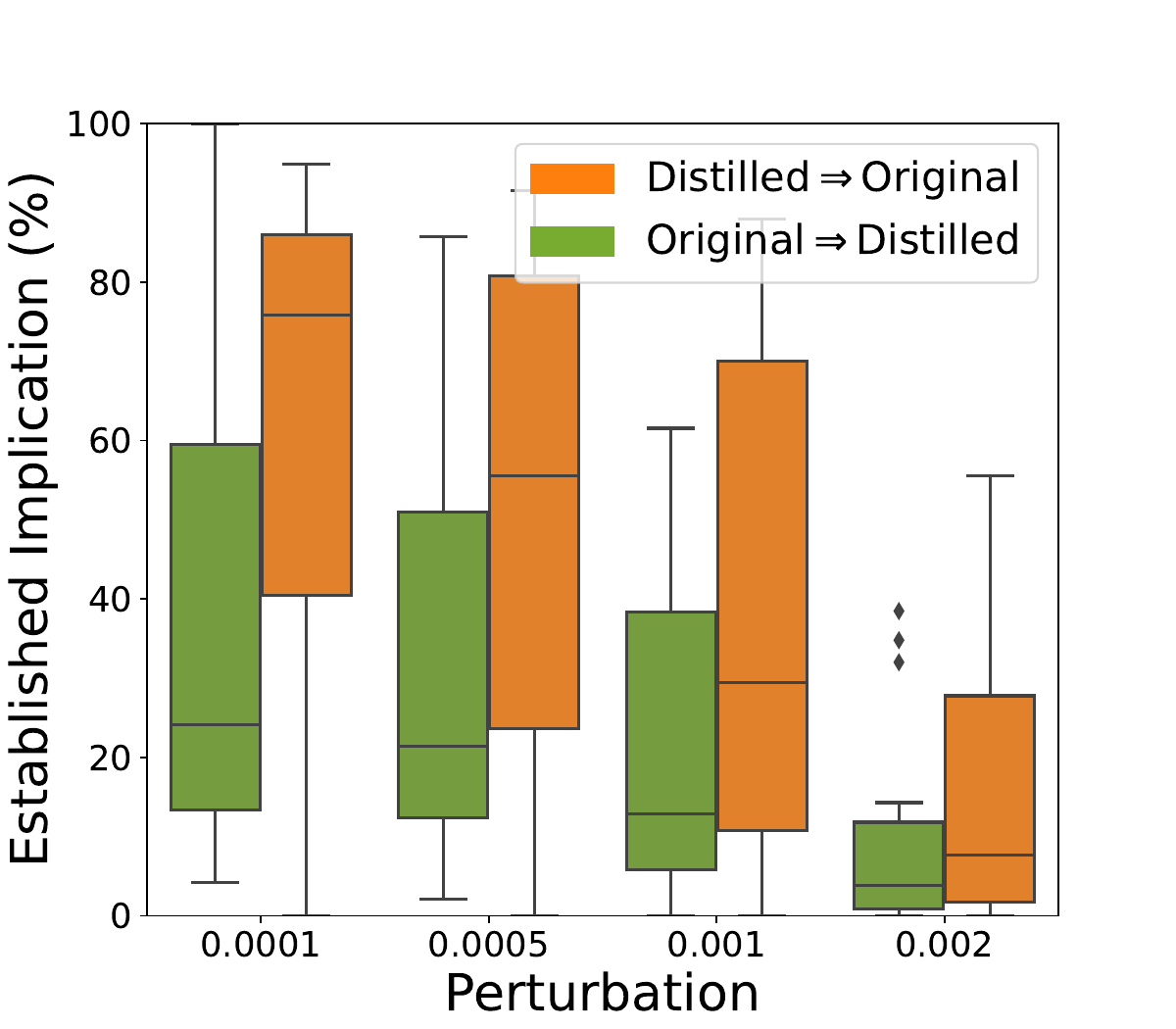}
    \caption{CHB, T = 2}
  \end{subfigure}
  \hfill
  \begin{subfigure}{0.24\textwidth}
    \includegraphics[width=\linewidth]{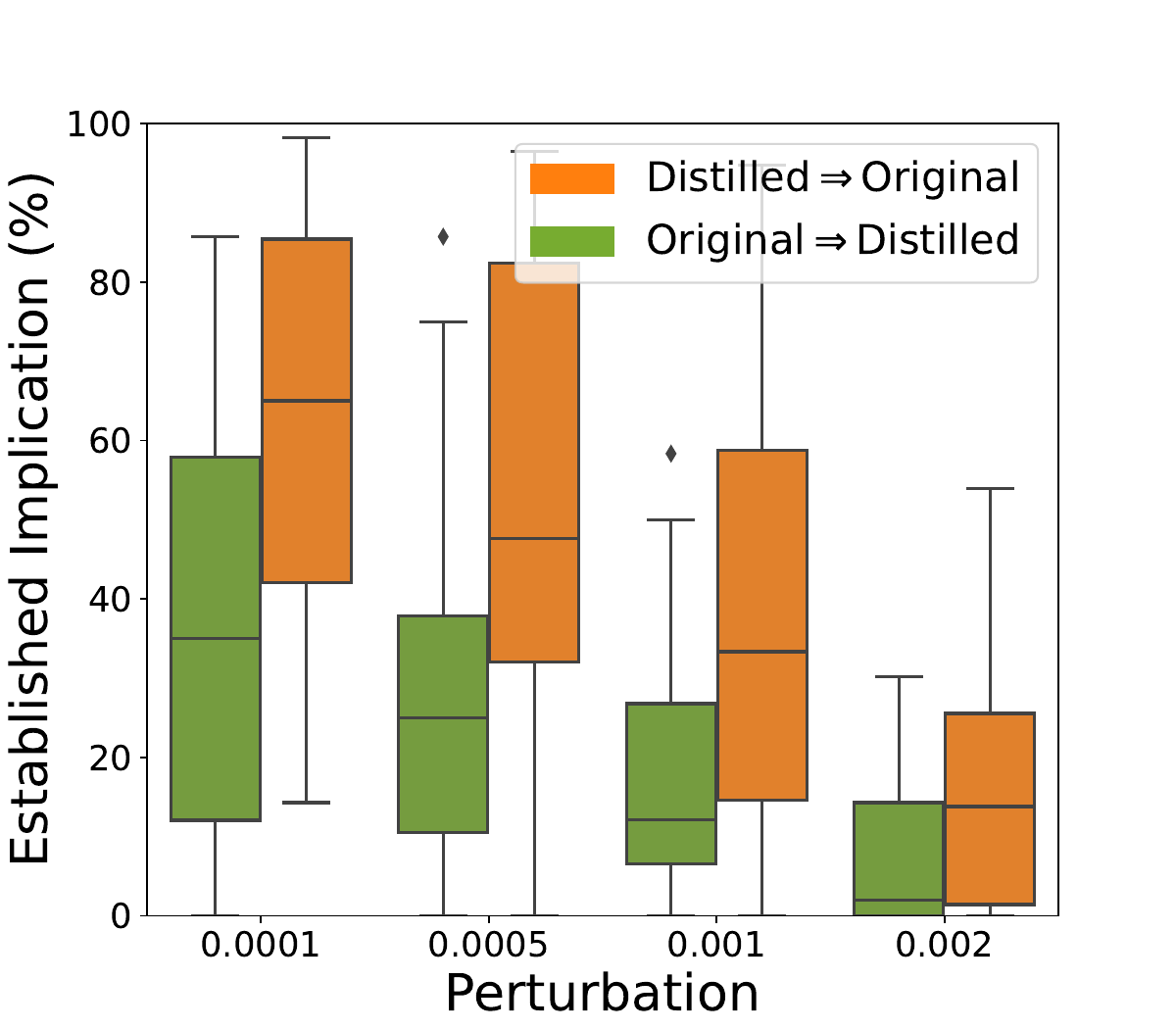}
    \caption{CHB, T = 3}
  \end{subfigure}
  \hfill
  \begin{subfigure}{0.24\textwidth}
    \includegraphics[width=\linewidth]{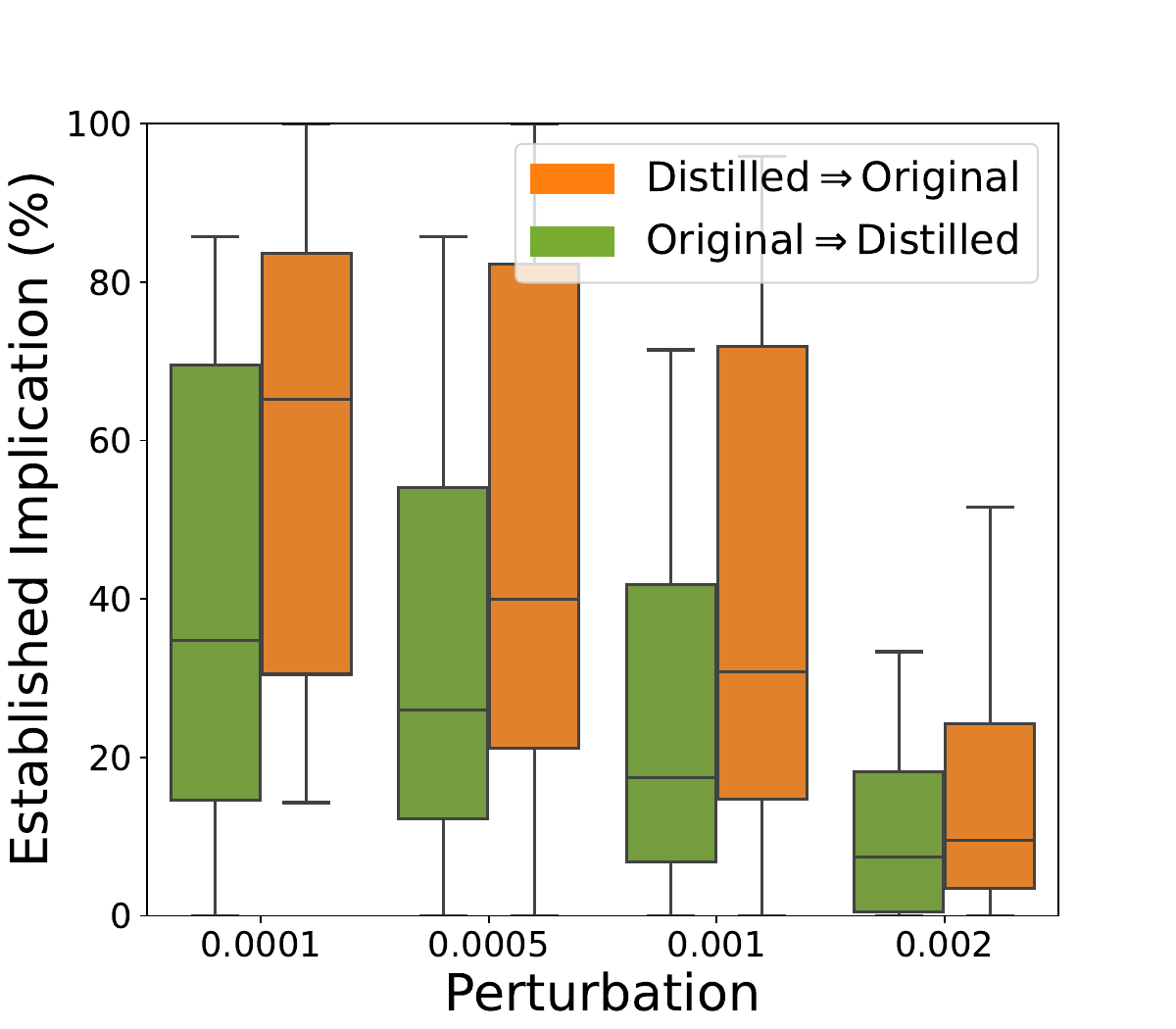}
    \caption{CHB, T = 4}
  \end{subfigure}


  \begin{subfigure}{0.24\textwidth}
    \includegraphics[width=\linewidth]{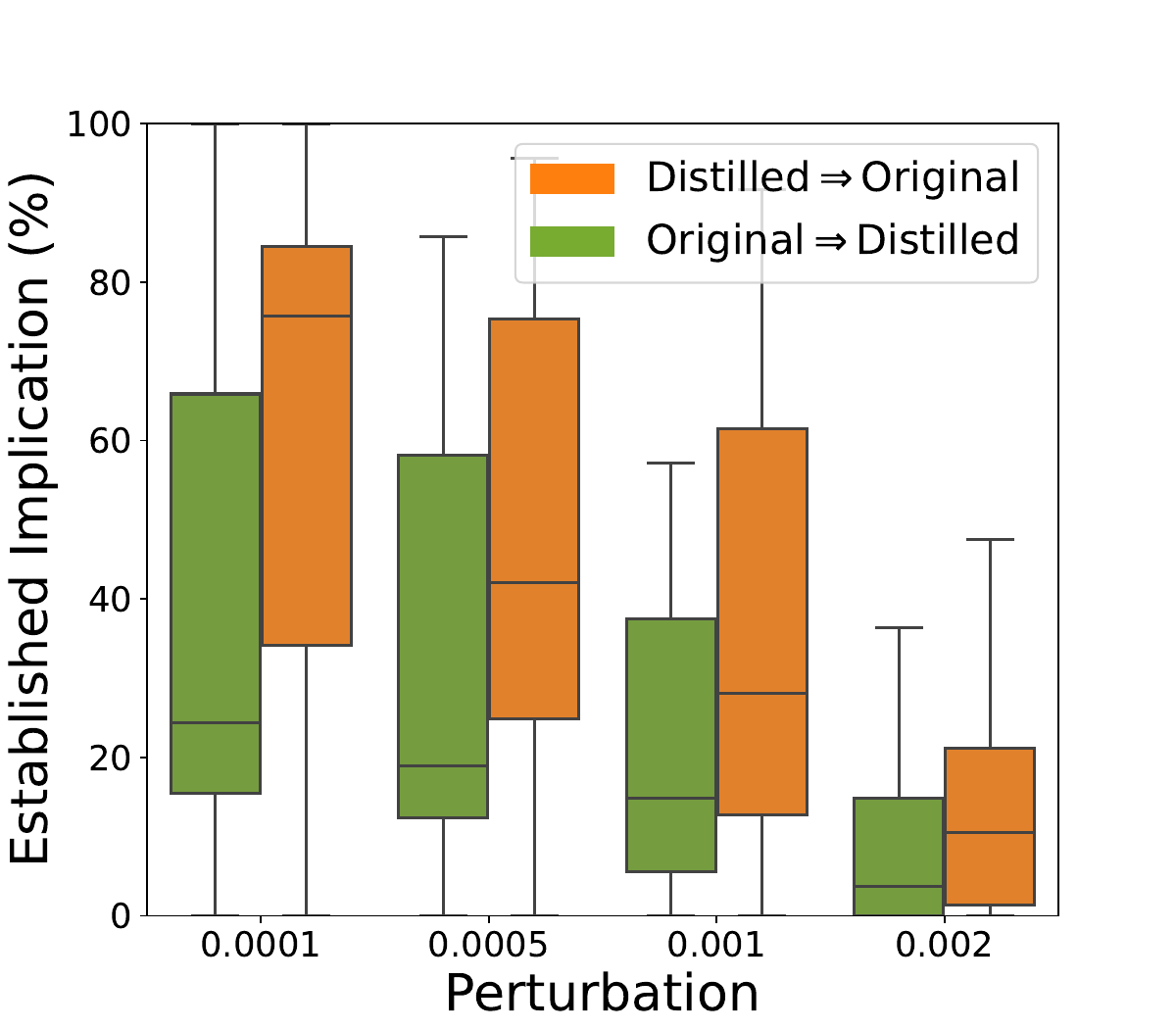}
    \caption{CHB, T = 6}
  \end{subfigure}
  \hfill
  \begin{subfigure}{0.24\textwidth}
    \includegraphics[width=\linewidth]{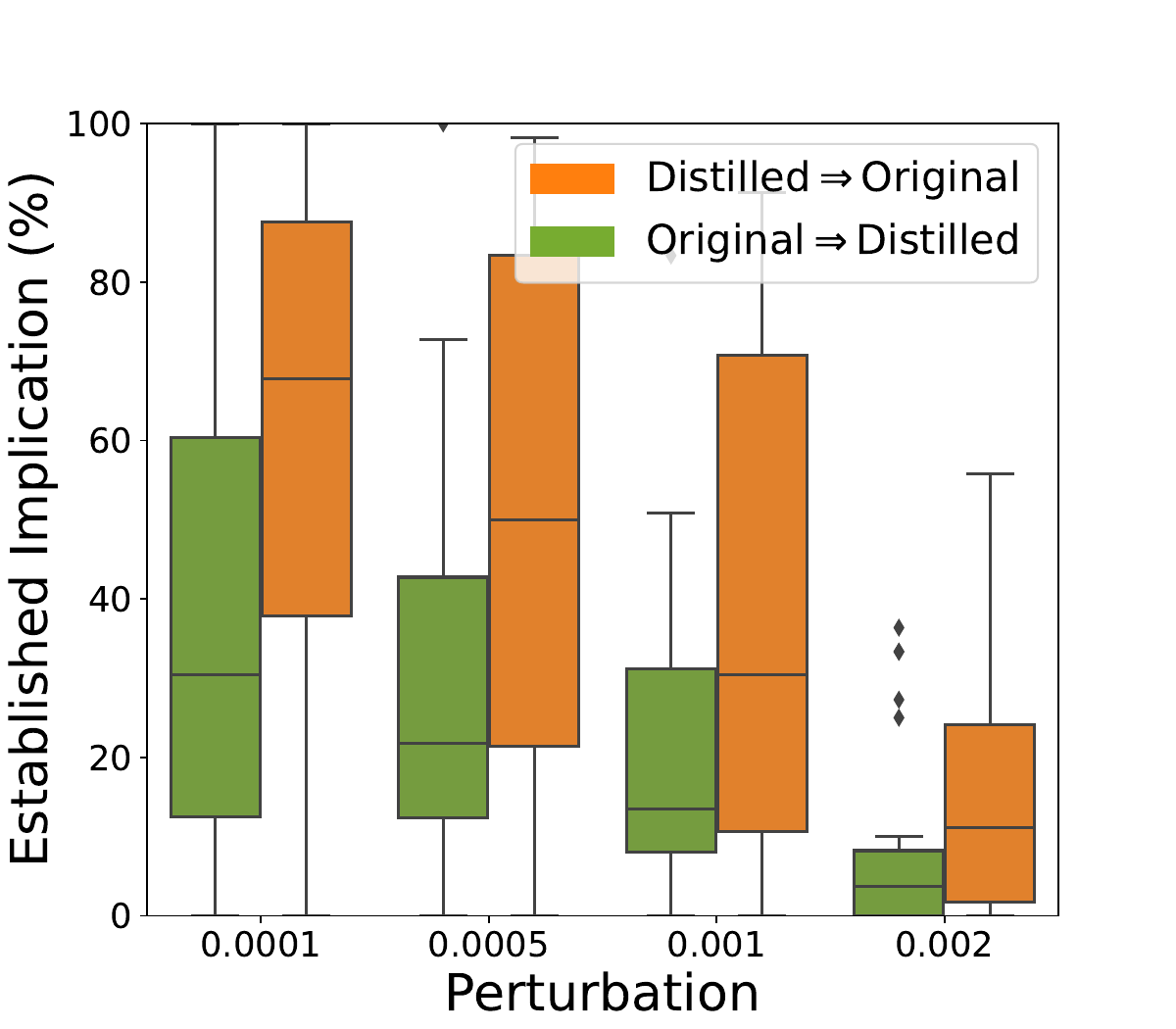}
    \caption{CHB, T = 7}
  \end{subfigure}
  \hfill
  \begin{subfigure}{0.24\textwidth}
    \includegraphics[width=\linewidth]{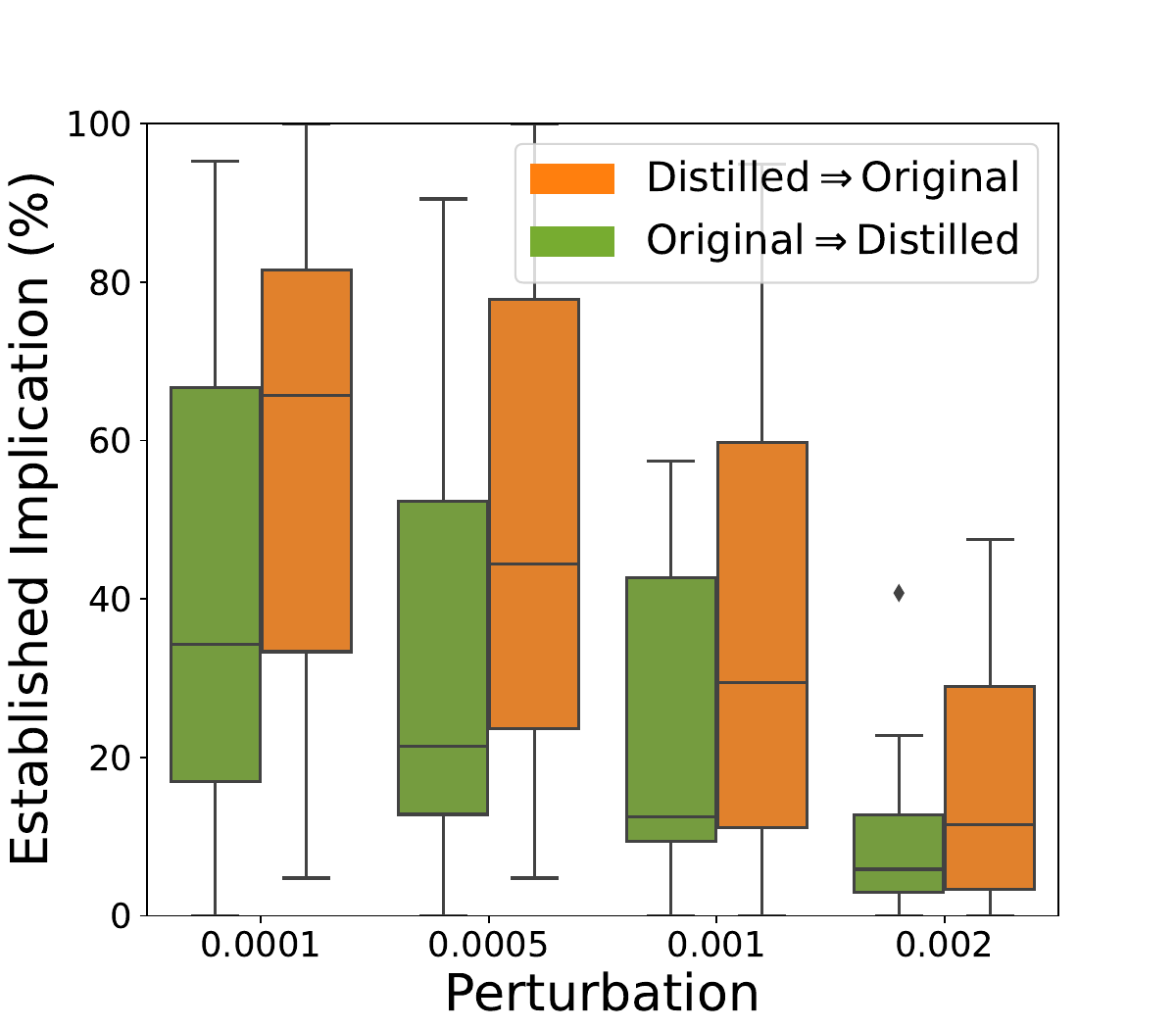}
    \caption{CHB, T = 8}
  \end{subfigure}
  \hfill
  \begin{subfigure}{0.24\textwidth}
    \includegraphics[width=\linewidth]{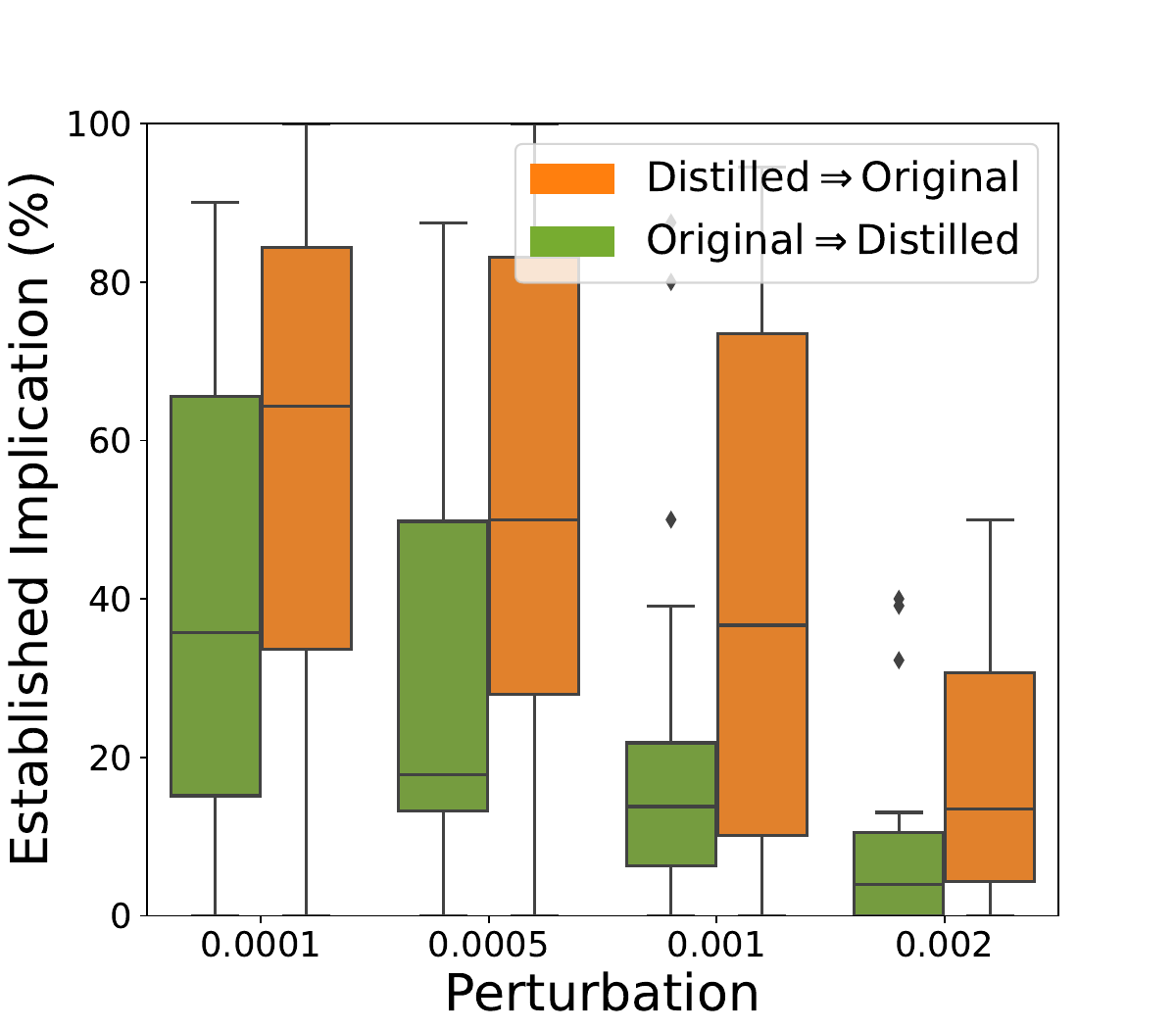}
    \caption{CHB, T = 9}
  \end{subfigure}

    \begin{subfigure}{0.24\textwidth}
    \includegraphics[width=\linewidth]{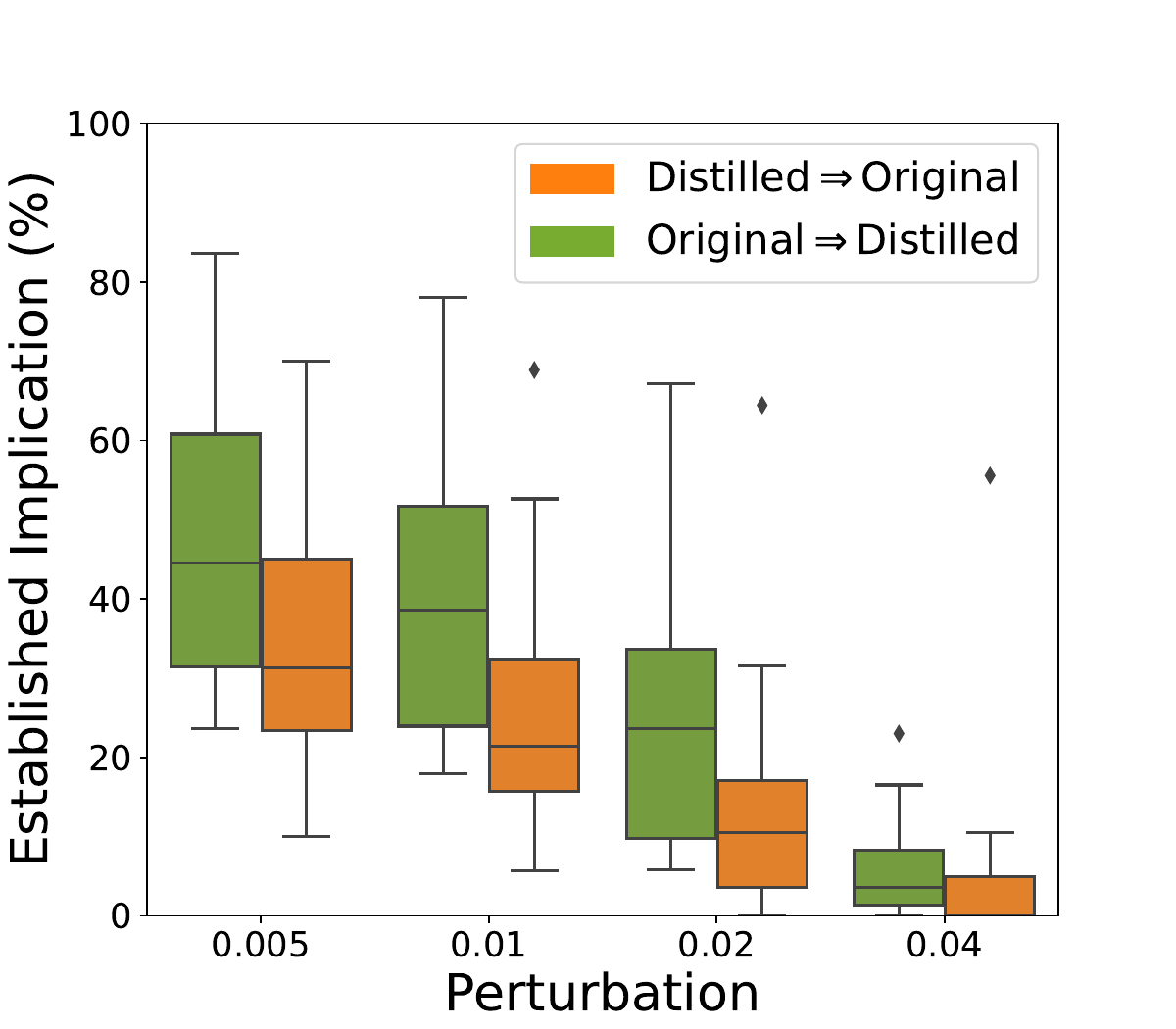}
    \caption{BIH, T = 1}
  \end{subfigure}
  \hfill
  \begin{subfigure}{0.24\textwidth}
    \includegraphics[width=\linewidth]{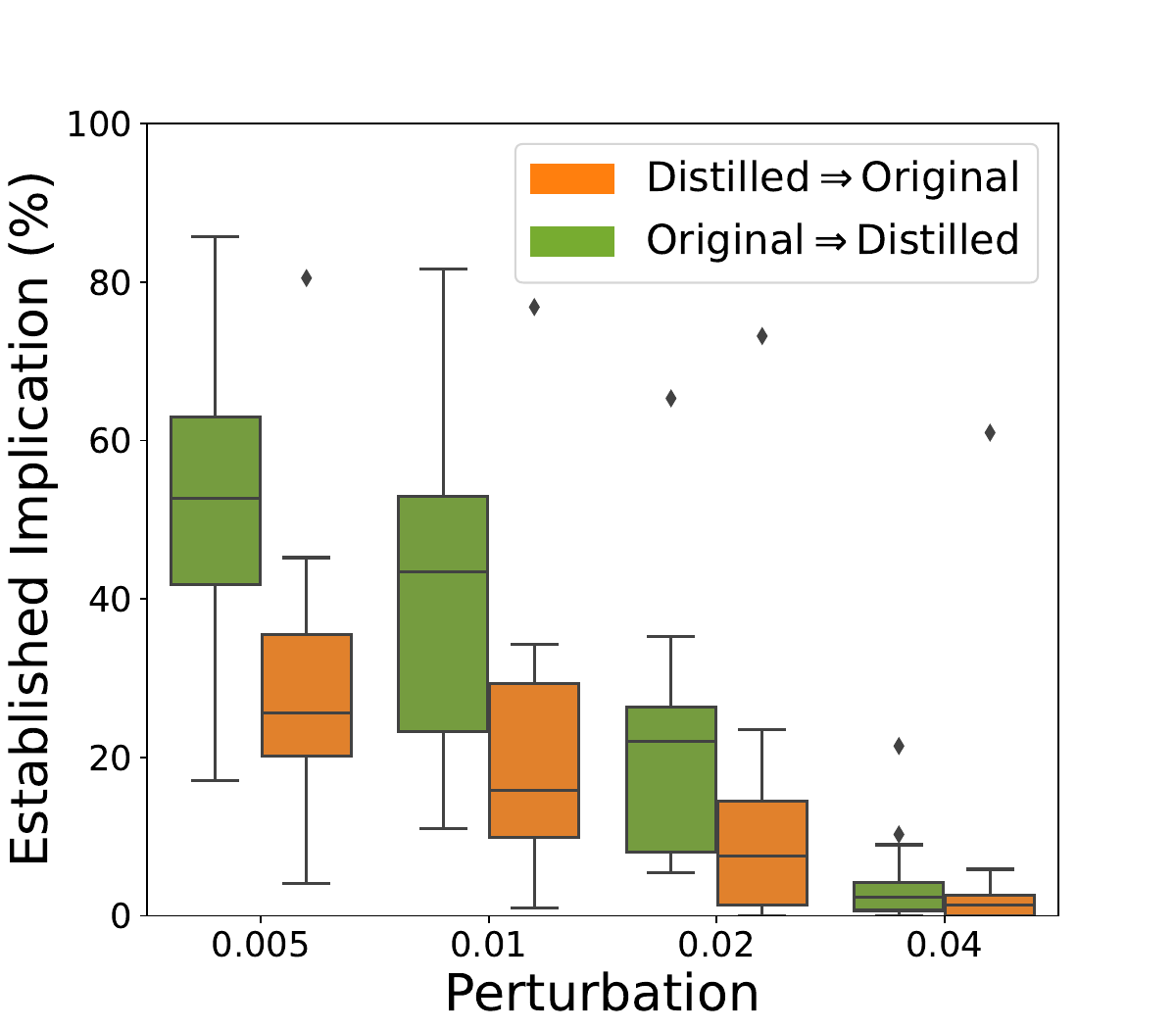}
    \caption{BIH, T = 2}
  \end{subfigure}
  \hfill
  \begin{subfigure}{0.24\textwidth}
    \includegraphics[width=\linewidth]{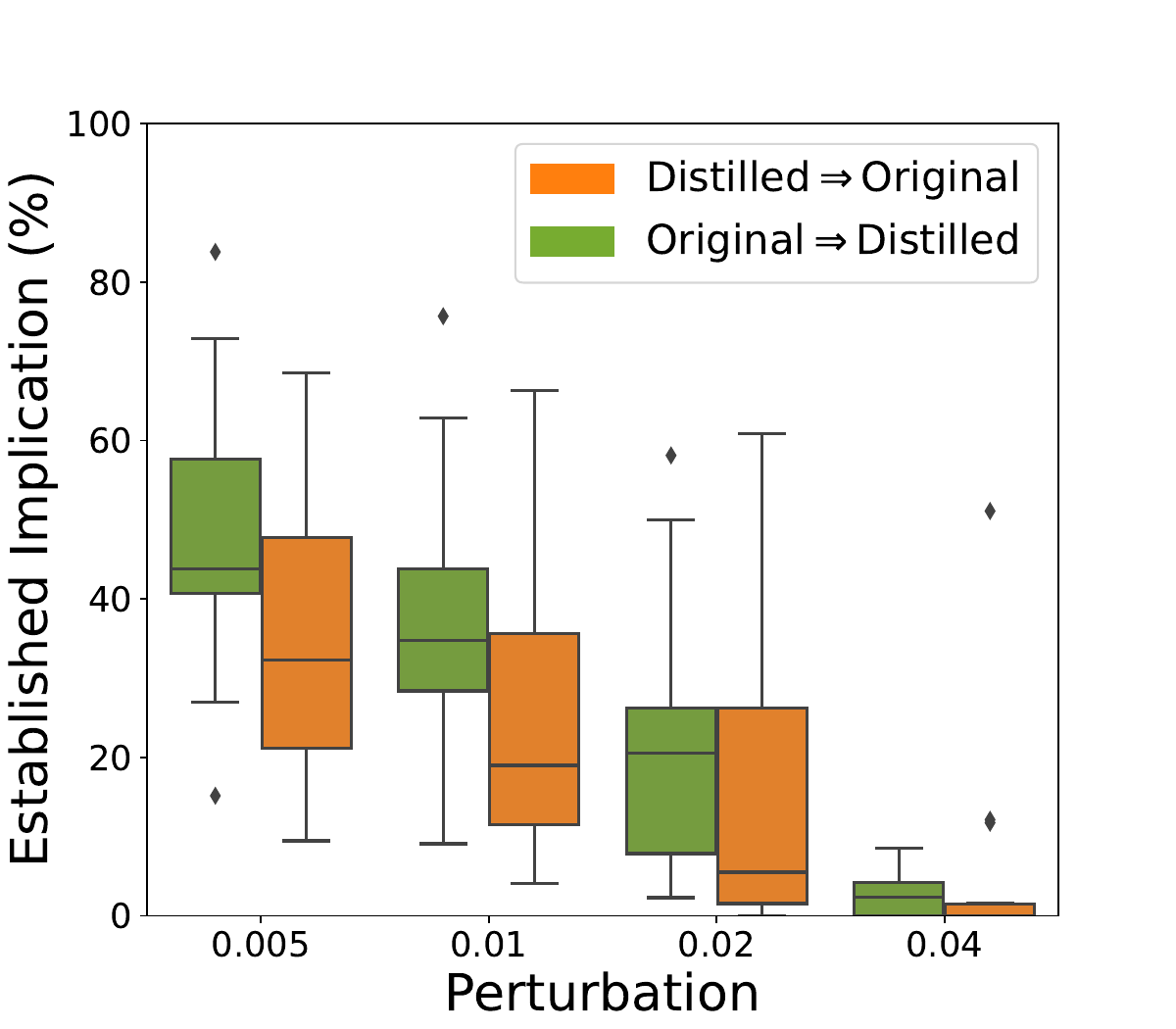}
    \caption{BIH, T = 3}
  \end{subfigure}
  \hfill
  \begin{subfigure}{0.24\textwidth}
    \includegraphics[width=\linewidth]{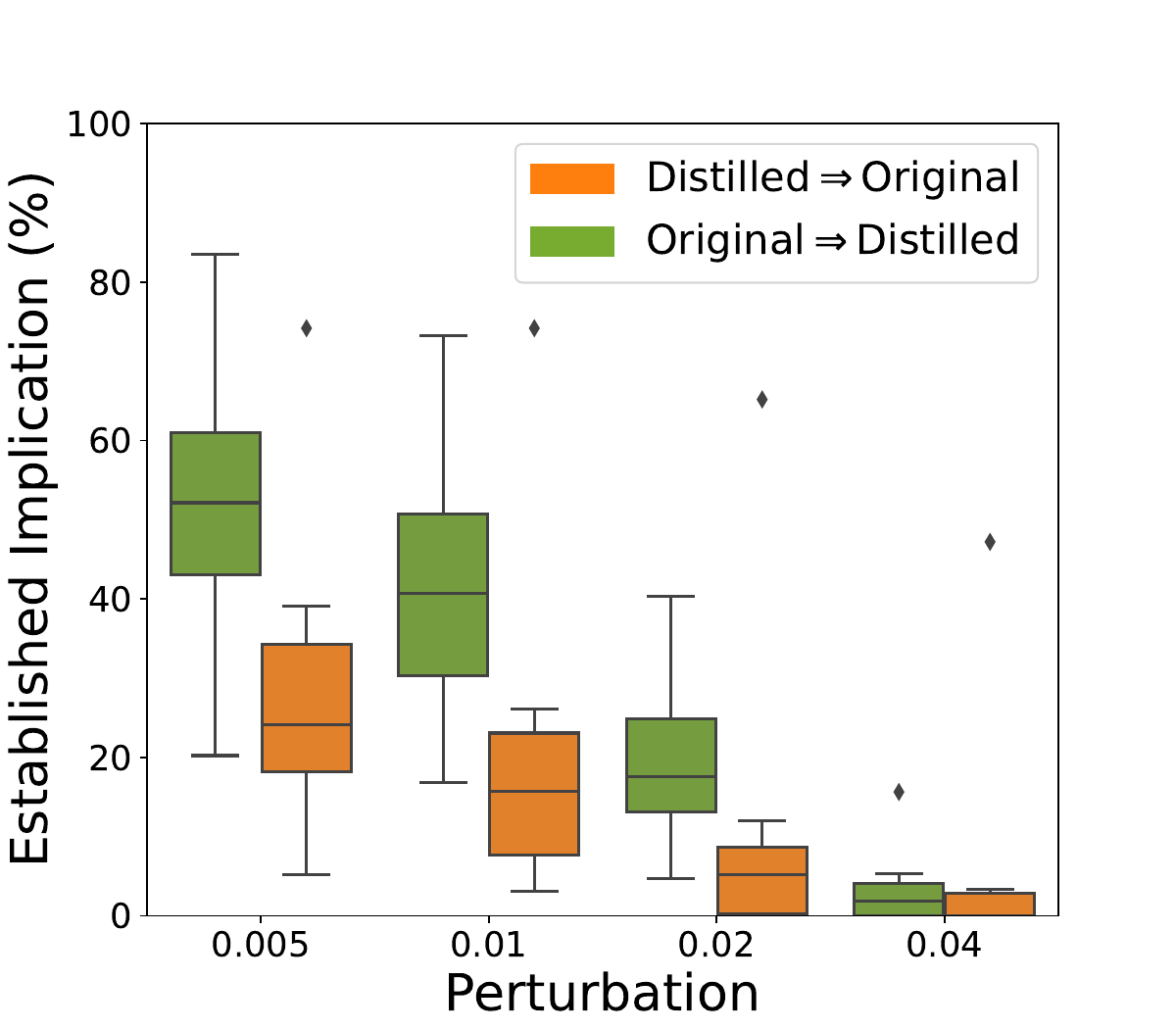}
    \caption{BIH, T = 4}
  \end{subfigure}


  \begin{subfigure}{0.24\textwidth}
    \includegraphics[width=\linewidth]{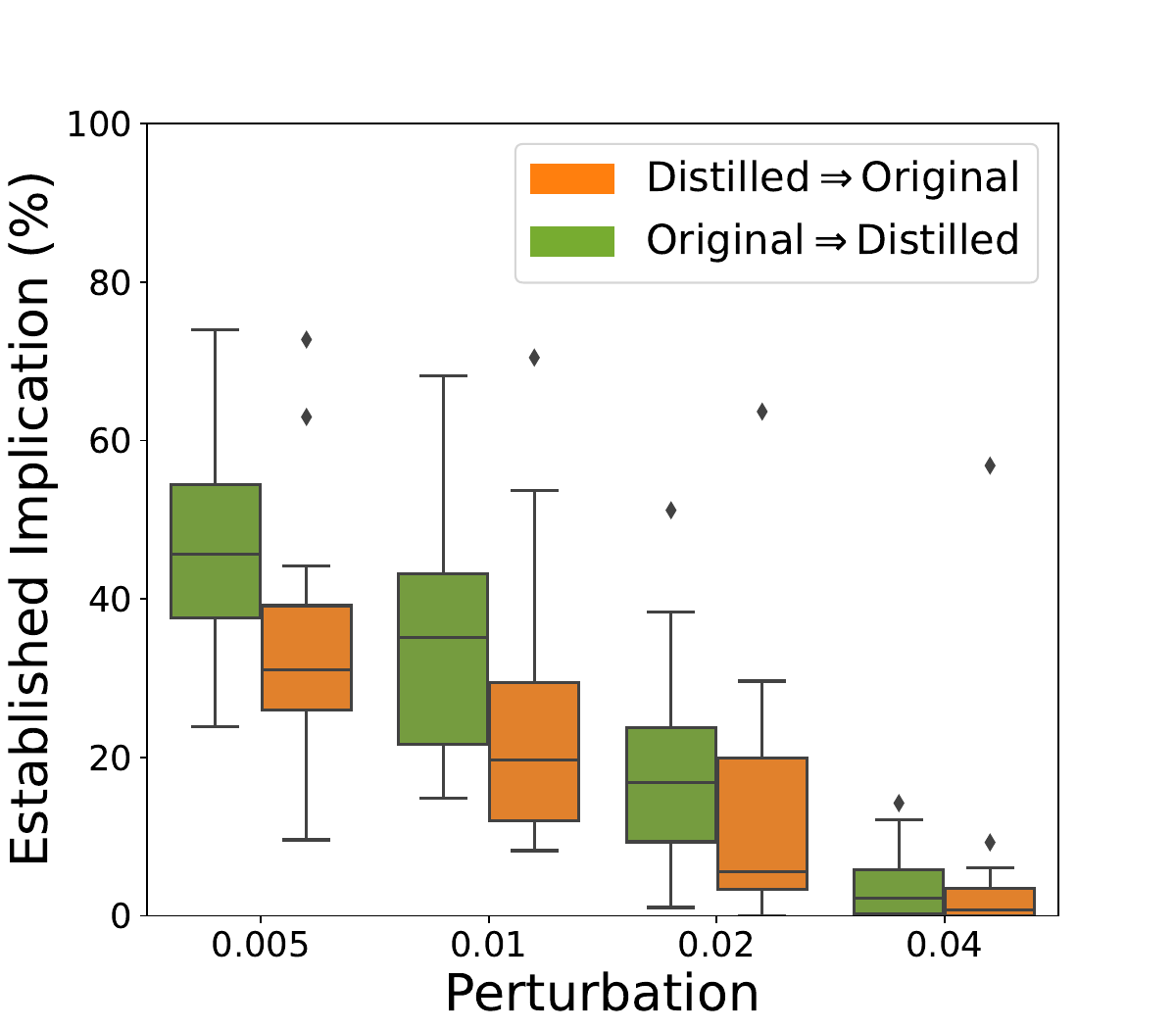}
    \caption{BIH, T = 6}
  \end{subfigure}
  \hfill
  \begin{subfigure}{0.24\textwidth}
    \includegraphics[width=\linewidth]{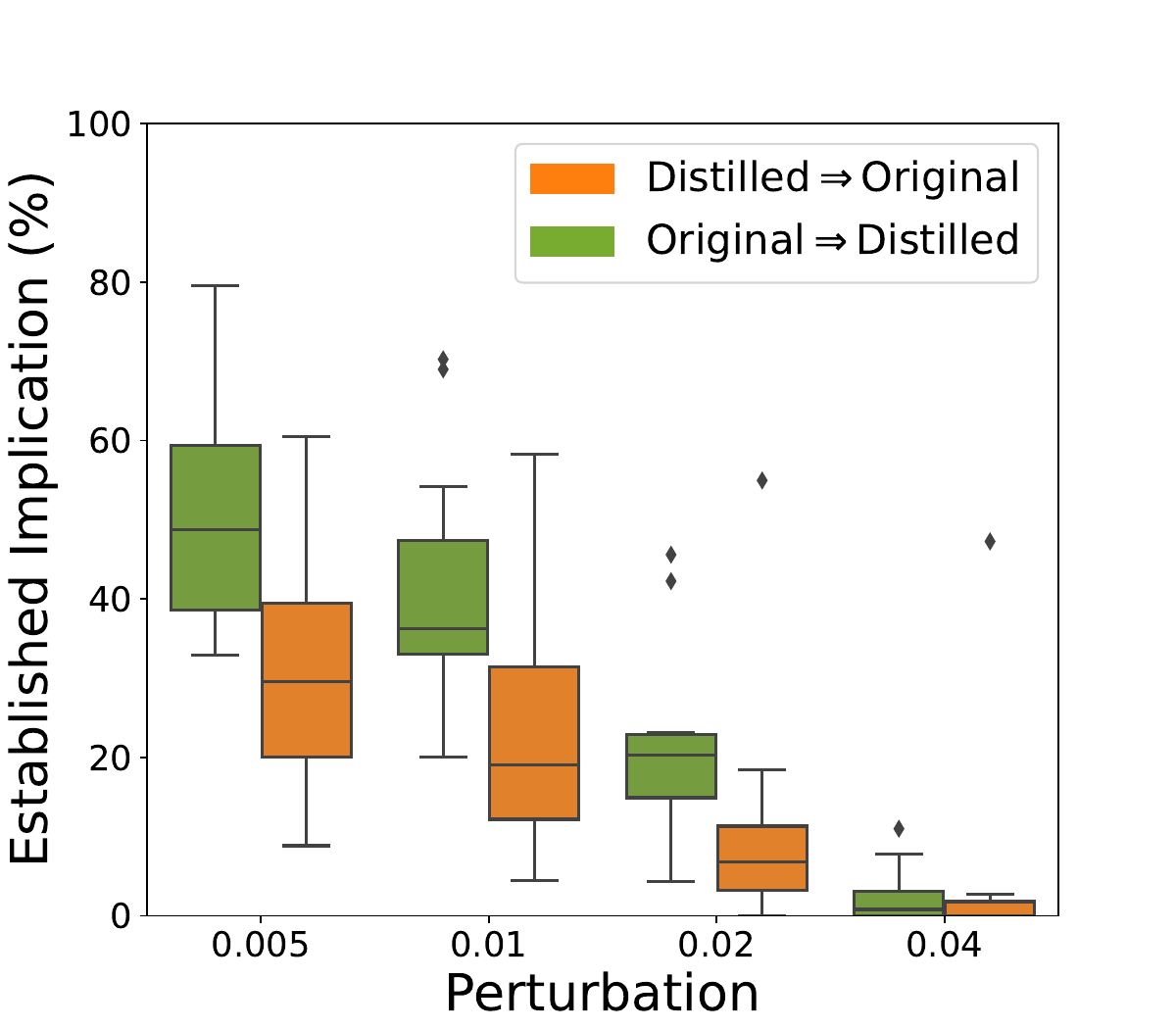}
    \caption{BIH, T = 7}
  \end{subfigure}
  \hfill
  \begin{subfigure}{0.24\textwidth}
    \includegraphics[width=\linewidth]{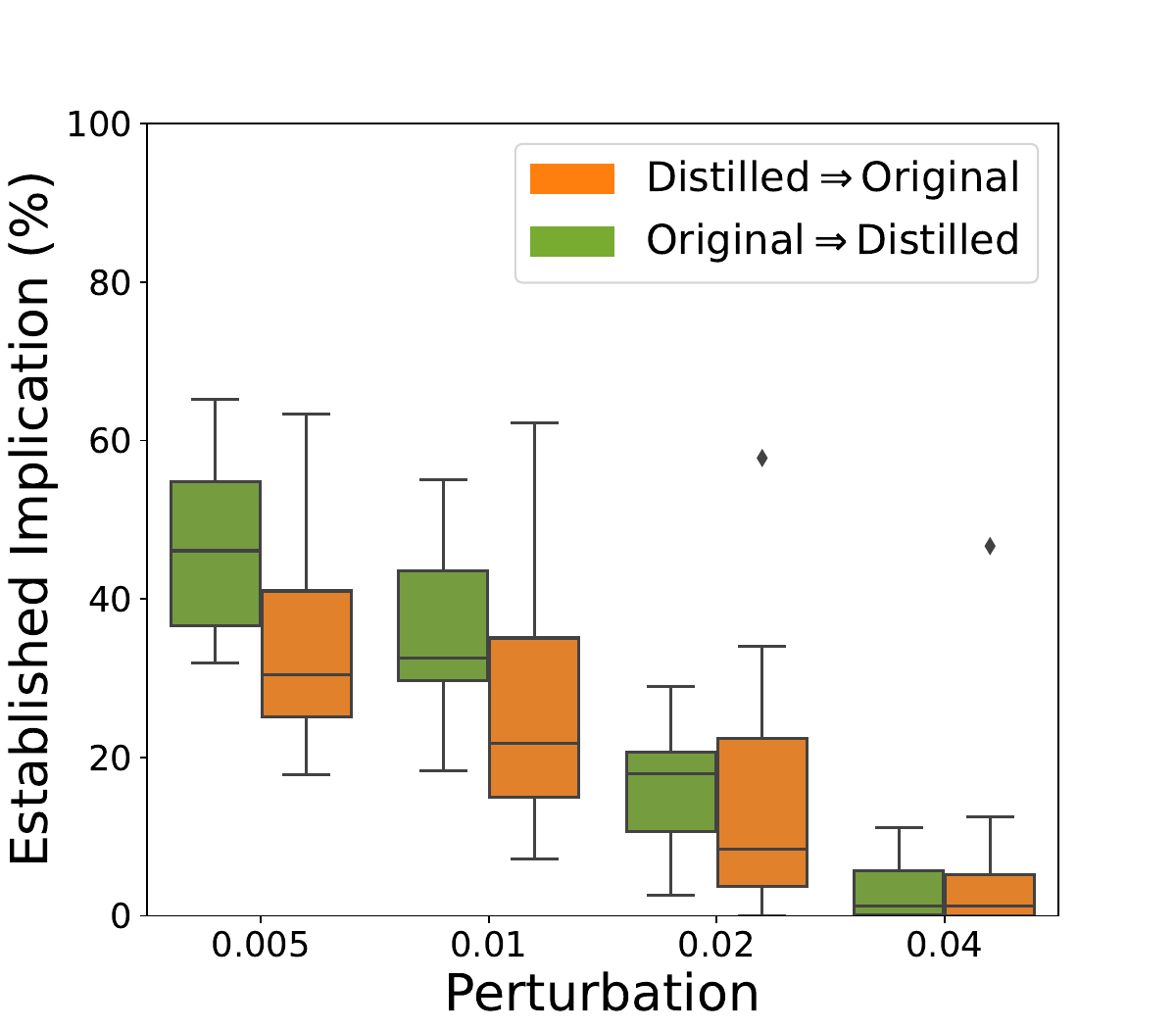}
    \caption{BIH, T = 8}
  \end{subfigure}
  \hfill
  \begin{subfigure}{0.24\textwidth}
    \includegraphics[width=\linewidth]{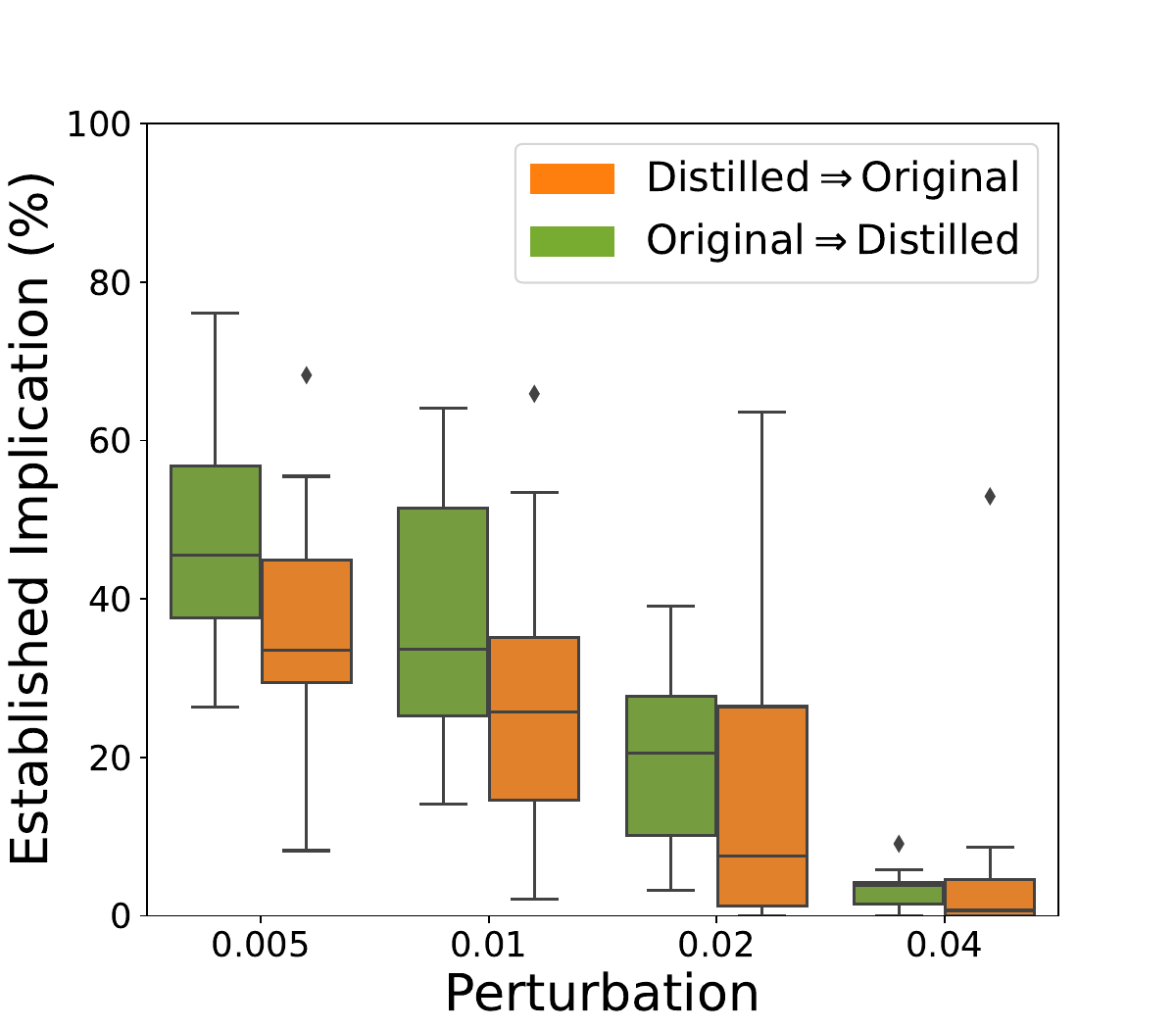}
    \caption{BIH, T = 9}
  \end{subfigure}
  
  \caption{Box plots illustrate the established implication of convolutional \glspl{DNN} trained on all patients in the CHB-MIT and MIT-BIH datasets for Original and Distilled networks. For each patient in the dataset, we evaluate the implication between an Original and a Distilled network using the patient’s own data and aggregate the results across all patients to present them in the box plots.}
  \label{CHB_BIH_distilled_appendix}
\end{figure}

\end{document}